\def\year{2022}\relax
\documentclass[letterpaper]{article} 
\usepackage{aaai22}  
\usepackage{times}  
\usepackage{helvet}  
\usepackage{courier}  
\usepackage[hyphens]{url}  
\usepackage{graphicx} 
\usepackage{natbib}  
\usepackage{caption} 
\DeclareCaptionStyle{ruled}{labelfont=normalfont,labelsep=colon,strut=off} 
\usepackage{algorithm}
\usepackage{algorithmic}

\usepackage{newfloat}
\usepackage{listings}
\floatstyle{ruled}
\newfloat{listing}{tb}{lst}{}
\floatname{listing}{Listing}

\usepackage{epstopdf,fancyhdr,amsmath,amsthm,xspace,enumerate,paralist, mathtools,tabularx,hhline,tabu,forest,array,stmaryrd}
\usepackage[capitalise]{cleveref}
\usepackage[export]{adjustbox}
\usepackage{pgf,tikz}
\usepackage{bm}
\usepackage{subcaption}
\usepackage{mathrsfs}
\usepackage{type1cm}
\usepackage{csquotes}
\usepackage{amssymb}
\usepackage{comment}
\usepackage{multirow}
\usepackage{verbatim}
\usepackage{enumitem}
\usepackage{upgreek}
\usepackage{natbib}
\usepackage{upgreek}
\usepackage{pdflscape}
\usepackage{thm-restate}
\usepackage{soul}
\usepackage{dsfont}
\usepackage{stmaryrd}
\usepackage{environ}
\usepackage{wrapfig}
\usepackage{float}

\tikzstyle{mybox} = [draw=gray, fill=gray!20, very thick,
rectangle, rounded corners, inner xsep=3pt, inner ysep=7pt]

\usepackage{pgfplots}
\DeclareUnicodeCharacter{2212}{−}
\usepgfplotslibrary{groupplots,dateplot}
\usetikzlibrary{patterns,shapes.arrows}
\pgfplotsset{compat=newest}

\usetikzlibrary{positioning,arrows,shapes,shapes.arrows,arrows.meta,trees,shapes.misc,shapes.geometric,decorations.pathreplacing,automata,patterns,fadings}
\tikzset{%
  >={Latex[width=1.5mm,length=1.8mm]},
  vertex/.style={draw,circle,inner sep=0mm,semithick,minimum width=4.7mm},
  bvertex/.style={draw,circle,inner sep=0mm,semithick,minimum width=5.2mm},
  point/.style = {circle, draw, inner sep=0.04cm,fill,node contents={}},
  uvertex/.style={outer sep=0},
  box/.style={rectangle,minimum size=1.01cm},
  bidir/.style={<->,dashed,semithick,line width=0.25mm},
  dir/.style={->, line width=0.25mm},
  regime/.style={shape=rectangle,fill=black,inner sep=0pt,outer sep=1mm,minimum size=3pt,draw},
  node distance=1cm,
  font=\scriptsize\sffamily,
}

\pgfdeclarefading{fading}
 {\tikz
    \fill[bottom color=transparent!0,top color=transparent!70] (0,0) rectangle (100bp,100bp);}

\definecolor{lightred}{RGB}{255, 150, 141}
\definecolor{lightblue}{RGB}{86, 193, 255}
\definecolor{lightgreen}{RGB}{115, 253, 134}

\definecolor{betterred}{RGB}{255, 10, 78}
\definecolor{betterblue}{RGB}{0, 162, 255}
\definecolor{bettergreen}{RGB}{22, 231, 207}

\definecolor{darkred}{RGB}{228,26,28}
\definecolor{darkblue}{RGB}{55,126,184}
\definecolor{darkgreen}{RGB}{77,175,74}
\definecolor{darkyellow}{RGB}{255,127,0}

\def\ci{\perp\!\!\!\perp}
\newcommand{\mli}[1]{\mathit{#1}}

\DeclarePairedDelimiter\ceil{\lceil}{\rceil}
\DeclarePairedDelimiter\floor{\lfloor}{\rfloor}

\Crefname{equation}{Eq.}{Eqs.}
\Crefname{figure}{Fig.}{Figs.}
\Crefname{tabular}{Tab.}{Tabs.}
\Crefname{theorem}{Thm.}{Thms.}
\Crefname{lemma}{Lem.}{Lems.}
\Crefname{proposition}{Prop.}{Props.}
\Crefname{definition}{Def.}{Defs.}
\Crefname{algorithm}{Alg.}{Algs.}
\Crefname{corollary}{Corol.}{Corol.}
\Crefname{section}{Sec.}{Sec.}

\newtheorem{lemma}{Lemma}

\theoremstyle{definition}
\newtheorem{definition}{Definition}

\newcommand{\angles}[2][]{#1\langle#2 #1\rangle}

\newcommand{\tuple}[2][]{\angles[#1]{#2}}

\makeatletter
\newcommand{\xdashleftrightarrow}[2][]{\ext@arrow 3359\leftrightarrowfill@@{#1}{#2}}
\def\rightarrowfill@@{\arrowfill@@\relax\relbar\rightarrow}
\def\leftarrowfill@@{\arrowfill@@\leftarrow\relbar\relax}
\def\leftrightarrowfill@@{\arrowfill@@\leftarrow\relbar\rightarrow}
\def\arrowfill@@#1#2#3#4{%
  $\m@th\thickmuskip0mu\medmuskip\thickmuskip\thinmuskip\thickmuskip
   \relax#4#1
   \xleaders\hbox{$#4#2$}\hfill
   #3$%
}
\makeatother

\newcommand{\I}{\mathds{1}}

\newcommand{\VV}{\bm{V}}

\newcommand{\XX}{\bm{X}}
\newcommand{\xx}{\bm{x}}

\newcommand{\D}{\Omega}

\newcommand{\G}{\mathcal{G}}

\newcommand{\Pa}{\mli{Pa}}
\newcommand{\pa}{\mli{pa}}
\newcommand{\PA}{\mli{PA}}

\newcommand{\An}{\mli{An}}

\newcommand{\ch}{\mli{ch}}

\newcommand{\doo}{\text{do}}

\newcommand{\interior}[1]{#1_{\mathrm{o}}}

\def\*#1{\boldsymbol{#1}}
\def\1#1{\mathcal{#1}}
\def\2#1{\mathscr{#1}}
\def\3#1{\mathbb{#1}}
\def\4#1{\mathds{#1}}
\def\5#1{\bar{\*#1}}

\title{Partial Counterfactual Identification from Observational and Experimental Data}
\author {
    Junzhe Zhang,\textsuperscript{\rm 1}
    Jin Tian, \textsuperscript{\rm 2}
    Elias Bareinboim \textsuperscript{\rm 1}
}
\affiliations {
    \textsuperscript{\rm 1} Columbia University\\
    \textsuperscript{\rm 2} Iowa State University\\
    junzhez@cs.columbia.edu, jtian@iastate.edu, eb@cs.columbia.edu
}

\begin{document}

\maketitle

\begin{abstract}
    This paper investigates the problem of bounding counterfactual queries from an arbitrary collection of observational and experimental distributions and qualitative knowledge about the underlying data-generating model represented in the form of a causal diagram. 
    We show that all counterfactual distributions in an arbitrary structural causal model (SCM) could be generated by a canonical family of SCMs with the same causal diagram where unobserved (exogenous) variables are discrete with  a finite domain. 
    Utilizing the canonical SCMs, we translate the problem of bounding counterfactuals into that of polynomial programming  whose solution provides optimal bounds for the counterfactual query. Solving such polynomial programs is in general computationally expensive. We therefore 
    develop effective Monte Carlo algorithms to approximate the optimal bounds from an arbitrary combination of observational and experimental data. Our algorithms are validated extensively on synthetic and real-world datasets.
\end{abstract}
\section{Introduction}
This paper studies the problem of inferring counterfactual queries from a combination of observations, experiments, and qualitative assumptions about the phenomenon under investigation. The assumptions are represented in the form of a \emph{causal diagram} \citep{pearl:95}, which is a directed acyclic graph where arrows indicate the potential existence of functional relationships among corresponding variables; some variables are unobserved. This problem arises in diverse fields such as artificial intelligence, statistics, cognitive science, economics, and the health and social sciences. For example, when investigating the gender discrimination in college admission, one may ask ``what would the admission outcome be for a female applicant had she been a male?'' Such a counterfactual query contains conflicting information: in the real world the applicant is female, in the hypothetical world she was not. Therefore, it is not immediately clear how to design effective experimental procedures for evaluating counterfactuals, or how to compute them from observational data.

The problem of identifying counterfactual distributions from the combination of data and a causal diagram has been studied in the causal inference literature. First, there exists a complete proof system for reasoning about counterfactual queries  \citep{halpern:98}. While such a system, in principle, is sufficient in evaluating any identifiable counterfactual expression, it lacks a proof guideline that  determines the feasibility of such evaluation efficiently. There are algorithms to determine whether a counterfactual distribution is inferrable from all possible controlled experiments \citep{shpitser:pea07}, or a special type of counterfactual distributions, called path-specific effects, from observational \citep{shpitser2018identification} and experimental data \citep{avin:etal05}. Finally, there exist an algorithm that decides whether any nested counterfactual is identifiable an arbitrary combination of observational and experimental distributions \citep{correa2021nested}.

In practice, however, the combination of quantitative knowledge and observed data does not always permit one to uniquely determine the target counterfactual query. In such cases, the counterfactual query is said to be \emph{non-identifiable}. \emph{Partial identification} methods concern with deriving informative bounds over the target counterfactual probability in non-identifiable settings. Several algorithms have been developed to bound counterfactual probabilities from the combination of observational and experimental data \citep{manski:90,robins:89b,balke:pea94b,balke:pea97,evans2012graphical,richardson2014nonparametric,kallus2018confounding,kallus2020confounding,finkelstein2020deriving,kilbertus2020class,zhang2021bounding}.  

In this work, we build on the approach introduced by Balke \& Pearl in \citep{balke:pea94b}, which involves direct discretization of unobserved domains, also referred to as the canonical partitioning or the principal stratification \citep{frangakis:rub02,pearl:11-r382}. Consider the causal diagram in \Cref{fig1a}, where $X, Y, Z$ are binary variables in $\{0, 1\}$; $U$ is an unobserved variable taking values in an arbitrary continuous domain. \citep{balke:pea94b} showed that domains of $U$ could be discretized into $16$ equivalent classes without changing the original counterfactual distributions and the graphical structure in \Cref{fig1a}. For instance, suppose that values of $U$ are drawn from an arbitrary distribution $P^*(u)$ over a continuous domain. It has been shown that the observational distribution $P(x, y, z)$ could be reproduced by a generative model of the form $P(x, y, z) = \sum_{u} P(x|u,z)P(y|x, u)P(u)P(z)$,  where $P(u)$ is a discrete distribution over a finite domain $\{1, \dots, 16\}$.

Using the finite-state representation of unobserved variables, \citep{balke:pea97} derived tight bounds on treatment effects under a set of constraints called \emph{instrumental variables} (e.g., \Cref{fig1a}). \citep{chickering:pea97,imbens1997bayesian} applied the parsimony of finite-state representation in a Bayesian framework, to obtain credible intervals for the posterior distribution of causal effects in noncompliance settings. 
Despite the optimality guarantees in their treatments, these bounds were only derived for specific settings. 
A systematic strategy for partial identification in an arbitrary causal diagram is still missing. 
There are significant challenges in bounding any counterfactual query in an arbitrary causal diagram given an arbitrary collection of observational and experimental data.

Our goal in this paper is to overcome these challenges. 
We show that when inferring about counterfactual distributions (over finite observed variables) in an arbitrary causal diagram, one could restrict domains of unobserved variables to a finite space without loss of generality. This result allows us to develop novel partial identification algorithms to bound unknown counterfactual probabilities from an arbitrary combination of observational and experimental data. In some way, this paper can be seen as closing a long-standing open problem introduced by \citep{balke:pea94b}, where they solve a special bounding instance in the case of instrumental variables. More specifically, our contributions are as follows. (1) We introduce a special family of \emph{discrete structural causal models}, 
and show that it could represent all categorical counterfactual distributions (with finite support) in an arbitrary causal diagram. (2) Using this result, we translate the partial identification task into an equivalent polynomial program. Solving such a program leads to bounds over target counterfactual probabilities that are provably optimal. (3) We develop an effective Monte Carlo algorithm to approximate optimal bounds from a finite number of observational and experimental data. Finally, our algorithms are validated extensively on synthetic and real-world datasets. 

\subsection{Preliminaries}\label{sec:prelim}
We introduce in this section some basic notations and definitions that will be used throughout the paper. We use capital letters to denote variables ($X$), small letters for their values ($x$) and $\D_X$ for their domains. For an arbitrary set $\*X$, let $|\*X|$ be its cardinality. The probability distribution over variables $\*X$ is denoted by $P(\*X)$. For convenience, we consistently use $P(\*x)$ as a shorthand for the probability $P(\*X = \*x)$. Finally, the indicator function $\I_{\*X = \*x}$ returns $1$ if an event $\*X = \*x$ holds; otherwise, $\I_{\*X = \*x}$ is equal to $0$.


The basic semantical framework of our analysis rests on \textit{structural causal models} (SCMs) \citep[Ch.~7]{pearl:2k}. An SCM $M$ is a tuple $\tuple{\*V, \*U, \2F, P}$ where $\*V$ is a set of endogenous variables and $\*U$ is a set of exogenous variables. $\2F$ is a set of functions where each $f_V \in \2F$ decides values of an endogenous variable $V \in \*V$ taking as argument a combination of other variables in the system. That is, $v \leftarrow f_{V}(\pa_V, u_V), \PA_V \subseteq \*V, U_V \subseteq \*U$. Exogenous variables $U \in \*U$ are mutually independent, values of which are drawn from the exogenous distribution $P(\*U)$. Naturally, $M$ induces a joint distribution $P(\*V)$ over endogenous variables $\*V$, called the \emph{observational distribution}. 

Each SCM $M$ is also associated with a causal diagram $\G$ (e.g., \Cref{fig1}), which is a directed acyclic graph (DAG) where solid nodes represent endogenous variables $\*V$, empty nodes represent exogenous variables $\*U$, and arrows represent the arguments $\PA_V, U_V$ of each structural function $f_{V}$. We will use graph-theoretic family abbreviations for graphical relationships such as parents and children. For example, the set of parents of $\*X$ in $\G$ is denoted by $\pa(\*X)_{\G} = \cup_{X \in \*X} \pa(X)_{\G}$; $\ch$ are similarly defined. 
The subscript $\G$ will be omitted when it is obvious from the context.

An intervention on an arbitrary subset $\XX \subseteq \VV$, denoted by $\doo(\xx)$, is an operation where values of $\XX$ are set to constants $\xx$, regardless of how they are ordinarily determined. For an SCM $M$, let $M_{\*x}$ denote a submodel of $M$ induced by intervention $\doo(\xx)$. For any subset $\*Y \subseteq \*V$, the \emph{potential response} $\*Y_{\*x}(\*u)$ is defined as the solution of $\*Y$ in the submodel $M_{\*x}$ given $\*U = \*u$. Drawing values of exogenous variables $\*U$ following the probability distribution $P$ induces a \emph{counterfactual variable} $\*Y_{\*x}$. Specifically, the event $\*Y_{\*x} = \*y$ (for short, $\*y_{\*x}$) can be read as ``$\*Y$ would be $\*y$ had $\*X$ been $\*x$''. 
For any subsets $\*Y, \dots, \*Z$, $\*X, \dots, \*W \subseteq \*V$, the distribution over counterfactuals $\*Y_{\*x}, \dots, \*Z_{\*w}$ is defined as:
\begin{align} \label{eq:cfd}
  P \left (\*y_{\*x}, \dots, \*z_{\*w} \right) = \int_{\D_{\*U}} \I_{\*Y_{\*x}(\*u) = \*y, \dots, \*Z_{\*w}(\*u) = \*z} dP(\*u).
\end{align}
Distributions of the form $P(\*Y_{\*x})$ are called  \emph{interventional distributions}; when $\*X = \emptyset$, $P(\*Y)$ coincides with the \emph{observational distribution}. Throughout this paper, we assume that endogenous variables $\*V$ are discrete and finite; while exogenous variables $\*U$ could take any (continuous) value. The counterfactual distribution $P \left (\*Y_{\*x}, \dots, \*Z_{\*w} \right)$ defined above is thus a categorical distribution. For a more detailed survey on SCMs, we refer readers to \citep[Ch.~7]{pearl:2k}.
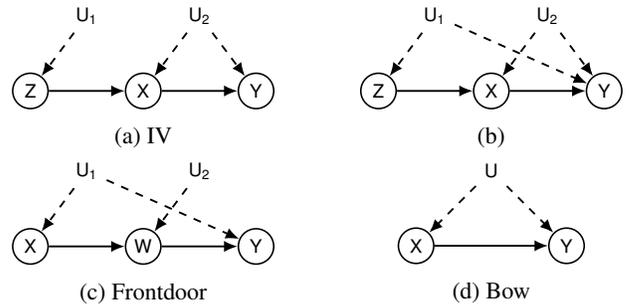
\begin{figure}[t]
  \centering
  \null
  \begin{subfigure}{0.45\linewidth}\centering
    \begin{tikzpicture}
      \node[vertex] (Z) at (0, 0) {Z};
      \node[vertex] (X) at (1.5, 0) {X};
      \node[vertex] (Y) at (3, 0) {Y};
      \node[uvertex] (U1) at (0.75, 1) {U\textsubscript{1}};
      \node[uvertex] (U2) at (2.25, 1) {U\textsubscript{2}};
      
      \draw[dir] (X) -- (Y);
      \draw[dir] (Z) -- (X);
      \draw[dir, dashed] (U1) -- (Z);
      \draw[dir, dashed] (U2) -- (X);
      \draw[dir, dashed] (U2) -- (Y);
    \end{tikzpicture}
  \caption{IV}
  \label{fig1a}
  \end{subfigure}\hfill
  \begin{subfigure}{0.45\linewidth}\centering
      \begin{tikzpicture}
          \node[vertex] (Z) at (0, 0) {Z};
          \node[vertex] (X) at (1.5, 0) {X};
          \node[vertex] (Y) at (3, 0) {Y};
          \node[uvertex] (U1) at (0.75, 1) {U\textsubscript{1}};
          \node[uvertex] (U2) at (2.25, 1) {U\textsubscript{2}};
          
          \draw[dir] (X) -- (Y);
          \draw[dir] (Z) -- (X);
          \draw[dir, dashed] (U1) -- (Z);
          \draw[dir, dashed] (U1) -- (Y);
          \draw[dir, dashed] (U2) -- (X);
          \draw[dir, dashed] (U2) -- (Y);
      \end{tikzpicture}
    \caption{}
    \label{fig1b}
  \end{subfigure}\hfill
  \begin{subfigure}{0.45\linewidth}\centering
    \begin{tikzpicture}
      \node[vertex] (X) at (0, 0) {X};
      \node[vertex] (W) at (1.5, 0) {W};
      \node[vertex] (Y) at (3, 0) {Y};
      \node[uvertex] (U1) at (0.75, 1) {U\textsubscript{1}};
      \node[uvertex] (U2) at (2.25, 1) {U\textsubscript{2}};
      
      \draw[dir] (X) -- (W);
      \draw[dir] (W) -- (Y);
      \draw[dir, dashed] (U1) -- (X);
      \draw[dir, dashed] (U1) -- (Y);
      \draw[dir, dashed] (U2) -- (W);
  \end{tikzpicture}
  \caption{Frontdoor}
  \label{fig1c}
\end{subfigure}\hfill
\begin{subfigure}{0.45\linewidth}\centering
\begin{tikzpicture}
  \node[vertex] (X) at (0, 0) {X};
  \node[vertex] (Y) at (2, 0) {Y};
  \node[uvertex] (U) at (1, 1) {U};
  
  \draw[dir] (X) -- (Y);
  \draw[dir, dashed] (U) to (X);
  \draw[dir, dashed] (U) to (Y);
\end{tikzpicture}
\caption{Bow}
\label{fig1d}
\end{subfigure}\null
  \caption{Causal diagrams containing a treatment $X$, an outcome $Y$, an ancestor $Z$, a mediator $W$, and unobserved $U$s.}
  \label{fig1}
\end{figure}
\section{Partial Counterfactual Identification}
We introduce the task of partial identification of a counterfactual probability from a combination of observational and interventional distributions, which  generalizes the previous partial identifiability settings that assume observational data are given \citep{balke:pea97,imbens1997bayesian}. Let $\3Z = \{\*z_i\}_{i = 1}^m$ be a finite collection of realizations $\*z_i$ for sets of variables $\*Z_i \subseteq \*V$. We assume data are available from all of the interventional distributions in $\left \{P(\*V_{\*z})\mid \*z \in \3Z \right\}$. Note that $\*Z = \emptyset$ corresponds to the observational distribution $P(\*V)$. Our goal is to find a bound $[l, r]$ for any counterfactual probability $P \left (\*y_{\*x}, \dots, \*z_{\*w} \right)$ from the collection $\left \{P(\*V_{\*z})\mid \*z \in \3Z \right\}$ and the causal diagram $\G$. 

Formally, let $\2M(\G)$ be the set of all SCMs associated with $\G$, i.e., $\2M(\G) =\left \{\forall M \mid \G_M =\G \right \}$\footnote{We will use the subscript $M$ to represent the restriction to an SCM $M$. Therefore, $\G_M$ represents the causal diagram associated with $M$; so does counterfactual distributions $P_M \left (\*y_{\*x}, \dots, \*z_{\*w} \right)$.}. The bound $[l, r]$ is obtainable by solving the following optimization problem:
\begin{equation}
  \begin{aligned}
    \underset{M \in \2M(\G)}{\min / \max} \quad &P_M \left (\*y_{\*x}, \dots, \*z_{\*w} \right)\\
    \textrm{s.t.} \quad & P_M(\*v_{\*z}) = P(\*v_{\*z}) \;\; \forall \*v, \forall \*z \in \3Z
  \end{aligned} \label{eq:opt1}
\end{equation}
where $P_M \left (\*y_{\*x}, \dots, \*z_{\*w} \right)$ and $P_M(\*v_{\*z})$ are given in the form of \Cref{eq:cfd}. The lower $l$ and upper bound $r$ are minimum and maximum of the above equation respectively. By the formulation of \Cref{eq:opt1}, $[l, r]$ must be the tight bound containing all possible values of the target counterfactual $P \left (\*y_{\*x}, \dots, \*z_{\*w} \right)$.

Since we do not have access to the parametric forms of the underlying structural functions $f_V$ nor the exogenous distribution $P(\*u)$, solving the optimization problem in \Cref{eq:opt1} appears theoretically challenging. 
It is not clear how the existing optimization procedures can be used. 
Next we show the optimization problem in \Cref{eq:opt1} can be reduced into a polynomial program by constructing an ``canonical'' SCM that is equivalent to the original SCM in representing the objective $P \left (\*y_{\*x}, \dots, \*z_{\*w} \right)$ and all constraints $ P(\*V_{\*z}), \forall \*z \in \3Z$.

\subsection{Canonical Structural Causal Models}
Our construction rests on the parametric family of discrete SCMs where values of each exogenous variable are drawn from a discrete distribution over a finite set of states. 
\begin{definition}\label{def:d-scm}
  An SCM $M = \langle\*V, \*U, \2F, P \rangle$ is said to be a discrete SCM if
  \begin{enumerate}[itemsep=0pt,topsep=0pt,parsep=0pt]
    \item For every exogenous $U \in \*U$, its values $u$ are contained in a discrete domain $\D_U$;
    \item For every endogenous $V \in \*V$, its values $v$ are given by a function $v \gets f_{V}(\pa_V, u_V)$ where for any $\pa_V, u_V$, $f_{V}(\pa_V, u_V)$ is contained in a finite domain $\D_V$.
  \end{enumerate}
\end{definition}
For endogenous variables $\*V$, let $\*P^*$ denote the collection of all possible counterfactual distributions over $\*V$, i.e.,
\begin{align}
  \*P^* = \left \{P \left (\*Y_{\*x}, \dots, \*Z_{\*w} \right) \mid \forall \*Y, \dots, \*Z, \*X, \dots, \*W \right\}. \label{eq:ctf2}
\end{align}
Recall that $\2M(\G)$ is the set of all SCMs compatible with a causal diagram $\G$. Counterfactual distributions in $\G$ are defined as $\left \{\*P^*_M: \forall M \in \2M(\G) \right \}$. We will next show that discrete SCMs are indeed ``canonical'', i.e., they could generates \emph{all counterfactual distributions} in any causal diagram.

Our analysis utilizes a special type of clustering of endogenous variables in the causal diagram developed by \citep{tian:pea02}, which we call \emph{confounded components}. 
\begin{definition}\label{def:c-component}
  For a causal diagram $\G$, let $U \in \*U$ be an arbitrary exogenous variable. A set of endogenous variables $\*C(U) \subseteq \*V$ (w.r.t. $U$) is a c-component if for every $V \in \*C(U)$, there exists a sequence $\{U_1, \dots, U_n \} \subseteq \*U$ such that:
  \begin{enumerate}[itemsep=0pt,topsep=0pt,parsep=0pt]
    \item $U_1 = U$ and $U_n \in U_V$;
    \item for every $i = 1, \dots, n-1$, $U_i$ and $U_{i+1}$ have a common child node, i.e., $\ch(U_i)\cap \ch(U_{i+1}) \neq \emptyset$.
  \end{enumerate}
\end{definition}
A c-component $\*C(U)$ in $\G$ is maximal if there exists no other c-component that strictly contains $\*C(U)$. For convenience, we will consistently use $\*C(U)$ to denote the maximal c-component w.r.t. every exogenous $U \in \*U$. 
For instance, \Cref{fig1a} contains two c-components $\*C(U_1) = \{Z\}$ and $\*C(U_2) = \{X, Y\}$; while exogenous variables $U_1, U_2$ in \Cref{fig1b} share the same c-component $\*C(U_1) = \*C(U_2) = \{X, Y, Z\}$ since they have a common child node $Y$.
\begin{restatable}{theorem}{thmcfinite}\label{thm:cfinite}
  For a DAG $\G$, consider following conditions\footnote{For every $V \in \*V$, we denote by $\D_{\PA_V} \mapsto \D_V$ the set of all possible functions mapping from domains $\D_{\PA_V}$ to $\D_V$.}: 
  \begin{enumerate}[itemsep=0pt,topsep=0pt,parsep=0pt]
    \item $\2M(\G)$ is the set of all SCMs compatible with $\G$.
    \item $\2N(\G)$ is the set of all discrete SCMs compatible with $\G$ such that for every exogenous $U \in \*U$, 
    \begin{align}
      \left|\D_U \right| = \prod_{V \in \*C(U)} \left |\D_{\PA_V} \mapsto \D_V \right|, \label{eq:cardinality}
    \end{align} 
    i.e., the number of functions mapping from domains of $\PA_V$ to $V$ for every endogenous $V \in \*C(U)$. 
  \end{enumerate} 
  Then, $\2M(\G), \2N(\G)$ are counterfactually equivalent, i.e., 
  \begin{align}
      \left \{\*P^*_M: \forall M \in \2M(\G) \right \} = \left \{\*P^*_N: \forall N \in \2N(\G) \right \}. \label{eq:ctf_eq}
  \end{align}
\end{restatable}
\Cref{thm:cfinite} establishes the expressive power of discrete SCMs in representing counterfactual distributions in a causal diagram $\G$. Henceforth, we will refer to $\2N(\G)$ in \Cref{thm:cfinite} as the family of \emph{canonical SCMs} for $\G$. As an example, consider a causal diagram $\G$ in \Cref{fig1b} where $X, Y, Z$ are binary variables in $\{0,1\}$. Since $U_1, U_2$ share the same c-component $\{X, Y, Z\}$, \Cref{eq:cardinality} implies that they also share the same cardinality $d = |\D_Z| \times |\D_Z \mapsto \D_X| \times |\D_X \mapsto \D_Y| = 32$ in the canonical family $\2N(\G)$. It follows from \Cref{thm:cfinite} that the counterfactual distribution $P(X_{z'}, Y_{x'})$ in $\G$ could be generated by a SCM in $\2N(\G)$ and be written as follows:
\begin{align*}
  &P(x_{z'}, y_{x'}) = \sum_{u_1, u_2 = 1}^d \I_{f_X(z', u_2)= x, f_Y(x', u_1, u_2) = y} P(u_1)P(u_2).
\end{align*}
More generally, \Cref{thm:cfinite} implies that counterfactual probabilities $P \left (\*Y_{\*x}, \dots, \*Z_{\*w} \right)$ in any SCM $M$ could be generically generated as follows: for $d_U = \prod_{V \in \*C(U)} \left |\D_{\Pa_V} \mapsto \D_V \right|$,
\begin{equation}
\begin{aligned}
  &P \left (\*y_{\*x}, \dots, \*z_{\*w} \right) \\
  &= \sum_{\*u}\I_{\*Y_{\*x}(\*u) = \*y, \dots, \*Z_{\*w}(\*u) = \*z} \prod_{U \in \*U} P(u). 
\end{aligned}\label{eq:sec3.4}
\end{equation}
Among above quantities, $P(U)$ is a discrete distribution over a finite domain $\{1, \dots, d_U\}$. Counterfactual variables $\*Y_{\*x}(\*u) = \left \{Y_{\*x}(\*u) \mid \forall Y \in \*Y \right \}$ are recursively defined as:
\begin{align}
  Y_{\*x}(\*u) = 
  \begin{dcases}
    \*x_Y &\mbox{if }Y \in \*X \\
    f_Y\left ( \left(\PA_{Y}\right)_{\*x}(\*u), u_Y \right) &\mbox{otherwise}
  \end{dcases}
  \label{eq:sec3.3}
\end{align}
where $\*x_Y$ is the value assigned to variable $Y$ in constants $\*x$. 


\paragraph{Related work}The discretization procedure in \citep{balke:pea94b} was originally designed for the ``IV'' diagram in \Cref{fig1a}, but it can not be immediately extended to other causal diagrams without loss of generality (see Appendix E for a detailed example). More recently, \citep{rosset2018universal} applied a classic result of Carath\'eodory theorem in convex geometry \citep{caratheodory1911variabilitatsbereich} and showed that the observational distribution in any causal diagram could be represented using finitely many latent states. \citep{evans2018margins} proved a special case of \Cref{thm:cfinite} for interventional distributions in a restricted class of causal diagrams satisfying a running intersection property. 

\Cref{thm:cfinite} generalizes existing results in several important ways. First, we prove that \emph{all} counterfactual distributions could be generated using discrete exogenous variables with finite domains, which subsume both observational and interventional distributions. Second, \Cref{thm:cfinite} is applicable to \emph{any} causal diagram, thus not relying on additional graphical conditions, e.g., IV constraints \citep{balke:pea94b}. More specifically, we introduce a general, canonical partitioning over exogenous domains in an arbitrary SCM. Any counterfactual distribution in this SCM could be written as a function of joint probabilities assigned to intersections of canonical partitions.
This allows us to discretize exogenous domains while maintaining all counterfactual distributions and structures of the causal diagram. We refer readers to Appendix A for more details about the proof for \Cref{thm:cfinite}.

\subsection{Bounding Counterfactual Distributions}
The expressive power of canonical SCMs in \Cref{thm:cfinite} suggests a natural algorithm for the partial identification of counterfactual distributions. Recall that the canonical family $\2N(\G)$ for a causal diagram $\G$ consists of discrete SCMs with finite exogenous states. We derive a bound $[l, r]$ over a counterfactual probability $P \left (\*y_{\*x}, \dots, \*z_{\*w} \right)$ from an arbitrary collection of interventional distributions $\left \{P(\*V_{\*z}) \mid \*z \in \3Z\right\}$ by solving the following optimization problem:
\begin{equation}
  \begin{aligned}
    \underset{N \in \2N(\G)}{\min / \max} \quad &P_N \left (\*y_{\*x}, \dots, \*z_{\*w} \right)\\
    \textrm{s.t.} \quad & P_N(\*v_{\*z}) = P(\*v_{\*z}) \;\; \forall \*v, \forall \*z \in \3Z
  \end{aligned} \label{eq:opt2}
\end{equation}
where $P_N \left (\*y_{\*x}, \dots, \*z_{\*w} \right)$ and $P_N(\*v_{\*z})$ are given in the form of \Cref{eq:sec3.4}. The optimization problem in \Cref{eq:opt2} is generally reducible to a polynomial program. To witness, for every $U \in \*U$, let parameters $\theta_u$ represent discrete probabilities $P(U = u)$. For every $V \in \*V$, we represent the output of function $f_V(\pa_V, u_V)$ given input $\pa_V, u_V$ using an indicator vector $\mu_V^{(\pa_V, u_V)} = \left ( \mu_v^{(\pa_V, u_V)} \mid \forall v \in \D_V \right)$ such that 
\begin{align*}
  &\mu_v^{(\pa_V, u_V)}\in \{0, 1\}, &\sum_{v \in \D_V} \mu_v^{(\pa_V, u_V)} = 1.
\end{align*} 
Doing so allows us to write any counterfactual probability $P \left (\*y_{\*x}, \dots, \*z_{\*w} \right)$ in \Cref{eq:sec3.4} as a polynomial function of parameters $\mu_v^{(\pa_V, u_V)}$ and $\theta_u$. More specifically, the indicator function $\I_{\*Y_{\*x}(\*u) = \*y}$ is equal to a product $\prod_{Y \in \*Y} \I_{Y_{\*x}(\*u) = y}$. For every $Y \in \*Y$, $\I_{Y_{\*x}(\*u) = y}$ is recursively given by: 
\begin{align*}
  \I_{Y_{\*x}(\*u) = y} = 
    \begin{dcases}
      \I_{y = \*x_Y} &\mbox{if }Y \in \*X \\
      \sum_{\pa_Y} \mu_y^{\left (\pa_Y, u_Y \right)} \I_{ \left( \PA_{Y} \right )_{\*x}(\*u) = \pa_Y }  &\mbox{otherwise}
    \end{dcases}
\end{align*} 
For instance, consider again the causal diagram $\G$ in \Cref{fig1b}. The counterfactual distribution $P(X_{z'}, Y_{x'})$ and the observational distribution $P(X, Y, Z)$ of any discrete SCM in $\2N(\G)$ and be written as the following polynomial functions:
\begin{align}
  &P(x_{z'}, y_{x'}) = \sum_{u_1, u_2 = 1}^d \mu_x^{(z', u_2)} \mu_Y^{(x', u_1, u_2)} \theta_{u_1}\theta_{u_2}, \label{eq:sec3.1}\\
  &P(x, y, z) = \sum_{u_1, u_2 = 1}^d  \mu_z^{(u_1) } \mu_x^{(z, u_2)} \mu_y^{(x, u_1, u_2)} \theta_{u_1}\theta_{u_2}, \label{eq:sec3.2}
\end{align}
where $\mu_z^{(u_1)}, \mu_x^{(z', u_2)}, \mu_y^{(x', u_1, u_2)}$ are parameters taking values in $\{0, 1\}$; $\theta_{u_i}$, $i = 1, 2$, are probabilities of the discrete distribution $P(u_i)$ over the finite domain $\{1, \dots, d\}$. One could derive a bound over $P(x_{z'}, y_{x'})$ from $P(X, Y, Z)$ by solving polynomial programs which optimize the objective \Cref{eq:sec3.1} over parameters $\theta_{u_1}, \theta_{u_2}, \mu_z^{(u_1)}, \mu_x^{(z, u_2)}, \mu_y^{(x, u_1, u_2)}$, subject to the constraints in \Cref{eq:sec3.2} for all entries $x, y, z$. We refer readers to Appendix D for additional examples demonstrating how to reduce the original partial identification problem to an equivalent polynomial program.

It follows immediately from \Cref{thm:cfinite} that the solution $[l, r]$ of the optimization program in \Cref{eq:opt2} is guaranteed to be a valid, tight bound over the target counterfactual probability. 
\begin{restatable}{theorem}{thmvalid}\label{thm:valid}
Given a DAG $\G$ and $\left \{P(\*V_{\*z})\mid \*z \in \3Z \right\}$, the solution $[l, r]$ of the polynomial program \Cref{eq:opt2} is a tight bound over the counterfactual probability $P\left (\*y_{\*x}, \dots, \*z_{\*w} \right)$.
\end{restatable}

Despite the soundness and tightness of its derived bounds, solving a polynomial program in \Cref{eq:opt2} may take exponentially long in the most general case \citep{lewis1983computers}. Our focus here is upon the causal inference aspect of the problem and like earlier discussions we do not specify which solvers are used \citep{balke:pea94b,balke:pea97}. In some cases of interest, effective approximate planning methods for polynomial programs do exist. Investigating these methods is an ongoing subject of research \citep{lasserre2001global,parrilo2003semidefinite}.
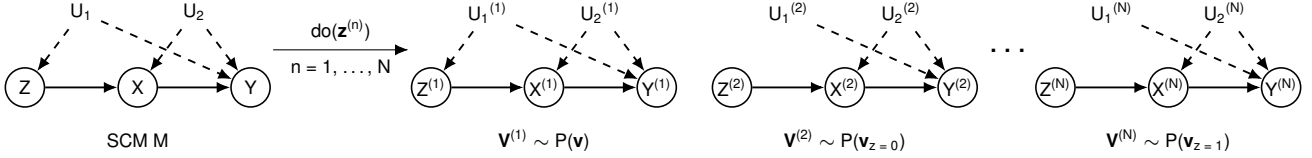
\begin{figure*}[t]
  \centering
  \begin{tikzpicture}
          \node[bvertex] (Z) at (0, 0) {Z};
          \node[bvertex] (X) at (1.5, 0) {X};
          \node[bvertex] (Y) at (3, 0) {Y};
          \node[uvertex] (U1) at (0.75, 1) {U\textsubscript{1}};
          \node[uvertex] (U2) at (2.25, 1) {U\textsubscript{2}};

          \node[uvertex] (t1) at (1.5, -0.7) {SCM M};

          \draw[->] (3.3, 0.5) to node[above] {do(\textbf{z}\textsuperscript{(n)})} node[below] {n = 1, \dots, N} (5.1, 0.5);


          \node[bvertex] (Z1) at (5.4, 0) {Z\textsuperscript{(1)}};
          \node[bvertex] (X1) at (6.9, 0) {X\textsuperscript{(1)}};
          \node[bvertex] (Y1) at (8.4, 0) {Y\textsuperscript{(1)}};
          \node[uvertex] (U11) at (6.15, 1) {U\textsubscript{1}\textsuperscript{(1)}};
          \node[uvertex] (U21) at (7.65, 1) {U\textsubscript{2}\textsuperscript{(1)}};

          \node[uvertex] (t1) at (6.9, -0.7) {\textbf{V}\textsuperscript{(1)} $\sim$ P(\textbf{v})};

          \node[bvertex] (Z2) at (9.4, 0) {Z\textsuperscript{(2)}};
          \node[bvertex] (X2) at (10.9, 0) {X\textsuperscript{(2)}};
          \node[bvertex] (Y2) at (12.4, 0) {Y\textsuperscript{(2)}};
          \node[uvertex] (U12) at (10.15, 1) {U\textsubscript{1}\textsuperscript{(2)}};
          \node[uvertex] (U22) at (11.65, 1) {U\textsubscript{2}\textsuperscript{(2)}};

          \node[uvertex] (t1) at (10.9, -0.7) {\textbf{V}\textsuperscript{(2)} $\sim$ P(\textbf{v}\textsubscript{z = 0})};

          \node[bvertex] (Z3) at (13.7, 0) {Z\textsuperscript{(N)}};
          \node[bvertex] (X3) at (15.2, 0) {X\textsuperscript{(N)}};
          \node[bvertex] (Y3) at (16.7, 0) {Y\textsuperscript{(N)}};
          \node[uvertex] (U13) at (14.45, 1) {U\textsubscript{1}\textsuperscript{(N)}};
          \node[uvertex] (U23) at (15.95, 1) {U\textsubscript{2}\textsuperscript{(N)}};

          \node[uvertex] (t1) at (15.2, -0.7) {\textbf{V}\textsuperscript{(N)} $\sim$ P(\textbf{v}\textsubscript{z = 1})};

          \node[uvertex] (dots) at (13.1, 0.5) {\Large \dots};





          \draw[dir] (X) -- (Y);
          \draw[dir] (Z) -- (X);
          \draw[dir, dashed] (U1) -- (Z);
          \draw[dir, dashed] (U1) -- (Y);
          \draw[dir, dashed] (U2) -- (X);
          \draw[dir, dashed] (U2) -- (Y);

          \draw[dir] (X1) -- (Y1);
          \draw[dir] (Z1) -- (X1);
          \draw[dir, dashed] (U11) -- (Z1);
          \draw[dir, dashed] (U11) -- (Y1);
          \draw[dir, dashed] (U21) -- (X1);
          \draw[dir, dashed] (U21) -- (Y1);

          \draw[dir] (X2) -- (Y2);
          \draw[dir] (Z2) -- (X2);
          \draw[dir, dashed] (U12) -- (Y2);
          \draw[dir, dashed] (U22) -- (X2);
          \draw[dir, dashed] (U22) -- (Y2);

          \draw[dir] (X3) -- (Y3);
          \draw[dir] (Z3) -- (X3);
          \draw[dir, dashed] (U13) -- (Y3);
          \draw[dir, dashed] (U23) -- (X3);
          \draw[dir, dashed] (U23) -- (Y3);







    \end{tikzpicture}
    \caption{The data-generating process for a finite dateset $\{x^{(n)}, y^{(n)}, z^{(n)} \}_{n = 1}^N$ in an SCM associated with \Cref{fig1b}; the set $\3Z = \{\emptyset, z = 0, z = 1\}$ where the idle intervention $\doo(\emptyset)$ corresponds to the observational distribution.} 
    \label{fig5}
\end{figure*}

\section{Bayesian Approach for Partial Identification} \label{sec4}
This section describes an effective algorithm to approximate the optimal bound in \Cref{eq:opt2} from finite samples drawn from interventional distributions $\left \{P(\*V_{\*z})\mid \*z \in \3Z \right\}$, provided with prior distributions over parameters $\theta_u$ and $\mu_V^{(\pa_V, u_V)}$ (possibly uninformative). Given space constraints, all proofs for results in this section are provided in Appendix B.

More specifically, the learner has access to a finite dataset $\5v = \left\{\*V^{(n)} = \*v^{(n)} \mid n = 1, \dots, N \right \}$, where each $\*V^{(n)}$ is an independent sample drawn from an interventional distribution $P\left (\*V_{\*z} \right)$ for some $\*z \in \3Z$. With a slight abuse of notation, we denote by $\*Z^{(n)}$ the set of variables $\*Z$ that are intervened for generating the $n$-th sample; therefore, its realization $\*z^{(n)} = \*z$. As an example, \Cref{fig5} shows a graphical representation of the data-generating process for a finite dateset $\{x^{(n)}, y^{(n)}, z^{(n)} \}_{n = 1}^N$ associated with SCMs in \Cref{fig1b}; the intervention set $\3Z = \{\emptyset, z = 0, z = 1\}$.

We first introduce effective Markov Chain Monte Carlo (MCMC) algorithms that sample the posterior distribution $P\left(\theta_{\text{ctf}} \mid \5v \right)$ over an arbitrary counterfactual probability $\theta_{\text{ctf}} = P\left(\*y_{\*x}, \dots, \*z_{\*w} \right)$. For every $V \in \*V$, $\forall \pa_V, u_V$, endogenous parameters $\mu^{(\pa_V, u_V)}_V$ are drawn uniformly over the finite domain $\D_V$. For every $U \in \*U$, exogenous parameters $\theta_u$ are drawn from a Dirichlet distribution \citep{connor1969concepts}. Formally, 
\begin{align}
    &\left( \theta_1, \dots, \theta_{d_U} \right) \sim \texttt{Dir}\left (\alpha^{(1)}_U, \dots, \alpha^{(d_U)}_U \right), \label{eq:sec5.1}
\end{align}
where the cardinality $d_U = \prod_{V \in \*C(U)} \left|\D_{\Pa_V} \mapsto \D_{V} \right|$ and hyperparameters $\alpha^{(u)}_1, \dots, \alpha^{(d_U)}_U > 0$. 

Gibbs sampling is a well-known MCMC algorithm that allows one to sample posterior distributions. We first introduce the following notations. Let parameters $\*\theta$ and $\*\mu$ be:
\begin{equation}
\begin{aligned}
  &\*\theta = \left \{\theta_u \mid \forall U \in \*U, \forall u \right\},\\
  &\*\mu = \left\{\mu^{(\pa_V, u_V)}_V \mid \forall V \in \*V, \forall \pa_V, u_V \right\}.
\end{aligned}
\end{equation}
We denote by $\5U = \left\{\*U^{(n)} \mid n = 1, \dots, N \right\}$ exogenous variables affecting $N$ endogenous variables $\5V = \left \{\*V^{(n)} \mid n = 1, \dots, N \right \}$; we use $\5u$ to represent its realization. Our blocked Gibbs sampler works by iteratively drawing values from the conditional distributions of variables as follows \citep{ishwaran2001gibbs}. Detailed derivations of complete conditionals are shown in Appendix B.1.
\begin{itemize}
  \item{\bf Sampling $P\left (\5u \mid \5v, \*\theta, \*\mu \right)$. }Exogenous variables $\*U^{(n)}$, $n = 1, \dots, N$, are mutually independent given parameters $\*\theta, \*\mu$. We could draw each $\left ( \*U^{(n)} \mid \*\theta, \*\mu, \5V \right)$ corresponding to the $n$-th sample induced by $\doo(\*z^{(n)})$ independently. The complete conditional of $\*U^{(n)}$ is given by 
  \begin{equation}
  \begin{aligned}
      &P\left (\*u^{(n)} \mid \*v^{(n)}, \*\theta, \*\mu \right) \\
      &\propto \prod_{V \in \*V \setminus \*Z^{(n)}} \mu_{v^{(n)}}^{\left(\pa^{(n)}_{V}, u^{(n)}_{V} \right)} \prod_{U \in \*U} \theta_{u}.
  \end{aligned}\label{eq:sample1}
  \end{equation}
  \item {\bf Sampling $P\left (\*\mu, \*\theta \mid \5v, \5u \right)$. }Note that parameters $\*\mu, \*\theta$ are mutually independent given $\5V, \5U$. Therefore, we will derive complete conditionals over $\*\mu, \*\theta$ separately. 
  
  Consider first endogenous parameters $\*\mu$. For every $V \in \*V$, fix $\pa_V, u_V$. If there exists an instance $n = 1, \dots, N$ such that $V \not \in \*Z^{(n)}$ and $\pa^{(n)}_{V} = \pa_V, u^{(n)}_{V} = u_V$, the posterior over $\mu^{(\pa_V, u_V)}_V$ is given by, for $\forall v \in \D_V$,
  \begin{align}
    P\left (\mu^{(\pa_V, u_V)}_v = 1 \mid \5v, \5u \right) = \I_{v = v^{(n)}}. \label{eq:sample2}
  \end{align}
  Otherwise, $\mu^{(\pa_V, u_V)}_V$ is drawn uniformly from $\D_V$.
  
  Consider now exogenous parameters $\*\theta$. For every $U \in \*U$, fix $u$. Let $n_u = \sum_{n = 1}^{N}\I_{u^{(n)} = u}$ be the number of instances in $u^{(n)}$ equal to $u$. By the conjugacy of the Dirichlet distribution, the complete conditional of $\theta_u$ is,
\begin{equation}
  \begin{aligned}
  &\left( \theta_1, \dots, \theta_{d_U} \right) \sim \texttt{Dir}\left (\beta^{(1)}_U, \dots, \beta^{(d_U)}_U \right),\\
  &\text{where} \;\; \beta^{(u)}_U = \alpha^{(u)}_U + n_u \;\; \text{for }u = 1, \dots, d_U.
  \end{aligned} \label{eq:sample3}
\end{equation}
\end{itemize}
Doing so eventually produces values drawn from the posterior distribution over $\left (\*\theta, \*\mu, \5U \mid \5V \right)$. Given parameters $\*\theta, \*\mu$, we compute the counterfactual probability $\theta_{\text{ctf}} = P(\*y_{\*x}, \dots, \*z_{\*w})$ following the three-step algorithm in \citep{pearl:2k} which consists of abduction, action, and prediction. Thus computing $\theta_{\text{ctf}}$ from each draw $\*\theta, \*\mu, \5U$ eventually gives us the draw from the posterior distribution $P\left (\theta_{\text{ctf}} \mid \5v \right)$.

\subsection{Collapsed Gibbs Sampling}
We also describe an alternative MCMC algorithm that applies to Dirichlet priors in \Cref{eq:sec5.1}. For $n = 1, \dots, N$, let $\5U_{-n}$ denote the set difference $\5U \setminus \*U^{(n)}$; similarly, we write $\5V_{-n} =\5V \setminus \*V^{(n)}$. Our collapsed Gibbs sampler first iteratively draws values from the conditional distribution over $\left (\*U^{(n)} \mid \5V, \5U_{-n} \right)$ for every $n = 1, \dots, N$ as follows.
\begin{itemize}
  \item {\bf Sampling $P\left (\*u^{(n)} \mid \5v, \5u_{-n} \right)$. }At each iteration, draw $\*U^{(n)}$ from the conditional distribution given by
  \begin{align}
    P\left (\*u^{(n)} \mid \5v, \5u_{-n}\right) \notag \\
    \propto \prod_{V \in \*V \setminus \*Z^{(n)}} &P\left (v^{(n)} \mid \pa^{(n)}_{V}, u^{(n)}_{V}, \5v_{-n}, \5u_{-n} \right) \notag \\
    \prod_{U \in \*U} &P\left (u^{(n)} \mid \5v_{-n}, \5u_{-n}\right). 
  \end{align}
  Among quantities in the above equation, for every $V \in \*V \setminus \*Z^{(n)}$, if there exists an instance $i \neq n$ such that $V \not \in \*Z^{(i)}$ and $\pa^{(i)}_{V} = \pa^{(n)}_{V}$, $u^{(i)}_{V} = u^{(n)}_{V}$,
  \begin{align}
    P\left (v^{(n)} \mid \pa^{(n)}_V, u^{(n)}_V, \5v_{-n}, \5u_{-n} \right) = \I_{v^{(n)} = v^{(i)}}. 
  \end{align}
  Otherwise, the above probability is equal to $1 / |\D_V|$.

  For every $U \in \*U$, let $\bar{u}_{-n}$ be a set of exogenous samples $\left \{ u^{(1)}, \dots, u^{(N)} \right \} \setminus \{u^{(n)}\}$. Let $\{u^*_{1}, \dots, u^*_{K}\}$ denote $K$ unique values that samples in $\bar{u}_{-n}$ take on. The conditional distribution over $\left (U^{(n)} \mid \5V_{-n}, \5U_{-n} \right)$ is given by as follows, for $\alpha_U = \sum_{u = 1}^{d_U} \alpha_U^{(u)}$,
  \begin{align}
    &P\left (u^{(n)} \mid \5v_{-n}, \5u_{-n}\right)\\
    &=
    \begin{dcases}
      \frac{n^*_k + \alpha^{\left(u^*_k \right)}_U}{\alpha_U + N - 1} & \mbox{if } u^{(n)} = u^*_k\\
      \frac{\alpha^{\left(u^{(n)} \right)}_U}{\alpha_U + N - 1} & \mbox{if } u^{(n)} \not \in \{u^*_{1}, \dots, u^*_{K}\}
    \end{dcases} \notag
\end{align}
where $n^*_k = \sum_{i \neq n} \I_{u^{(i)} = u^*_k}$, for $k = 1, \dots, K$, records the number of values $u^{(i)} \in \bar{u}_{-n}$ that are equal to $u^*_k$. 
\end{itemize}
Doing so eventually produces exogenous variables drawn from the posterior distribution of $\left (\5U \mid \5V \right)$. We then sample parameters from the posterior distribution of $\left (\*\theta, \*\mu \mid \5U, \5V \right)$; complete conditional distributions $P\left (\*\mu, \*\theta \mid \5v, \5u \right)$ are given in \Cref{eq:sample2,eq:sample3}. Finally, computing $\theta_{\text{ctf}}$ from each sample $\*\theta, \*\mu$ gives a draw from the posterior $P\left (\theta_{\text{ctf}} \mid \5v \right)$. 

When the cardinality $d_U$ of exogenous domains is high, the collapsed Gibbs sampler described here is more computational efficient than the blocked sampler, since it does not iteratively draw parameters $\*\theta, \*\mu$ in the high-dimensional space. Instead, the collapsed sampler only draws $\*\theta, \*\mu$ once after samples drawn from the distribution of $\left (\5U \mid \5V \right)$ converge. On the other hand, when the cardinality $d_U$ is reasonably low, the blocked Gibbs sampler is preferable since it exhibits better convergence \citep{ishwaran2001gibbs}.

\subsection{Credible Intervals over Counterfactuals}\label{sec4.2}
Given a MCMC sampler, one could bound the counterfactual probability $\theta_{\text{ctf}}$ by computing credible intervals from the posterior distribution $P\left(\theta_{\text{ctf}} \mid \5v \right)$.
\begin{definition}
  Fix $\alpha \in [0, 1)$. A $100(1-\alpha)\%$ credible interval $[l_{\alpha}, r_{\alpha}]$ for $\theta_{\text{ctf}}$ is given by 
\begin{equation}
  \begin{aligned}
    &l_{\alpha} = \sup \left \{x \mid P\left ( \theta_{\text{ctf}} \leq x \mid \5v \right) = \alpha / 2 \right\}, \\
    &r_{\alpha} = \inf \left \{x \mid P\left ( \theta_{\text{ctf}} \leq x \mid \5v \right) = 1 - \alpha / 2 \right\}.
  \end{aligned}
\end{equation}
\end{definition}
For a $100(1-\alpha)\%$ credible interval $[l_{\alpha}, r_{\alpha}]$, any counterfactual probability $\theta_{\text{ctf}}$ that is compatible with observational data $\5v$ lies between the interval $l_{\alpha}$ and $r_{\alpha}$ with probability $1 - \alpha$. Credible intervals have been widely applied for computing bounds over counterfactuals provided with finite observations \citep{imbens2004confidence,vansteelandt2006ignorance,romano2008inference,bugni2010bootstrap,todem2010global}. Let $N_{\*z}$ denote the number of samples in $\5v$ that are drawn from an interventional distribution $P\left(\*v_{\*z}\right)$. Assume that the prior distribution over $\theta_{\text{ctf}}$ has full support over Borel sets in $[0, 1]$. It follows from the law of large numbers that the $100\%$ credible interval $[l_{0}, r_{0}]$ converges to the optimal bound $[l, r]$ in \Cref{eq:opt2} as the sample size $N_{\*z}$ grows (to infinite) for all $\*z \in \3Z$ \citep{chickering:pea97}.

Let $\left \{\theta^{(t)} \right \}_{t = 1}^T$ be $T$ samples drawn from $P\left (\theta_{\text{ctf}} \mid \5v \right)$. One could compute the $100(1-\alpha)\%$ credible interval for $\theta_{\text{ctf}}$ using following estimators \citep{sen1994large}:
\begin{align}
  &\hat{l}_{\alpha}(T) = \theta^{(\ceil{(\alpha/2)T})}, &&\hat{r}_{\alpha}(T) = \theta^{(\ceil{(1 - \alpha/2)T})}, \label{eq:estimate}
\end{align}
where $\theta^{(\ceil{(\alpha/2)T})}, \theta^{(\ceil{(1 - \alpha/2)T})}$ are the $\ceil{(\alpha/2)T}$th smallest and the $\ceil{(1 - \alpha/2)T}$th smallest of $\left \{\theta^{(t)} \right \}$\footnote{For any real $\alpha \in \mathbb{R}$, let $\ceil{\alpha}$ denote the smallest integer $n \in \mathbb{Z}$ larger than $\alpha$, i.e., $\ceil{\alpha} = \min\{n \in \mathbb{Z} \mid n \geq \alpha\}$.}. 
\begin{restatable}{lemma}{lemci}\label{lem:ci}
  Fix $T > 0$ and $\delta \in (0, 1)$. Let function $f(T, \delta) = \sqrt{2T^{-1}\ln(4 / \delta)}$. With probability at least $1 - \delta$, estimators $\hat{l}_{\alpha}(T), \hat{r}_{\alpha}(T)$ for any $\alpha \in [0, 1)$ is bounded by
  \begin{equation}
    \begin{aligned}
      &l_{\alpha - f(T, \delta)} \leq \hat{l}_{\alpha}(T) \leq l_{\alpha + f(T, \delta)}, \\
      &r_{\alpha + f(T, \delta)} \leq \hat{r}_{\alpha}(T) \leq r_{\alpha - f(T, \delta)}.
    \end{aligned}
  \end{equation}
\end{restatable}

\begin{algorithm}[t]
  \caption{\textsc{CredibleInterval}}
  \label{alg:ci}
\begin{algorithmic}[1]
 \STATE {\bfseries Input:} Credible level $\alpha$, tolerance level $\delta, \epsilon$.
 \STATE {\bfseries Output:} An credible interval $[l_{\alpha}, h_{\alpha}]$ for $\theta_{\text{ctf}}$.
 \STATE Draw $T = \ceil{2\epsilon^{-2}\ln(4/\delta)}$ samples $\left \{\theta^{(1)}, \dots, \theta^{(T)} \right\}$ from the posterior distribution $P\left(\theta_{\text{ctf}} \mid \5v \right)$.
 \STATE Return interval $\left[ \hat{l}_{\alpha}(T), \hat{r}_{\alpha}(T) \right]$ (\Cref{eq:estimate}).
\end{algorithmic}
\end{algorithm}

We summarize our algorithm, \textsc{CredibleInterval}, in \Cref{alg:ci}. It takes a credible level $\alpha$ and tolerance levels $\delta, \epsilon$ as inputs. In particular, \textsc{CredibleInterval} repeatedly draw $T \geq \ceil{2\epsilon^{-2}\ln(4/\delta)}$ samples from $P\left(\theta_{\text{ctf}} \mid \5v \right)$. It then computes estimates $\hat{l}_{\alpha}(T), \hat{h}_{\alpha}(T)$ from drawn samples following \Cref{eq:estimate} and return them as the output. 
\begin{restatable}{corollary}{corolci}\label{corol:ci}
  Fix $\delta \in (0, 1)$ and $\epsilon > 0$. With probability at least $1 - \delta$, the interval $[\hat{l}, \hat{r}] = \textsc{CredibleInterval}(\alpha, \delta, \epsilon)$ for any $\alpha \in [0, 1)$ is bounded by $\hat{l} \in \left [l_{\alpha - \epsilon}, l_{\alpha + \epsilon} \right]$ and $\hat{r} \in \left [r_{\alpha + \epsilon}, r_{\alpha - \epsilon} \right]$.
\end{restatable}
\Cref{corol:ci} implies that any counterfactual probability $\theta_{\text{ctf}}$ compatible with the dataset $\5v$ falls between $[\hat{l}, \hat{r}] = \textsc{CredibleInterval}(\alpha, \delta, \epsilon)$ with $P \left( \theta_{\text{ctf}} \in [\hat{l}, \hat{r}] \mid \5v \right) \approx 1 - \alpha \pm \epsilon$. As the tolerance rate $\epsilon \to 0$, $[\hat{l}, \hat{r}]$ converges to a $100(1-\alpha)\%$ credible interval with high probability. 

\begin{figure*}[t]
\centering
\null
\begin{subfigure}{0.245\linewidth}\centering
  \includegraphics[width=\linewidth]{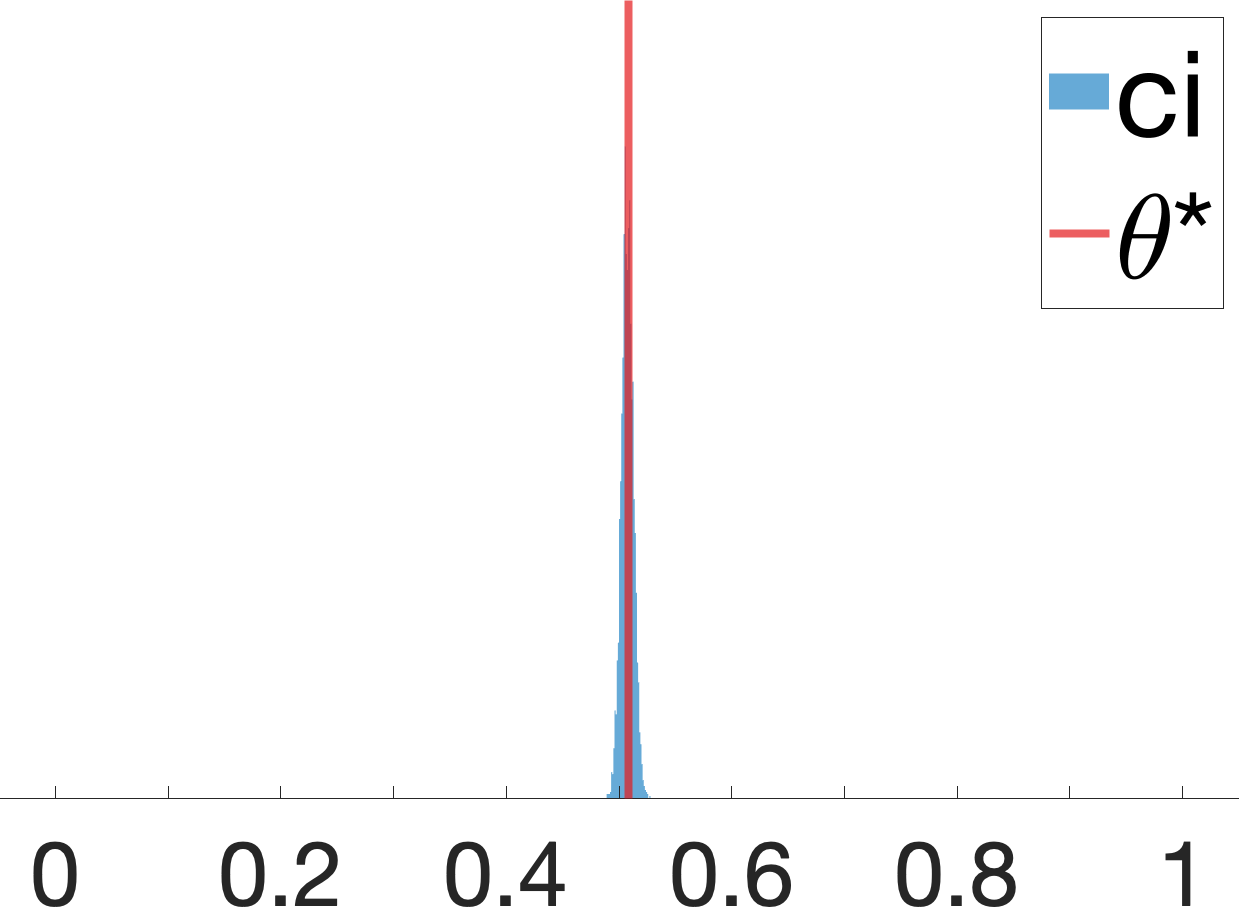}
\caption{Frontdoor}
\label{fig7a}
\end{subfigure}\hfill
\begin{subfigure}{0.245\linewidth}\centering
  \includegraphics[width=\linewidth]{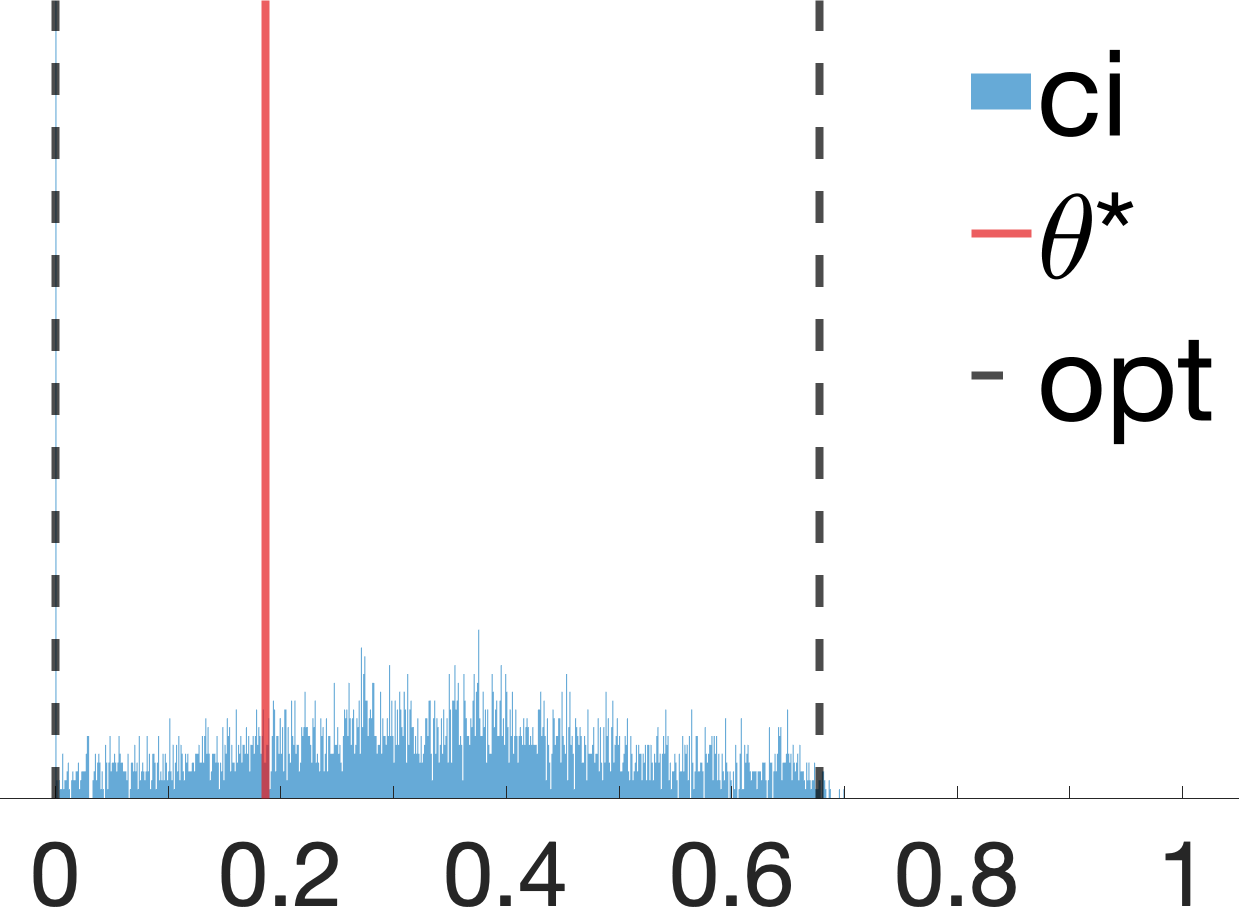}
\caption{PNS}
\label{fig7b}
\end{subfigure}\hfill
\begin{subfigure}{0.245\linewidth}\centering
  \includegraphics[width=\linewidth]{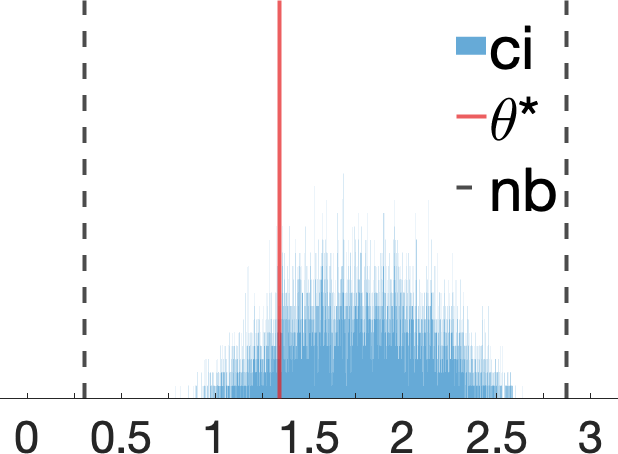}
\caption{IST}
\label{fig7c}
\end{subfigure}\hfill
\begin{subfigure}{0.245\linewidth}\centering
  \includegraphics[width=\linewidth]{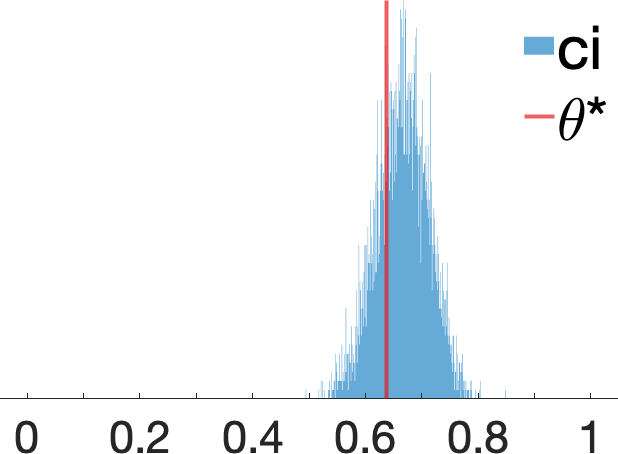}
\caption{Obs. + Exp.}
\label{fig7d}
\end{subfigure}\null
\caption{Simulation results for Experiments 1-4. For all plots (\subref{fig7a} - \subref{fig7d}), \textit{ci} represents our proposed algorithm; $\theta^*$ is the actual counterfactual probability; \textit{opt} is the optimal asymptotic bounds (if exists); \textit{nb} stands for the natural bounds \citep{manski:90}.}
\label{fig7}
\end{figure*}

\section{Simulations and Experiments}
We demonstrate our algorithms on various synthetic and real datasets in different causal diagrams. Overall, we found that simulation results support our findings and the proposed bounding strategy consistently dominates state-of-art algorithms. When target probabilities are identifiable (Experiment 1), our bounds collapse to the actual 
counterfactual probabilities. For non-identifiable settings, our algorithm obtains sharp asymptotic bounds when the closed-form solutions already exist (Experiments 2); and obtains novel counterfactual bounds in other more general cases which consistently improve over existing strategies (Experiment 3 \& 4).

In all experiments, we evaluate our proposed strategy using credible intervals (\textit{ci}). We draw at least $4 \times 10^3$ samples from the posterior distribution $P\left(\theta_{\text{ctf}} \mid \5v \right)$ over the target counterfactual. This allows us to compute $100\%$ credible interval over $\theta_{\text{ctf}}$ within error $\epsilon = 0.05$, with probability at least $1 - \delta = 0.95$. As the baseline, we include the actual counterfactual probability $\theta^*$. We refer readers to Appendix C for more details on simulations and additional experiments with other causal diagrams and datasets.

\paragraph{Experiment 1: Frontdoor Graph}
Consider the ``Frontdoor'' graph described in \Cref{fig1c} where $X, Y, W$ are binary variables in $\{0, 1\}$; $U_1, U_2 \in \3R$. In this case, any interventional probability $P(y_x)$ is identifiable from the observational distribution $P(X, W, Y)$ through the frontdoor adjustment \citep[Thm.~3.3.4]{pearl:2k}. We collect $N = 10^4$ observational samples $\5v = \{x^{(n)}, y^{(n)}, w^{(n)}\}_{n = 1}^N$ from a synthetic SCM instance. \Cref{fig7a} shows samples drawn from the posterior distribution $\left ( P(Y_{x = 0} = 1) \mid \5v \right)$. The analysis reveals that these samples collapse to the actual interventional probability $P(Y_{x = 0} = 1) = 0.5085$, which confirms the identifiability of $P(y_x)$ in the ``frontdoor'' graph.

\paragraph{Experiment 2: Probability of Necessity and Sufficiency (PNS)} Consider the ``Bow'' diagram in \Cref{fig1d} where $X, Y \in \{0, 1\}$ and $U \in \3R$. We study the problem of evaluating the \emph{probability of necessity and sufficiency} $P(Y_{x=1}=1, Y_{x=0}=0)$ from the observational distribution $P(X, Y)$. 
The sharp bound for $P(Y_{x=1}=1, Y_{x=0}=0)$ from $P(X, Y)$ was introduced in \citep{tian:pea00b} (labelled as \textit{opt}). We collect $N = 10^3$ observational samples $\5v = \{x^{(n)}, y^{(n)} \}_{n = 1}^N$ from a randomly generated SCM instance. \Cref{fig7c} shows samples drawn from the posterior distribution over $\left ( P(Y_{x=1}=1, Y_{x=0}=0) \mid \5v \right)$. 
The analysis reveals that the $100\%$ credible interval (\textit{ci}) matches the optimal PNS bound $l = 0, r = 0.6775$ over the actual counterfactual probability $P(Y_{x=1}=1, Y_{x=0}=0) = 0.1867$. 

\paragraph{Experiment 3: International Stroke Trials (IST)}IST was a large, randomized, open trial of up to $14$ days of antithrombotic therapy after stroke onset \citep{carolei1997international}. 
In particular, the treatment $X$ is a pair $(i, j)$ where $i \in \{0, 1\}$ stands for aspirin allocation; $j \in \{0, 1, 2\}$ stands for heparin allocation. The primary outcome $Y \in \{0, \dots, 3\} $ is the health of the patient $6$ months after the treatment.
To emulate the presence of unobserved confounding, we filter the experimental data  following a procedure in \citep{kallus2018confounding}. Doing so allows us to obtain $N = 10^3$ synthetic observational samples $\5v = \{x^{(n)}, y^{(n)}, z^{(n)}\}_{n = 1}^N$ that are compatible with the ``IV'' diagram of \Cref{fig1a} where $Z \in \{0, \dots, 9\}$. We are interested in evaluating the treatment effect $E[Y_{x = (1, 0)}]$ for only assigning aspirin $X = (1, 0)$. As a baseline, we also include the natural bound \citep{manski:90} estimated at the $95\%$ confidence level (\textit{nb}) \citep{zhang2021bounding}. The analysis (\Cref{fig7c}) reveals that both algorithms achieve effective bounds containing target causal effect $E[Y_{x = (1, 0)}] = 1.3418$. The $100\%$ credible interval is  $l_{\textit{ci}}= 0.4363, r_{\textit{ci}}= 2.3162$, which  improves over the existing strategy ($l_{\textit{nb}}= 0.3050, r_{\textit{nb}}= 2.8686$). 

\paragraph{Experiment 4: Obs. + Exp.}Consider the causal diagram in \Cref{fig1b} where $X, Y, Z \in \{0, \dots, 9\}$ and $U_1, U_2 \in \3R$. We are interested in evaluating counterfactual probabilities $P\left (z, x_{z'}, y_{x'} \right)$ from the observational distribution $P(X, Y, Z)$ and a collection of interventional distributions $P(X_{z}, Y_{z})$ induced by interventions $\doo(z)$ for $z = 0, \dots, 9$. We collect $N = 10^3$ samples $\5v = \{x^{(n)}, y^{(n)}, z^{(n)} \}_{n = 1}^N$ from a SCM instance of \Cref{fig1b} where each sample $X^{(n)}, Y^{(n)}, Z^{(n)}$ is an independent draw from $P(X, Y, Z)$ or $P(X_{z}, Y_{z})$. To address the challenge of the high-dimensional exogenous domains, we apply the proposed collapsed Gibbs sampler to obtain samples from the posterior distribution $\left ( P\left (Z + X_{z = 0} + Y_{x=0} \geq 14 \right) \mid \5v \right)$. Simulation results are shown in \Cref{fig7d}. The analysis reveals that our proposed approach is able to achieve an effective bound that contains the actual counterfactual probability $P\left (Z + X_{z = 0} + Y_{x=0} \geq 14 \right) = 0.6378$. The $100\%$ credible interval (\textit{ci}) is equal to $l= 0.4949, r= 0.8482$. To our best knowledge, no existing strategy is applicable for this setting.

\section{Conclusion}
This paper investigated the problem of partial identification of counterfactual distributions, which concerns with bounding  counterfactual probabilities from an arbitrary combination of observational and experimental data, provided with a causal diagram encoding qualitative assumptions about the data-generating process. We introduced a special parametric family of SCMs with discrete exogenous variables, taking values from a finite set of unobserved states, and showed that it could represent \emph{all} counterfactual distributions (over finite observed variables) in \emph{any} causal diagram. Using this result, we reduced the partial identification problem into a polynomial program and developed a novel algorithm to approximate the optimal asymptotic bounds over target counterfactual probabilities from finite samples obtained through arbitrary observations and experiments.

\clearpage
\appendix
\section{A. On the Expressive Power of Canonical Structural Causal Models} 
In this section, we provide the proof for \Cref{thm:cfinite} which establishes the expressive power of discrete SCMs in representing counterfactual distributions in an arbitrary causal diagram containing observed variables with finite domains.

Recall that $\2M(\G)$ and $\2N(\G)$ in \Cref{thm:cfinite} are collections of all SCMs and discrete SCMs (thereafter, canonical) compatible with a causal diagram $\G$ respectively. Since $\2N(\G) \subset \2M(\G)$, the reverse direction of \Cref{eq:ctf_eq} is self-evident. The main challenge here is to prove the other direction. That is, given an SCM $M \in \2M(\G)$ with arbitrary exogenous domains, we want to construct a discrete SCM $N \in \2N(\G)$ with finite exogenous domains such that $N$ and $M$ are both compatible with the same causal diagram $\G$ and induces the same set of counterfactual distributions $\*P^*$. 

To illustrate the idea of this constructive proof, consider as an example the ``Bow'' graph in \Cref{fig1d} where $X, Y$ are binary variables in $\{0, 1\}$; the exogenous variable $U$ takes values in the real numbers $\3R$. Let domain $\D_X$ be ordered by $h_X^{(1)} = 0$ and $h_X^{(2)} = 1$. We denote by $\D_X \mapsto \D_Y$ the set of all functions mapping from domains of $X$ to $Y$, i.e.,
\begin{equation}
\begin{aligned}
  &h_Y^{(1)}(x) = 0, &&h_Y^{(2)}(x) = x, \\
  &h_Y^{(3)}(x) = \neg x, &&h_Y^{(4)}(x) = 1.
\end{aligned}
\end{equation}
Consider the following families of SCMs:
\begin{enumerate}
  \item $\2M$ is the set of all SCMs compatible with the ``Bow'' graph in \Cref{fig1d}.
  \item $\2N$ is the set of all discrete SCMs compatible with the ``Bow'' graph in \Cref{fig1d} with the cardinality $|\D_U| = 8$.
\end{enumerate}
Our goal is to prove that $\2M$ and $\2N$ are counterfactually equivalent for binary $X, Y \in \{0, 1\}$. Since intervening on $Y$ has no causal effect on $X$ \citep{galles:pea98}, it is sufficient to show that for any SCM $M \in \2M$, one could construct a discrete SCM $N \in \2N$ so that 
\begin{align}
  P_M \left (x, y_{x = 0}, y_{x = 1} \right) = P_N \left (x, y_{x = 0}, y_{x = 1} \right). \label{eq:ctf_N}
\end{align} 
The construction procedure is described as follows. Let the exogenous variable $U$ in $N$ be a pair $(U_X, U_Y)$ where $U_X \in \{1, 2\}$ and $U_Y \in \{1, \dots, 4\}$. Values of $X$ and $Y$ are given by the following functions, respectively,
\begin{align}
  &x \gets h_X^{(u_X)}, &&y \gets h_Y^{(u_Y)}(x).
\end{align} 
It is verifiable that in such $N$, the counterfactual distribution $P(x, y_{x = 0}, y_{x = 1})$ is given by, for $\forall i, j, k \in \{0, 1\}$,
\begin{equation}
\begin{aligned}
  &P_N(X = i, Y_{x = 0} = j, Y_{x = 1} = k) \\
  &= P_N(U_X = i + 1, U_Y = 2j + k + 1). 
\end{aligned} \label{eq:ctf_N1}
\end{equation}
For any SCM $M \in \1M$, we define the exogenous distribution $P_N(u_X, u_Y)$ of the discrete SCM $N$ as, for $\forall i, j, k \in \{0, 1\}$,
\begin{equation}
\begin{aligned}
  &P_N(U_X = i + 1, U_Y = 2j + k + 1) \\
  &= P_M(X = i, Y_{x = 0} = j, Y_{x = 1} = k). 
\end{aligned} \label{eq:ctf_N2}
\end{equation}
It follows immediately from \Cref{eq:ctf_N1,eq:ctf_N2} that $M$ and $N$ induce the same counterfactual distribution $P(x, y_{x = 0}, y_{x = 1})$, i.e., the condition in \Cref{eq:ctf_N} holds. This means that when inferring counterfactual distributions in the ``Bow'' graph of \Cref{fig1d} with binary $X, Y$, we could assume that the exogenous variable $U$ is discrete and takes values in the domain $\left \{1, \dots, 8\right\}$, without loss of generality.

Our goal is to generalize the construction described above to any SCMs compatible with an arbitrary causal diagram. The remainder of this section is organized as follows. In Appendix A.1, we introduce a general canonical partitioning \citep{balke:pea94b} over exogenous domains for any SCMs with discrete endogenous variables. This allows us to write counterfactual distributions as functions of products of probabilities assigned to intersections of canonical partitions in every c-component. Appendix A.2 shows that probabilities over canonical partitions could be represented using discrete exogenous variables taking values in finite domains. This allows us to prove \Cref{thm:cfinite} for any causal diagram in the theoretical framework of measure-theoretic probability. Finally, we describe in Appendix A.3 a more fine-grained decomposition for canonical partitions, which provides intuitive explanations for the discretization procedure.

\subsection{A.1 Canonical Partitions of Exogenous Domains}
For every endogenous variable $V \in \*V$, let $\D_{\PA_V} \mapsto \D_V$ denote the hypothesis class containing all functions mapping from domains of $\PA_V$ to $V$. Since $\*V$ are discrete variables with finite domains, the cardinality of the class $\D_{\PA_V} \mapsto \D_V$ must be also finite. Given any configuration $U_V = u_V$, the induced function $f_V(\cdot, u_V)$ must correspond to a unique element in the hypothesis class $\D_{\PA_V} \mapsto \D_V$. Such mappings lead to a finite partition over the exogenous domain $\D_{U_V}$. 
\begin{definition}
  For an SCM $M = \langle \*V, \*U, \2F, P\rangle$, for every $V \in \*V$, let functions in $\D_{\PA_V} \mapsto \D_V$ be ordered by $\left \{h_V^{(i)} \mid i \in \*I_V \right \}$ where $\*I_V = \{1, \dots, m_V\}, m_V = |\D_{\PA_V} \mapsto \D_V|$. A \emph{equivalence class} $\1U_V^{(i)}$ for function $h_V^{(i)}$, $i = 1, \dots, m_V$, is a subset in $\D_{U_V}$ such that 
  \begin{align}
      \1U_V^{(i)} = \left \{u_V \in \D_{U_V} \mid f_V(\cdot, u_V) = h_V^{(i)}\right\}.
  \end{align}
\end{definition}
\begin{definition}[Canonical Partition]\label{def:canonical}
  For an SCM $M = \langle \*V, \*U, \2F, P\rangle$, $\left \{\1U_V^{(i)} \mid i \in \*I_V \right \}$ is the \emph{canonical partition} over exogenous domain $\D_{U_V}$ for every $V \in \*V$.
\end{definition}
\Cref{def:canonical} extends the canonical partition in \citep{balke:pea94b} which was designed for binary variables $X, Y, Z \in \{0, 1\}$ in the ``IV'' diagram of \Cref{fig1a}.

As exogenous variables $U_V$ vary along its domain, regardless of how complex the variation is, its only effect is to switch the functional relationship between $\Pa_V$ and $V$ among elements in class $\D_{\PA_V} \mapsto \D_V$. Formally,
\begin{restatable}{lemma}{lemcp}\label{lem:cp}
    For an SCM $M = \tuple{\*V, \*U, \2F, P}$, for each $V \in \*V$, function $f_V \in \2F$ could be decomposed as:
    \begin{align}
        f_V(\pa_V, u_V) = \sum_{i \in \*I_V} h^{(i)}_{V}(\pa_V) \I_{u_V \in \1U_{V}^{(i)}}. \label{eq:cp1}
    \end{align}
\end{restatable}
\begin{proof}
  By the definition of canonical partitions (\Cref{def:canonical}), for every $i = 1, \dots, m_V$, fix any $u_V \in \1U_{V}^{(i)}$. We must have $f_V(\cdot, u_V) = h^{(i)}_V(\cdot)$. This implies $f_V(\pa_V, u_V) = h^{(i)}_V(\pa_V)$ for any $\PA_V = \pa_V$. Recall that $\left \{ \1U_{V}^{(i)} \mid i = 1, \dots, m_V \right \}$ forms a partition over exogenous domains $\D_{U_V}$. Given the same $\pa_V, u_V$, the r.h.s. of \Cref{eq:cp1} must equate to $h^{(i)}_V(\pa_V)$, which completes the proof.
\end{proof}
As an example, consider an SCM $M$ associated with the ``IV'' graph of \Cref{fig1a} where $X, Y, Z$ are binary variables contained in $\{0, 1\}$; $U_1, U_2$ are continuous variables drawn uniformly from the interval $[0, 3]$. Values of $X, Y, Z$ are decided by functions defined as follows, respectively,
\begin{equation}
  \begin{aligned}
    &x \gets f_X(z, u_2) = \I_{z \leq u_2 \leq z + 2}, \\
    &y \gets f_Y(x, u_2) = \I_{u_2 < x} + \I_{u_2 > x + 2},\\
    &z \gets f_Z(u_1) = \I_{u_1 \leq 1.5},
  \end{aligned} \label{eq:fxyz}
\end{equation}
We show in \Cref{fig:app_a1} the graphical representation of canonical partitions induced by functions $f_X, f_Y$ and $f_Z$ respectively. A detailed description is provided in \Cref{table:app_a1}.
It follows from the decomposition of \Cref{lem:cp} that functions $f_X, f_Y, f_Z$ in \Cref{eq:fxyz} could be written as follows:
\begin{align*}
    f_X(z, u_2) &=  \I_{u_2 \in [0, 1)} \neg z + \I_{u_2 \in [1, 2]} 1 + \I_{u_2 \in (2, 3]} z\\
    f_Y(x, u_2) &= \I_{u_2 \in [0, 1)} x + \I_{u_2 \in [1, 2]} 0 + \I_{u_2 \in (2, 3]} \neg x,\\
    f_Z(u_1) &= \I_{u_1 \in [0, 1.5]} 1 + \I_{u_1 \in (1.5, 3]} 0.
\end{align*}
Let $\*I$ denote the product of indexing sets $\vartimes_{V \in \*V} \*I_V$. For any index $\*i \in \*I$, we use $i_V$ to represent the element in $\*i$ restricted to $V \in \*V$. We omit the subscript $V$ when it is obvious; therefore, $\1U^{(i)}_V = \1U^{(i_V)}_V$, $h^{(i)}_V = h^{(i_V)}_V$. Our next result establishes a universal decomposition of counterfactual distributions in any SCM using canonical partitions.
\begin{restatable}{lemma}{lemctfcp1}\label{lem:ctf_cp1}
  For an SCM $M = \tuple{\*V, \*U, \2F, P}$, for any $\*Y, \dots, \*Z, \*X, \dots, \*W \subseteq \*V$\footnote{For an arbitrary subset $\1U \subseteq \D_{\*U}$, we will consistently use $P\left( \1U \right)$ as a shorthand for the probability $P\left( \*U \in \1U\right)$.},
  \begin{equation}    
  \begin{aligned}
    &P \left (\*y_{\*x}, \dots, \*z_{\*w} \right)\\
    &= \sum_{\*i \in \*I} \I_{\*Y_{\*x}(\*i) = \*y, \dots, \*Z_{\*w}(\*i) = \*z} P\left( \bigcap_{V \in \*V} \1U_V^{(i)} \right), 
  \end{aligned}\label{eq:ctf_cp1}
  \end{equation}
  where variables of the form $\*Y_{\*x}(\*i) = \left \{Y_{\*x}(\*i) \mid \forall Y \in \*Y \right \}$; every $Y_{\*x}(\*i)$ is recursively defined as:
  \begin{align}
  Y_{\*x}(\*i) = 
    \begin{dcases}
      \*x_Y &\mbox{if }Y \in \*X \\
      h_Y^{(i)} \left ( \left( \PA_{Y}\right)_{\*x}(\*i) \right) &\mbox{otherwise}
    \end{dcases} \label{eq:ctf_cp2}
  \end{align}
\end{restatable}

\begin{figure}[t]
  \centering
  \begin{subfigure}{0.95\linewidth}\centering
  \resizebox{\linewidth}{!}{
  \begin{tikzpicture}
    \draw[->, >={Latex}] (-0.5,0) -- (3.5,0) node[below] {$u_2$}; 

    \fill[fill=lightred] (2,0.01) rectangle (3,0.5);
    \fill[fill=lightblue] (0,0.01) rectangle (1,0.5);
    \fill[fill=lightgreen] (1,0.01) rectangle (2,0.5);

    \node (r1) at (2.5, 0.33) {\scalebox{0.7}{$\1U^{(2)}_X$}};	
    \node (h1) at (2.5, 0.1) {\scalebox{0.7}{$x \gets z$}};

    \node (r1) at (0.5, 0.33) {\scalebox{0.7}{$\1U^{(3)}_X$}};	
    \node (h1) at (0.5, 0.1) {\scalebox{0.7}{$x \gets \neg z$}};

    \node (r1) at (1.5, 0.33) {\scalebox{0.7}{$\1U^{(4)}_X$}};	
    \node (h1) at (1.5, 0.1) {\scalebox{0.7}{$x \gets 1$}};

    \draw[thick, dashed, -] (0, 0) -- (0, 0.5);
    \draw[thick, dashed, -] (1, 0) -- (1, 0.5);
    \draw[thick, dashed, -] (2, 0) -- (2, 0.5);
    \draw[thick, dashed, -] (3, 0) -- (3, 0.5);

    \node [below] at (0, 0) {\scalebox{0.7}{0}};	
    \node [below] at (1, 0) {\scalebox{0.7}{1}};	
    \node [below] at (2, 0) {\scalebox{0.7}{2}};
    \node [below] at (3, 0) {\scalebox{0.7}{3}};	
  \end{tikzpicture}
  }
  \caption{$x \gets f_X(z, u_2)$}
  \label{fig:app_a1a}
\end{subfigure}\hfill

\begin{subfigure}{0.95\linewidth}\centering
  \resizebox{\linewidth}{!}{
  \begin{tikzpicture}
    \draw[->, >={Latex}] (-0.5,0) -- (3.5,0) node[below] {$u_2$}; 

    \fill[fill=lightblue] (2,0.01) rectangle (3,0.5);
    \fill[fill=lightred] (0,0.01) rectangle (1,0.5);
    \fill[fill=none] (1,0.01) rectangle (2,0.5);

    \node (r1) at (2.5, 0.33) {\scalebox{0.7}{$\1U^{(3)}_Y$}};	
    \node (h1) at (2.5, 0.1) {\scalebox{0.7}{$y \gets \neg x$}};

    \node (r1) at (0.5, 0.33) {\scalebox{0.7}{$\1U^{(2)}_Y$}};	
    \node (h1) at (0.5, 0.1) {\scalebox{0.7}{$y \gets x$}};

    \node (r1) at (1.5, 0.33) {\scalebox{0.7}{$\1U^{(1)}_Y$}};	
    \node (h1) at (1.5, 0.1) {\scalebox{0.7}{$y \gets 0$}};

    \draw[thick, dashed, -] (0, 0) -- (0, 0.5);
    \draw[thick, dashed, -] (1, 0) -- (1, 0.5);
    \draw[thick, dashed, -] (2, 0) -- (2, 0.5);
    \draw[thick, dashed, -] (3, 0) -- (3, 0.5);

    \node [below] at (0, 0) {\scalebox{0.7}{0}};	
    \node [below] at (1, 0) {\scalebox{0.7}{1}};	
    \node [below] at (2, 0) {\scalebox{0.7}{2}};
    \node [below] at (3, 0) {\scalebox{0.7}{3}};	
  \end{tikzpicture}
  }
  \caption{$y \gets f_Y(x, u_2)$}
  \label{fig:app_a1b}
  \end{subfigure}\hfill

\begin{subfigure}{0.95\linewidth}\centering
  \resizebox{\linewidth}{!}{
  \begin{tikzpicture}
    \draw[->, >={Latex}] (-0.5,0) -- (3.5,0) node[below] {$u_1$}; 

    \fill[fill=lightred] (0,0.01) rectangle (1.5,0.5);

    \node (r1) at (2.25, 0.33) {\scalebox{0.7}{$\1U^{(1)}_Z$}};	
    \node (h1) at (2.25, 0.1) {\scalebox{0.7}{$z \gets 0$}};

    \node (r1) at (0.75, 0.33) {\scalebox{0.7}{$\1U^{(2)}_Z$}};	
    \node (h1) at (0.75, 0.1) {\scalebox{0.7}{$z \gets 1$}};

    \draw[thick, dashed, -] (0, 0) -- (0, 0.5);
    \draw[thick, dashed, -] (1.5, 0) -- (1.5, 0.5);
    \draw[thick, dashed, -] (3, 0) -- (3, 0.5);

    \node [below] at (0, 0) {\scalebox{0.7}{0}};
    \node [below] at (1.5, 0) {\scalebox{0.7}{1.5}};	
    \node [below] at (3, 0) {\scalebox{0.7}{3}};	
  \end{tikzpicture}
  }
  \caption{$z \gets f_Z(u_1)$}
  \label{fig:app_a1c}
  \end{subfigure}
  \caption{Canonical partitions for exogenous domains over $U_1, U_2$ induced by functions of $X, Y, Z$ defined in \Cref{eq:fxyz}.}
  \label{fig:app_a1}
\end{figure}
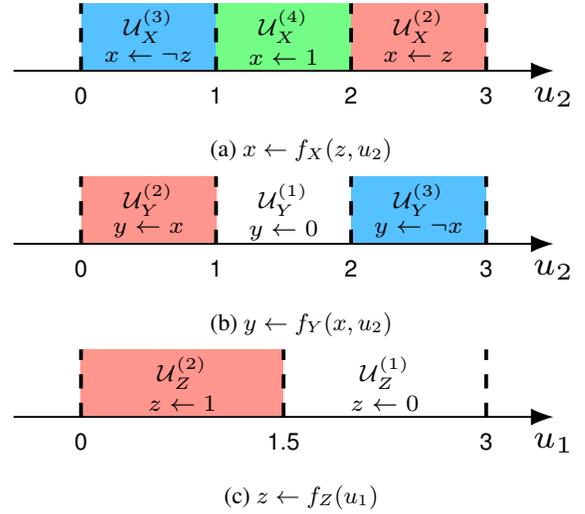
\begin{table}[t]
  \centering
  \begin{subfigure}{0.95\linewidth}\centering
  \begin{tabular}{|l||*{3}{p{1.7cm}|}} \hline
    &$0 \leq U_2 < 1$ &$1 \leq U_2 \leq 2$ & $2 < U_2 \leq 3$ \\\hline
    $Z = 0$ & $X = 1$ & $X = 1$ & $X = 0$ \\\hline
    $Z = 1$ & $X = 0$ & $X = 1$ & $X = 1$ \\\hline
    \end{tabular}
    \caption{$x \gets f_X(z, u_2)$}
    \label{table:app_a1a}\hfill
  \end{subfigure}
  \begin{subfigure}{0.95\linewidth}\centering
  \begin{tabular}{|l||*{3}{p{1.7cm}|}} \hline
    &$0 \leq U_2 < 1$ &$1 \leq U_2 \leq 2$ & $2 < U_2 \leq 3$ \\\hline
    $X = 0$ & $Y = 0$ & $Y = 0$ & $Y = 1$ \\\hline
    $X = 1$ & $Y = 1$ & $Y = 0$ & $Y = 0$ \\\hline
    \end{tabular}
    \caption{$y \gets f_Y(x, u_2)$}
    \label{table:app_a1b}\hfill
  \end{subfigure}
  \begin{subfigure}{0.95\linewidth}\centering
  \begin{tabular}{|*{2}{p{2cm}|}} \hline
    $0 \leq U_1 < 1.5$ &$1.5 \leq U_1 \leq 3$ \\\hline
    $Z = 1$ & $Z = 0$ \\\hline
    \end{tabular}
    \caption{$z \gets f_Z(u_1)$}
    \label{table:app_a1c}
  \end{subfigure}
  \caption{Canonical partitions for exogenous domains over $U_1, U_2$ induced by functions of $X, Y, Z$ defined in \Cref{eq:fxyz}.}
  \label{table:app_a1}
\end{table}

\begin{proof}
  We will first prove the following claims: for arbitrary subsets $\*Y, \*X \subseteq \*V$, for any $\*u, \*x, \*y$, 
  \begin{align}
      \I_{\*Y_{\*x}(\*u) = \*y} = \sum_{\*i \in \*I}\I_{\*Y_{\*x}(\*i) = \*y} \prod_{V \in \*V}\I_{u_V \in \1U^{(i)}_V}. \label{eq:ctfcp3}
  \end{align}
  Let $\G_{\overline{\*X}}$ be a subgraph obtained from the causal diagram $\G$ by removing all incoming arrows of $\*X$. We will prove \Cref{eq:ctfcp3} by induction on $n = \max_{Y \in \*Y} \left|\An(Y)_{\G_{\overline{\*X}}} \right|$.
  \paragraph{Base Case $n = 1$.} Recall that an intervention $\doo(\*x)$ set values of variables $\*X$ as constants $\*x$. For any $Y \in \*X \cap \*Y$, let $\*x_Y$ be the values assigned to $Y$ in $\*x$. It is verifiable that
  \begin{align}
    \I_{Y_{\*x}(\*u) = y} = \I_{y = \*x_Y} \label{eq:ctf_cp1.1}
  \end{align}
  As for every variable $Y \in \*Y \setminus \*X$, we must have its parent nodes $\PA_Y = \emptyset$ since $n = 1$. This implies
  \begin{align}
      \I_{Y_{\*x}(\*u) = y} = \I_{f_Y(u_Y) = y}=\sum_{i \in \*I_Y} \I_{h^{(i)}_Y = y} \I_{u_Y \in \1U^{(i)}_Y} \label{eq:ctf_cp1.2}
  \end{align}
  The last step follows from the decomposition in \Cref{lem:cp}. \Cref{eq:ctf_cp1.1,eq:ctf_cp1.2} together imply that
  \begin{align*}
      &\I_{\*Y_{\*x}(\*u) = \*y} \\
      &= \sum_{i \in \*I} \prod_{Y\in \*Y \cap \*X}\I_{y = \*x_Y}\prod_{Y \in (\*Y \setminus \*X)} \I_{h^{(i)}_Y = y} \prod_{V \in \*V}\I_{u_V \in \1U^{(i)}_V} \\
      &= \sum_{i \in \*I} \I_{\*Y_{\*x}(\*i) = \*y} \prod_{V \in \*V}\I_{u_V \in \1U^{(i)}_V}. 
  \end{align*}
  The last step follows from the definition of variables $\*Y_{\*x}(\*i)$ in \Cref{eq:ctf_cp2} given an index $\*i \in \*I$. 

  \paragraph{Induction Case $n = k+1$.} Assume that \Cref{eq:ctfcp3} holds for $n = k$. We will prove for the case $n = k+1$. For every $Y \in \*X \cap \*Y$, $\I_{Y_{\*x}(\*u) = y}$ is given in \Cref{eq:ctf_cp1.1}. For every $Y \in \*Y \setminus \*X$, the decomposition in \Cref{lem:cp} implies:
  \begin{align*}
      &\I_{Y_{\*x}(\*u) = y} \\
      &= \I_{f_Y(\left( \PA_{Y}\right)_{\*x}(\*u), u_Y) = y} \\
      &= \I\left \{y = \sum_{i \in \*I_Y}h^{(i)}_Y(\left( \PA_{Y}\right)_{\*x}(\*u)) \I_{u_Y \in \1U^{(i)}_Y} \right\} \\
      &= \sum_{i \in \*I_Y} \I_{h^{(i)}_Y(\left( \PA_{Y}\right)_{\*x}(\*u)) = y} \I_{u_Y \in \1U^{(i)}_Y}\\
      &= \sum_{i \in \*I_Y} \sum_{\pa_Y} \I_{h^{(i)}_Y(\pa_Y) = y} \I_{\left( \PA_{Y}\right)_{\*x}(\*u) = \pa_Y} \I_{u_Y \in \1U^{(i)}_Y}.
  \end{align*}
  The last step hold by conditioning on events $\left( \PA_{Y}\right)_{\*x}(\*u) = \pa_Y$, $\forall \pa_Y \in \D_{\PA_Y}$. Since we assume \Cref{eq:ctfcp3} holds for Case $n = k$, the above equation could be further written as
  \begin{align*}
      \I_{Y_{\*x}(\*u) = y} = & \sum_{i \in \*I_Y} \sum_{\pa_Y}\I_{h^{(i)}_Y(\pa_Y) = y} \I_{u_Y \in \1U^{(i)}_Y}\\
      &\cdot \sum_{i \in \*I} \I_{\left( \PA_{Y}\right)_{\*x}(\*u) = \pa_Y} \prod_{V \in \*V}\I_{u_V \in \1U^{(i)}_V}
  \end{align*}
  A few simplification gives:
  \begin{align}
    &\I_{Y_{\*x}(\*u) = y} \notag \\
    &= \sum_{i \in \*I} \sum_{\pa_Y} \I_{h^{(i)}_Y(\pa_Y) = y} \I_{\left( \PA_{Y}\right)_{\*x}(\*u)= \pa_Y}  \prod_{V \in \*V} \I_{u_V \in \1U^{(i)}_V} \notag \\
    &= \sum_{i \in \*I} \I_{h^{(i)}_Y(\left( \PA_{Y}\right)_{\*x}(\*u) ) = y} \prod_{V \in \*V}\I_{u_V \in \1U^{(i)}_V}. \label{eq:ctf_cp1.3}
  \end{align}
  \Cref{eq:ctf_cp1.1,eq:ctf_cp1.3} together imply that
  \begin{align*}
      &\I_{\*Y_{\*x}(\*u) = \*y} \\
      &= \sum_{i \in \*I}  \left (\prod_{Y \in \*Y \cap \*X} \I_{y = \*x_Y} \prod_{Y \in (\*Y \setminus \*X)} \I_{h^{(i)}_Y(\left( \PA_{Y}\right)_{\*x}(\*u)) = y} \right) \\
      &\cdot \prod_{V \in \*V}\I_{u_V \in \1U^{(i)}_V} \\
      &= \sum_{i \in \*I} \I_{\*Y_{\*x}(\*i) = \*y} \prod_{V \in \*V}\I_{u_V \in \1U^{(i)}_V}.
  \end{align*}
  Again, the last step follows from the definition of variables $\*Y_{\*x}(\*i)$ in \Cref{eq:ctf_cp2} given an index $\*i \in \*I$. 

  We are now ready to prove \Cref{eq:ctf_cp1}. The statement of \Cref{eq:ctfcp3} implies that for any $\*Y, \dots, \*Z, \*X, \dots, \*W \subseteq \*V$, 
  \begin{align*}
      &P \left (\*y_{\*x}, \dots, \*z_{\*w} \right) \\
      &= \int_{\D_{\*U}} \I_{\*Y_{\*x}(\*u) = \*y, \dots, \*Z_{\*w}(\*u) = \*z} dP(\*u)\\
      &= \int_{\D_{\*U}} \left (\sum_{\*i \in \*I}\I_{\*Y_{\*x}(\*i) = \*y} \prod_{V \in \*V}\I_{u_V \in \1U^{(i)}_V}\right) \wedge \\
      &\cdots \wedge \left(\sum_{\*i \in \*I} \I_{\*Z_{\*w}(\*i) = \*z} \prod_{V \in \*V}\I_{u_V \in \1U^{(i)}_V} \right)dP(\*u)
  \end{align*}
  Simplifying the above equation gives:
  \begin{align*}
    &P \left (\*y_{\*x}, \dots, \*z_{\*w} \right)\\
    &=\int_{\D_{\*U}} \sum_{\*i \in \*I} \I_{\*Y_{\*x}(\*i) = \*y} \wedge \dots \wedge \I_{\*Z_{\*w}(\*i) = \*z} \prod_{V \in \*V}\I_{u_V \in \1U^{(i)}_V} dP(\*u)\\
    &=\sum_{\*i \in \*I} \I_{\*Y_{\*x}(\*i) = \*y} \wedge \dots \wedge \I_{\*Z_{\*w}(\*i) = \*z}\int_{\D_{\*U}} \prod_{V \in \*V}\I_{u_V \in \1U^{(i)}_V} dP(\*u)\\
    &= \sum_{\*i \in \*I} \I_{\*Y_{\*x}(\*i) = \*y, \dots, \*Z_{\*w}(\*i) = \*z} P\left( \bigcap_{V \in \*V} \1U_V^{(i)} \right).
  \end{align*}
  In the above equations, the last two steps hold since variables $\*Y_{\*x}(\*i), \dots, \*Z_{\*w}(\*i)$ are not functions of exogenous variables $\*U$. This completes the proof.
\end{proof}
Let $\1C(\G)$ denote the collection of all maximal c-components (\Cref{def:c-component}) in a causal diagram $\G$. For instance, in the ``IV'' diagram $\G$ of \Cref{fig1a}, $\1C(\G)$ contains c-components $\*C(U_1) = \{Z\}$, $\*C(U_2) = \{X, Y\}$. The following proposition shows that probabilities over canonical partitions factorize over c-components in a causal diagram.
\begin{restatable}{lemma}{lemctfcp2}\label{lem:ctf_cp2}
  For an SCM $M = \tuple{\*V, \*U, \2F, P}$, let $\G$ be the associated causal diagram. For any $\*i \in \*I$,
  \begin{align}
    P\left( \bigcap_{V \in \*V} \1U_V^{(i)} \right) = \prod_{\*C \in \1C(\G)} P\left( \bigcap_{V \in \*C} \1U_V^{(i)} \right).
  \end{align}
\end{restatable}
\begin{proof}
  For any c-compoment $\*C \in \1C(\G)$, let $U_{\*C} = \cup_{V \in \*C} U_V$ the set of exogenous variables affecting (at least one of) endogenous variables in $\*C$. By the definition of c-components (\Cref{def:c-component}), it is verifiable that for two different c-compoments $\*C_1, \*C_2 \in \1C(\G)$, their corresponding exogenous variables $U_{\*C_1}, U_{\*C_2}$ do not share any element, i.e., $U_{\*C_1} \cap U_{\*C_2} = \emptyset$. We complete the proof by noting that exogenous variables in $\*U$ are mutually independent. 
\end{proof}
As an example, consider again the SCM $M$ compatible with \Cref{fig1a} defined in \Cref{eq:fxyz}. The event $Z = 1, X_{z=0} = 1, Y_{x=1}=0$ occurs if any only if $U_1 \in \1U_Z^{(2)}$ and $U_2 \in \left (\1U_X^{(3)} \cup \1U_X^{(4)} \right ) \cap \left (\1U_Y^{(1)} \cup \1U_Y^{(3)} \right)$. This implies
\begin{align*}
  &P\left(Z = 1, X_{z=0} = 1, Y_{x=1} = 0 \right) \\
  &= P\left ( \1U_Z^{(2)} \cap (\1U_X^{(3)} \cup \1U_X^{(4)}) \cap (\1U_Y^{(1)} \cup \1U_Y^{(3)}) \right)\\
  &= P\left ( \1U_Z^{(2)} \right) P\left (\1U_X^{(3)} \cup \1U_X^{(4)}) \cap (\1U_Y^{(1)} \cup \1U_Y^{(3)}) \right).
\end{align*}
The last step holds since $\{Z\}$ and $\{X, Y\}$ are two different c-components. It is verifiable from \Cref{fig:app_a1} that $\1U_Z^{(2)} = \{u_1 \in [0, 1.5]\}$, $\left (\1U_X^{(3)} \cup \1U_X^{(4)} \right ) \cap \left (\1U_Y^{(1)} \cup \1U_Y^{(3)} \right) = \{u_2 \in [1, 2] \}$. The above equation could be further written as:
\begin{align*}
  &P\left(Z = 1, X_{z=0} = 1, Y_{x=1} = 0 \right) \\
  &= P\left ( U_1 \in [0, 1.5] \right) P\left(U_2 \in [1, 2] \right) = \frac{1}{6}.
\end{align*}
The last step follows since variables $U_1, U_2$ are drawn uniformly at random over the interval $[0, 3]$. 

\subsection{A.2 Bounding Cardinalities of Exogenous Domains}
\Cref{lem:ctf_cp1,lem:ctf_cp2} together allow us to write any counterfactual distribution in an SCM as a function of products of probabilities assigned to the intersections of canonical partitions in every c-component. 
To prove the counterfactual equivalence in \Cref{thm:cfinite}, it is thus sufficient to construct a canonical SCM $N$ from an arbitrary SCM $M$ such that (1) $M, N$ are compatible with the same causal diagram $\G$; and (2) $M, N$ generate the same probabilities over canonical partitions. This section will describe how to construct such a discrete SCM.


We start the discussion by introducing some necessary notations and concepts. The probability distribution for every exogenous variable $U \in \*U$ is characterized with a probability space. It is frequently designated $\tuple{\D_U, \1F_U, P_U}$ where $\D_U$ is a sample space containing all possible outcomes; $\1F_U$ is a $\sigma$-algebra containing subsets of $\D_U$; $P_U$ is a probability measure on $\1F_U$ normalized by $P_U(\D_U) = 1$. Elements of $\1F_U$ are called \emph{events}, which are closed under operations of set complement and unions of countably many sets. By means of $P_U$, a real number $P_U(\1A) \in [0, 1]$ is assigned to every event $\1A \in \1F_U$; it is called the probability of event $\1A$.

For an arbitrary set of exogenous variables $\*U$, its realization $\*U = \*u$ is an element in the Cartesian product $\bigtimes_{U \in \*U} \D_U$, represented by a sequence $(u)_{U \in \*U}$. If now $\1A_U \in \D_U$, $\forall U \in \*U$, we may be interested in inferring whether a sequence of events $U \in \1A_U$ for every $U \in \*U$ occurs. Such an event is represented by a subset $\bigtimes_{U \in \*U} \1A_U \subseteq \bigtimes_{U \in \*U} \D_U$. The products $\times_{U \in \*U} \1A_U$ with $\1A_U$ running through $\1F_U$ generate precisely the product $\sigma$-algebra $\bigotimes_{U \in \*U} \1F_U$. The product measure $\bigotimes_{U\in \*U} P_U$ is the only probability measure $P$ with restrictions to $\bigotimes_{U \in \*U} \1F_U$ that satisfies the following consistency condition
\begin{align}
  P\left (\bigtimes_{U \in \*U} \1A_U \right) = \prod_{U \in \*U}P_U\left (\1A_U \right),
\end{align}
for arbitrary $\1A_U \in \1F_U$. It is obvious that $P$ is a probability measure. Consequently,
\begin{align}
  \left \langle \bigtimes_{U \in \*U} \D_U, \bigotimes_{U \in \*U} \1F_U, \bigotimes_{U \in \*U} P_U \right \rangle
\end{align}
defines the product of probability spaces $\langle \D_U, \1F_U, P_U \rangle$, $U \in \*U$. It is adequate to describe all ``measurable events'' occurring to exogenous variables $\*U$.

Recall that for subsets $\*X, \*Y \subseteq \*V$, counterfactual random variables (or potential responses) $\*Y_{\*x}(\*u)$ is defined as the solution of $\*Y$ in the submodel $M_{\*x}$ induced by intervention $\doo(\*x)$ given the configuration $\*U = \*u$. For any $\*y \in \D_{\*Y}$, let the inverse image $\*Y^{-1}_{\*x}(\*y)$ be the set of values $\*u$ generating the event $\*Y_{\*x}(\*u) = \*y$, i.e., 
\begin{align}
  \*Y^{-1}_{\*x}(\*y) = \left \{\*u \in \D_{\*U} \mid \*Y_{\*x}(\*u) = \*y\right \}.
\end{align}
Evidently, we are dealing with a $\bigotimes_{U \in \*U} \1F_U$-measurable mapping $\*Y_{\*x}: \*u \mapsto \*y$. Because of this measurability, the inverse image $\*Y^{-1}_{\*x}(\*y)$ is an event in $\bigotimes_{U \in \*U} \1F_U$ for any realization $\*y$. Thus $P\left (\*Y^{-1}_{\*x}(\*y) \right)$ is defined as the probability of $\*Y_{\*x}$ taking on a value $\*y$. Similarly, for any subsets $\*Y, \dots, \*Z$, $\*X, \dots, \*W \subseteq \*V$, the probability of a sequence of counterfactual events $\*Y_{\*x} = \*y, \dots, \*Z_{\*w} = \*z$ is defined as:
\begin{align*}
  P \left (\*y_{\*x}, \dots, \*z_{\*w} \right) = P \left (\*Y^{-1}_{\*x}(\*y) \cap \cdots \cap \*Z^{-1}_{\*w}(\*z) \right). \label{eq:product}
\end{align*}
We refer readers to \citep{durrett2019probability,bauer1972probability} for a detailed discussion on measure-theoretic probability concepts.

For a c-component $\*C$ in a causal diagram $\G$, we denote by $U_{\*C} = \cup_{V \in \*C}U_V$ the union of exogenous variables $U_V$ affecting an endogenous variable $V$ for every $V \in \*C$. Let exogenous variables in $U_{\*C}$ be ordered by $U_1, \dots, U_m$, $m = |U_{\*C}|$. For convenience, we consistently write $\langle \D_i, \1F_i, P_i\rangle$ as the probability space of $U_i$, $i = 1, \dots, m$. The product of these probability spaces is thus written as
\begin{align}
  \left \langle \bigtimes_{i = 1}^m \D_i, \bigotimes_{i = 1}^m \1F_i, \bigotimes_{i = 1}^m P_i \right \rangle.
\end{align}
For any SCM $M$ compatible with the diagram $\G$, the joint distribution over events defined by canonical partitions $\1U^{(i)}_V$ associated with variables $V \in \*C$ is given by
\begin{align}
  P\left( \bigcap_{V \in \*C} \1U_V^{(i)} \right) = \int_{\bigtimes_{i = 1}^m \D_i} \prod_{V \in \*C} \I_{u_V \in \1U_V^{(i)} } d \left(\bigotimes_{i = 1}^m P_i \right). \label{eq:ctf_cp3}
\end{align}
Our goal is to show that all correlations among events $\1U_V^{(i)}$, $V \in \*V$, induced by exogenous variables described by arbitrary probability spaces could be produced by a ``simpler'' generative process with discrete exogenous domains. 
\begin{lemma}\label{lem:ctf_cp3}
  Any distribution $P\left( \bigcap_{V \in \*C} \1U_V^{(i)} \right)$ in \Cref{eq:ctf_cp3} could be reproduced with a generic model of the form:
  \begin{align}
    P\left( \bigcap_{V \in \*C} \1U_V^{(i)} \right) = \sum_{j = 1}^m \sum_{u_j = 1}^{d} \prod_{V \in \*C} \I_{u_V \in \1U_V^{(i)} } \prod_{j = 1}^m P(u_j), \label{eq:ctf_cp4}
  \end{align}
  where every exogenous variable $U_j \in \*U$ takes values in a finite domain $\{1, \dots, d\}$, $d = \prod_{V \in \*C} |\D_{\PA_V} \mapsto \D_V|$.
\end{lemma}

\citep[Prop.~2]{rosset2018universal} applied a classic result of Carath\'eodory theorem in convex geometry \citep{caratheodory1911variabilitatsbereich} and showed that the observational distribution in any causal diagram could be generated using discrete exogenous variables, assuming that exogenous variables are drawn from distributions characterized with well-defined probability density functions. We here present a constructive proof that applies to the general framework of measure-theoretic probability theory.
\begin{proof}[Proof of Lemma~\ref{lem:ctf_cp3}]
  Let $\vec{P}$ be a vector representing probabilities of $\left ( P\left (\bigcap_{V \in \*C} \1U^{(i)}_V \right) \right )_{\*i \in \*I}$. Note that for every $V \in \*V$, there are $|\D_{\PA_V} \mapsto \D_V|$ equivalence classes $\1U^{(i)}_V$. $\vec{P}$ is thus a vector with $d = \prod_{V \in \*C} |\D_{\PA_V} \mapsto \D_V|$ elements. Since $\sum_{\*i}P\left (\bigcap_{V \in \*C} \1U^{(i)}_V \right) = 1$, it only takes a vector with $d - 1$ dimensions to uniquely determine $\vec{P}$. We could thus see $\vec{P}$ as a point in the $(d -1)$-dimensional real space. Similarly, $\left(\vec{P}, 1\right)$ is vector in $\3R^{d}$ where the $d$-th element is equal to $1$
  
  Fix an exogenous variable $U_1 \in U_{\*C}$. We define function $P_{u_1}\left (\bigcap_{V \in \*C} \1U^{(i)}_V \right)$ as the distribution over canonical partitions when $U_1$ is fixed as a constant $u_1 \in \D_{1}$. That is, 
  \begin{equation}
  \begin{aligned}
    &P_{u_1}\left (\bigcap_{V \in \*C} \1U^{(i)}_V \right) \\
    &= \left [ \int_{\bigtimes_{j = 2}^m \D_i} \prod_{V \in \*C} \I_{u_V \in \1U_V^{(i)} } d \left(\bigotimes_{j = 2}^m P_j \right) \right ]_{U_1 = u_1}
  \end{aligned}
  \end{equation}
  The associativity of the product of probability spaces \citep[Ch.~3.3]{bauer1972probability} generally implies:
  \begin{equation}
  \begin{aligned}
    &\bigotimes_{j = 1}^m \1F_j = \1F_1 \otimes \left (\bigotimes_{j = 2}^m \1F_j \right),\\
    &\bigotimes_{j = 1}^m P_j = P_1 \otimes \left (\bigotimes_{j = 2}^m P_j \right).
  \end{aligned}
  \end{equation}
  Let $\vec{P}_{u_1}$ be a vector in $\3R^{d-1}$ representing probabilities of $P_{u_1}\left (\bigcap_{V \in \*C} \1U^{(i)}_V \right)$ and let $\left(\vec{P}_{u_1}, 1\right)$ be vector in $\3R^{d}$ where the $d$-th element is equal to $1$. Applying Fubini's Theorem \citep[Thm.~1.7.2]{durrett2019probability} implies that function $u_1 \mapsto \left (\vec{P}_{u_1}, 1\right)$ is $\1F_1$-measurable. That is, $\langle \D_1, \1F_1, P_1 \rangle$ yields a probability measure for a set $\left \{ \left(\vec{P}_{u_1}, 1 \right) \mid \forall u_1 \in \D_1 \right\}$ with respective to Borel sets in real space $\3R^{d}$ with average 
  \begin{align}
    \left (\vec{P}, 1\right) = \int_{\D_{1}}\left(\vec{P}_{u_1}, 1 \right) dP_1.
  \end{align}
  It can be shown that the probability vector $\left (\vec{P}, 1\right)$ is a point lying in the convex hull of a set $\left \{\left(\vec{P}_{u_1}, 1 \right) \mid \forall u \in \D_U \right\}$ (see \citep[Thm.~2.4.1]{blackwell1979theory} and its extension to arbitrary probability measures in \citep{rubin1958note}). This means that there exists a finite set of vectors $\left(\vec{P}_{u^{(1)}_1}, 1 \right), \dots, \left(\vec{P}_{u^{(n)}_1}, 1 \right)$ and a sequence of positive coefficients $\alpha_1, \dots, \alpha_n >0$ such that
  \begin{align}
      \left (\vec{P}, 1\right) = \sum_{k = 1}^n \alpha_k \left(\vec{P}_{u^{(k)}_1}, 1 \right).
  \end{align}
  The above equation implies
  \begin{align}
    &\vec{P} = \sum_{k = 1}^{n} \alpha_{k} \vec{P}_{u^{(k)}_1}, &&\text{and }\sum_{k = 1}^{n} \alpha_k = 1
 \end{align}
 Indeed, we could further reduce the number of coefficients $n$ by removing linearly dependent vectors. If vectors $\left(\vec{P}_{u^{(k)}_1}, 1 \right)$ are not linearly independent, there exists a non-trivial solution $\lambda_1, \dots \lambda_n$ such that $\sum_{k} \lambda_k \left(\vec{P}_{u^{(k)}_1}, 1 \right) = \vec{0}$. It is verifiable that for any real value $\beta > 0$
 \begin{align}
     &\sum_{k = 1}^n (\alpha_k - \beta \lambda_k) \left(\vec{P}_{u^{(k)}_1}, 1 \right) \\
     &= \sum_{k = 1}^n \alpha_k \left(\vec{P}_{u^{(k)}_1}, 1 \right) - \beta \sum_{k = 1}^n \lambda_k \left(\vec{P}_{u^{(k)}_1}, 1 \right)\\
     &= \sum_{k = 1}^n \alpha_k \left(\vec{P}_{u^{(k)}_1}, 1 \right).
 \end{align}
 The last step holds since $\sum_{k} \lambda_k \left(\vec{P}_{u^{(k)}_1}, 1 \right) = \vec{0}$. Therefore, coefficients $\alpha_k - \beta \lambda_k$, $k = 1, \dots, n$, satisfy
  \begin{align}
     \sum_{k = 1}^n (\alpha_k - \beta \lambda_k) \left(\vec{P}_{u^{(k)}_1}, 1 \right) = \left (\vec{P}, 1\right).
 \end{align}
 Let $\beta$ be the largest value such that $\alpha_k - \beta \lambda_k \geq 0$ for all $k$. Consequently, there must exist a coefficient $\alpha_k - \beta \lambda_k = 0$. We could then remove the corresponding vector $\left(\vec{P}_{u^{(k)}_1}, 1 \right)$ from the base. This procedure continues until all remaining vectors are linearly independent. Since $\left(\vec{P}_{u_1}, 1 \right) \in \3R^d$, there are at most $d$ linearly independent vectors, i.e., $n \leq d$. 
 
 Finally, we replace the probability measure $P_1$ with a discrete distribution $P\left(U_1 = u^{(k)}_1 \right) = w_k$ over a finite discrete domain $\D^*_{1} = \left\{u^{(1)}_1, \dots, u^{(d)}_1 \right \}$. Doing so generated a new SCM $N^*$, with cardinality $|\D_{1}| \leq d$, that reproduces probabilities $P\left (\bigcap_{V \in \*C} \1U^{(i)}_V \right)$ over canonical partitions in the original SCM $M$. Repeatedly applying this procedure for every exogenous $U_2, \dots, U_m$ completes the proof.
\end{proof}
\Cref{lem:ctf_cp1,lem:ctf_cp2,lem:ctf_cp3} together yield a natural constructive proof for \Cref{thm:cfinite} in an arbitrary causal diagram $\G$.
\thmcfinite*
\begin{proof}
  By the definition of c-components (\Cref{def:c-component}), it is verifiable that for two different c-compoments $\*C_1, \*C_2 \in \1C(\G)$, their corresponding exogenous variables $U_{\*C_1}, U_{\*C_2}$ do not share any element, i.e., $U_{\*C_1} \cap U_{\*C_2} = \emptyset$. Therefore, we could repeatedly apply the construction of \Cref{lem:ctf_cp3} for every c-component $\*C \in \1C(\G)$. Doing so generates a discrete SCM $N$ satisfying conditions as follows: 
  \begin{enumerate}
    \item $N$ is compatible with $\G$;
    \item $N$ and $M$ share the same set of structural functions $\2F$;
    \item $N$ and $M$ generate the same joint distribution over the intersections of canonical partitions associated with every c-component. 
  \end{enumerate}
  It follows from \Cref{lem:ctf_cp1,lem:ctf_cp2} that $M$ and $N$ must coincide in all counterfactual distributions $\*P^*$ over endogenous variables. This completes the proof.
\end{proof}

\subsection{A.3 Decomposing Canonical Partitions}
This section provides a more fine-grained decomposition for equivalence classes in canonical partitions. Such a decomposition provides new insights to the discretization procedure.
\begin{definition}[Cell]\label{def:rec}
  For an SCM $M = \tuple{\*V, \*U, \2F, P}$, for each $V \in \*V$, a subset $\1R_V \subseteq \D_{U_V}$ is a \emph{cell} if $\1R_V = \vartimes_{U \in U_V} \1R_{V, U}$ where $\1R_{V, U} \subseteq \D_U$, for every $U \in \*U$.
\end{definition}
Obviously, for $|U_V| = 1$, any subset of $\D_{U_V}$ is a cell. However, the same is not necessarily true for $|U_V| \geq 2$. As an example, consider an SCM $M$ associated with the causal diagram of \Cref{fig1b} where $X, Y, Z$ are binary variables in $\{0, 1\}$; $U_1, U_2$ are continuous variables drawn uniformly from the interval $[0, 3]$. More specifically,
\begin{equation}
  \begin{aligned}
    &y \gets f_Y(x, u_1, u_2) = \sum_{i = 1}^{4} \I_{(u_1, u_2) \in \1U^{(i)}_Y} h^{(i)}_Y(x),\\
    &x \gets f_X(z, u_2) = \sum_{i = 1}^{4} \I_{u_2 \in \1U^{(i)}_X} h^{(i)}_X(z),\\
    &z \gets f_Z(u_1) = \sum_{i = 1}^{2}\I_{u_1 \in \1U^{(i)}_Z} h^{(i)}_Z
  \end{aligned} \label{eq:fxyz2}
\end{equation}
Canonical partitions $\1U^{(i)}_Y, \1U^{(j)}_X, \1U^{(k)}_Z$ are described in \Cref{fig:app_a2}. For points on the boundary, we include them in the equivalence class with a higher indices $i, j, k$. As an example, consider the equivalence class $\1U^{(1)}_Y$, i.e.,
\begin{align}
  &\1U^{(1)}_Y = \left ([0, 2) \times [0, 1) \right) \cup \left ((2, 3] \times (2,3] \right).
\end{align}
It is a subset in the union of two cells $\1R^{(1)}_Y, \1R^{2)}_Y$ given by
\begin{align}
  &\1R^{(1)}_Y = [0, 2] \times [0, 1], &\1R^{(2)}_Y = [2, 3] \times [2,3]. \label{eq:cell1}
\end{align}
However, one could not write the equivalence class $\1U^{(1)}_Y$ as a product of intervals in $\D_{U_1}, \D_{U_2}$, i.e., $\1U^{(1)}_Y$ is not a cell.

\begin{figure}[t]
  \centering
  \null
  \begin{subfigure}{0.95\linewidth}\centering
  \resizebox{\linewidth}{!}{
  \begin{tikzpicture}

    \node[box, fill=lightblue] at (0.5, 1.5) {};
    \node[box, fill=lightblue] at (2.5, 0.5) {};

    \node[box, fill=lightgreen] at (1.5, 1.5) {};
    \node[box, fill=lightgreen] at (2.5, 1.5) {};

    \node[box, fill=lightred] at (1.5, 2.5) {};
    \node[box, fill=lightred] at (0.5, 2.5) {};

    \draw[->, >={Latex}] (-0.5,0) -- (3.5,0) node[below] {\scalebox{0.7}{$u_1$}}; 
    \draw[->, >={Latex}] (0,-0.5) -- (0,3.5) node[left] {\scalebox{0.7}{$u_2$}};

    \node (r1) at (1, 0.7) {\scalebox{0.7}{$\1U^{(1)}_Y$}};	
    \node (h1) at (1, 0.2) {\scalebox{0.7}{$y \gets 0$}};


    \node (r1) at (2.5, 2.7) {\scalebox{0.7}{$\1U^{(1)}_Y$}};	
    \node (h1) at (2.5, 2.2) {\scalebox{0.7}{$y \gets 0$}};


    \node (r1) at (0.5, 1.7) {\scalebox{0.7}{$\1U^{(3)}_Y$}};	
    \node (h1) at (0.5, 1.2) {\scalebox{0.7}{$y \gets \neg x$}};


    \node (r1) at (2, 1.7) {\scalebox{0.7}{$\1U^{(4)}_Y$}};	
    \node (h1) at (2, 1.2) {\scalebox{0.7}{$y \gets 1$}};

    \node (r1) at (2.5, 0.7) {\scalebox{0.7}{$\1U^{(3)}_Y$}};	
    \node (h1) at (2.5, 0.2) {\scalebox{0.7}{$y \gets \neg x$}};

    \node (r1) at (1, 2.7) {\scalebox{0.7}{$\1U^{(2)}_Y$}};	
    \node (h1) at (1, 2.2) {\scalebox{0.7}{$y \gets x$}};

    \draw[thick, dashed, -] (0, 1) -- (3, 1);
    \draw[thick, dashed, -] (0, 2) -- (3, 2);
    \draw[thick, dashed, -] (0, 3) -- (3, 3);

    \draw[thick, dashed, -] (1, 1) -- (1, 2);
    \draw[thick, dashed, -] (2, 0) -- (2, 1);
    \draw[thick, dashed, -] (2, 2) -- (2, 3);
    \draw[thick, dashed, -] (3, 0) -- (3, 3);
    
    \draw[thin, -] (1, 0) -- (1, 0.05);

    \node at (-0.2, -0.2) {\scalebox{0.7}{0}};

    \node [below] at (1, 0) {\scalebox{0.7}{1}};	
    \node [below] at (2, 0) {\scalebox{0.7}{2}};
    \node [below] at (3, 0) {\scalebox{0.7}{3}};
    \node [left] at (0, 1) {\scalebox{0.7}{1}};
    \node [left] at (0, 2) {\scalebox{0.7}{2}};
    \node [left] at (0, 3) {\scalebox{0.7}{3}};
  \end{tikzpicture}
  }
  \caption{$y \gets f_Y(x, u_1, u_2)$}
  \label{fig:app_a2a}
  \end{subfigure}
\hfill
  \begin{subfigure}{0.95\linewidth}\centering
  \resizebox{\linewidth}{!}{
  \begin{tikzpicture}
    \draw[->, >={Latex}] (-0.5,0) -- (3.5,0) node[below] {\scalebox{0.7}{$u_2$}}; 

    \fill[fill=lightred] (2,0.01) rectangle (3,0.5);
    \fill[fill=lightblue] (0,0.01) rectangle (1,0.5);
    \fill[fill=lightgreen] (1,0.01) rectangle (2,0.5);

    \node (r1) at (2.5, 0.33) {\scalebox{0.7}{$\1U^{(2)}_X$}};	
    \node (h1) at (2.5, 0.1) {\scalebox{0.7}{$x \gets z$}};

    \node (r1) at (0.5, 0.33) {\scalebox{0.7}{$\1U^{(3)}_X$}};	
    \node (h1) at (0.5, 0.1) {\scalebox{0.7}{$x \gets \neg z$}};

    \node (r1) at (1.5, 0.33) {\scalebox{0.7}{$\1U^{(4)}_X$}};	
    \node (h1) at (1.5, 0.1) {\scalebox{0.7}{$x \gets 1$}};

    \draw[thick, dashed, -] (0, 0) -- (0, 0.5);
    \draw[thick, dashed, -] (1, 0) -- (1, 0.5);
    \draw[thick, dashed, -] (2, 0) -- (2, 0.5);
    \draw[thick, dashed, -] (3, 0) -- (3, 0.5);

    \node [below] at (0, 0) {\scalebox{0.7}{0}};	
    \node [below] at (1, 0) {\scalebox{0.7}{1}};	
    \node [below] at (2, 0) {\scalebox{0.7}{2}};
    \node [below] at (3, 0) {\scalebox{0.7}{3}};	
  \end{tikzpicture}
  }
  \caption{$x \gets f_X(z, u_2)$}
  \label{fig:app_a2b}
  \end{subfigure}\hfill
  \begin{subfigure}{0.95\linewidth}\centering
  \resizebox{\linewidth}{!}{
  \begin{tikzpicture}
    \draw[->, >={Latex}] (-0.5,0) -- (3.5,0) node[below] {\scalebox{0.7}{$u_1$}}; 

    \fill[fill=lightred] (0,0.01) rectangle (1,0.5);

    \node (r1) at (2, 0.33) {\scalebox{0.7}{$\1U^{(1)}_Z$}};	
    \node (h1) at (2, 0.1) {\scalebox{0.7}{$z \gets 0$}};

    \node (r1) at (0.5, 0.33) {\scalebox{0.7}{$\1U^{(2)}_Z$}};	
    \node (h1) at (0.5, 0.1) {\scalebox{0.7}{$z \gets 1$}};

    \draw[thick, dashed, -] (0, 0) -- (0, 0.5);
    \draw[thick, dashed, -] (1, 0) -- (1, 0.5);
    \draw[thin, -] (2, 0) -- (2, 0.05);
    \draw[thick, dashed, -] (3, 0) -- (3, 0.5);
    
    \node [below] at (0, 0) {\scalebox{0.7}{0}};	
    \node [below] at (1, 0) {\scalebox{0.7}{1}};	
    \node [below] at (2, 0) {\scalebox{0.7}{2}};
    \node [below] at (3, 0) {\scalebox{0.7}{3}};

  \end{tikzpicture}
  }
  \caption{$z \gets f_Z(u_1)$}
  \label{fig:app_a2c}
  \end{subfigure}\null
  \caption{Canonical partitions of exogenous domains associated with $X, Y, Z$. Each equivalence class (e.g., $\1U^{(i)}_Y$) is covered by a finite set of (almost) disjoint cells (e.g., $\1U^{(4)}_Y \subseteq [2, 3] \times [0, 1]$). For points on the boundary, we break the ties in favor of equivalence classes with higher indices.}
  \label{fig:app_a2}
\end{figure}
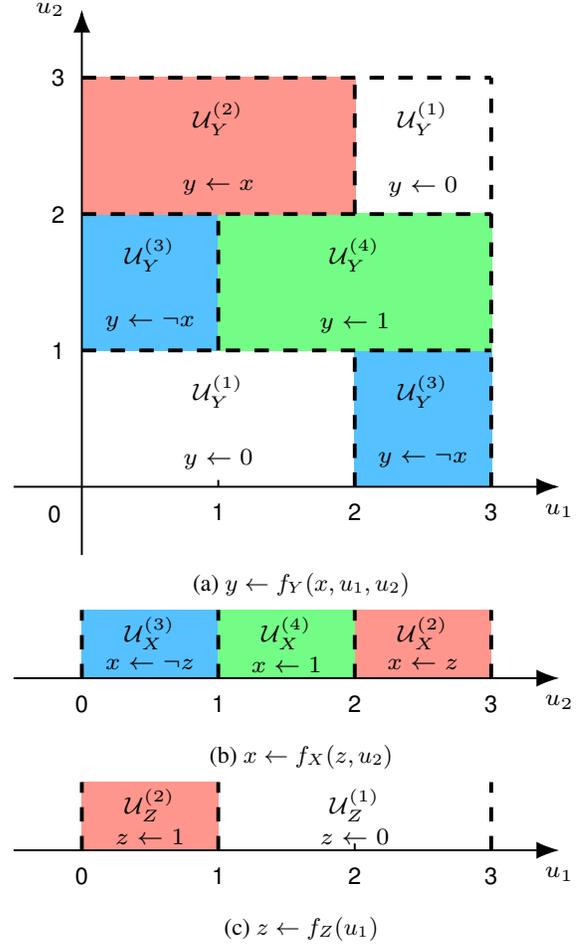

Our next result shows that each equivalence class in the canonical partition could be decomposed into a countable union of almost disjoint cells.
\begin{definition}[Covering]\label{def:covering}
  For an SCM $M = \tuple{\*V, \*U, \2F, P}$, for every $V \in \*V$, let $\1U_V$ be an arbitrary subset of $\D_{U_V}$. Consider the following conditions:
  \begin{enumerate}[itemsep=0pt,topsep=0pt,parsep=0pt]
    \item $\left \{ \1R^{(j)}_V \mid j \in \*J_V \right\}$ is a countable set of cells.
    \item For any $i \neq j$, $\1R^{(i)}_V$ and $\1R^{(j)}_V$ are almost disjoint, i.e., 
    \begin{align}
        P \left (\1R^{(i)}_V \cap \1R^{(j)}_V \right) = 0.
    \end{align}
    \item $\1U_V$ is a subset for $\cup_{j \in \*J_V}\1R^{(j)}_V$.
  \end{enumerate}
  Then, $\left \{ \1R^{(j)}_V \mid j \in \*J_V  \right\}$ is said to be a covering for $\1U_V$.
\end{definition} 
\begin{restatable}{lemma}{lemrec}\label{lem:covering}  
  For an SCM $M = \tuple{\*V, \*U, \2F, P}$, for every $V \in \*V$, let $\1U^{(i)}_V$ be the equivalent class for an arbitrary function $h_V^{(i)} \in \D_{\PA_V} \mapsto \D_V$. There exists a covering $\left \{ \1R^{(j)}_V \mid j\in \*J_V \right\}$ for $\1U^{(i)}_V$ such that 
  \begin{align}
      P \left (\1U^{(i)}_V \right) = \sum_{j \in \*J_V} P \left (\1R^{(j)}_{V} \right).  \label{eq:lemrec2}
  \end{align}
\end{restatable}
\begin{proof}
We first consider a weaker version of the covering cells which does not require every pair of cells to be disjoint. That is, condition (2) in \Cref{def:covering} does not necessarily hold. For any $\1A \subseteq \D_{U_V}$, define a set of coverings $\2C(\1A)$:
  \begin{align}
      \2C(\1A) = \left \{\1C \subseteq 2^{\D_{U_V}} \mid \1C \text{ is a covering for } \1A \right\}. \label{eq:lemrec1}
  \end{align}
  where $2^{\D_{U_V}}$ represents the set of all subsets of $\D_{U_V}$.
  
  Recall that every $U \in \*U$ is associated with a probability space $\tuple{\D_U, \1F_U, P_U}$. The product measure $\bigotimes_{U\in \*U} P_U$ is the only probability measure $P$ with restrictions to $\bigotimes_{U \in \*U} \1F_U$ which satisfies the independence restriction in \Cref{eq:product}. It follows from the construction of product measures \citep[Theorem 1.5.2]{bauer1972probability} that such a probability measure $P$ must satisfy the following property: for any $\1A \subseteq \D_{U_V}$,
  \begin{align}
      P\left(\1U^{(i)}_V \right) = \inf \left \{\sum_{\1R_V \in \1C} P \left( \1R_V \right) \mid \forall \1C \in \2C \left (\1U_V \right) \right\}. \label{eq:lemrec4}
  \end{align}
  Therefore, we could obtain a covering $\left \{ \1R^{(j)}_V \mid j\in \*J_V \right\}$ for an arbitrary equivalence class $\1U^{(i)}_V$ such that
  \begin{align}
      P \left (\1U^{(i)}_V \right) = \sum_{j \in \*J_V} P \left (\1R^{(j)}_{V} \right).  \label{eq:lemrec5}
  \end{align}
  What remains is to show that every pair $\1R_V^{(i)}, \1R_V^{(j)}$ are almost disjoint. This is equivalent to proving the following:
  \begin{align}
      P \left (\bigcup_{j \in \*J_V} \1R^{(j)}_V \right) =  \sum_{j \in \*J_V} P \left( \1R^{(j)}_V \right).
  \end{align}
  By basic properties of probability measures, 
  \begin{align}
    P \left (\bigcup_{j \in \*J_V} \1R^{(j)}_V \right) \leq  \sum_{j \in \*J_V} P \left( \1R^{(j)}_V \right).
  \end{align}
  Therefore, it is sufficient to show that 
  \begin{align}
    P \left (\bigcup_{j \in \*J_V} \1R^{(j)}_V \right) \geq  \sum_{j \in \*J_V} P \left( \1R^{(j)}_V \right). \label{eq:lemrec3}
  \end{align}
  Suppose now \Cref{eq:lemrec3} does not hold. This means that there exists a covering $\1C' \in \2C \left (\cup_{j \in \*J_V}\1R^{(j)}_V \right)$ such that 
  \begin{align}
      P \left (\bigcup_{j \in \*J_V}\1R^{(j)}_V \right) = \sum_{\1R'_V \in \1C'} P \left( \1R'_V \right) < \sum_{j \in \*J_V} P \left( \1R^{(j)}_V \right).
  \end{align}
  By the definition in \Cref{eq:lemrec1}, $\1C'$ is also a covering in $\2C \left (\1U^{(i)}_V \right)$. The property in \Cref{eq:lemrec4} implies:
  \begin{align}
      P\left(\1U^{(i)}_V \right) \leq \sum_{\1R'_V \in \1C'} P \left( \1R'_V \right) < \sum_{j \in \*J_V} P \left( \1R^{(j)}_V \right), 
  \end{align}
  which contradicts \Cref{eq:lemrec5}. This completes the proof.
\end{proof}
Henceforth, we will consistently refer to a set of cells as a covering if they satisfy conditions \emph{both} in \Cref{def:covering} and \Cref{eq:lemrec2}. For instance, consider the equivalence class $\1U^{(1)}_Y$ in \Cref{fig:app_a2a} and cells $\1R^{(1)}_Y, \1R^{(2)}_Y$ defined in \Cref{eq:cell1}. Since $\1U^{(1)}_Y \subseteq  \1R^{(1)}_Y \cup \1R^{(1)}_Y$, $\left \{\1R^{(1)}_Y, \1R^{(2)}_Y\right\}$ forms a covering for $\1U^{(1)}_Y$. By noting that finite segments in  $\D_{U_1} \times \D_{U_2}$ (e.g., a line $U_1 = 2$) has zero measure, we have
\begin{equation}
\begin{aligned}
  &P \left ( \1U^{(1)}_Y \right) \\
  &= P \left (\left(U_1, U_2 \right) \in  \left ([0, 2] \times [0, 1] \right) \cup \left ((2, 3] \times (2,3] \right) \right) \\
  &= P \left (\left(U_1, U_2 \right) \in [0, 2] \times [0, 1] \right) \\
  &+ P \left (\left(U_1, U_2 \right) \in [2, 3] \times [2,3] \right) \\
  &= P \left (\1R^{(1)}_Y \right) + P \left (\1R^{(2)}_Y \right).
\end{aligned}
\end{equation}
The existence of covering cells also allows us to decompose probabilities over intersections of equivalence classes across canonical partitions. Formally,
\begin{restatable}{lemma}{lemctfcp}\label{lem:ctf_cp4}
  For an SCM $M = \tuple{\*V, \*U, \2F, P}$, for any $\*i \in \*I$, there exists a sequence of coverings $\left \{ \1R^{(j)}_V \mid j \in \*J_V  \right\}$ for $\1U^{(i)}_V$, $\forall V \in \*V$, such that 
  \begin{align}
    P\left( \bigcap_{V \in \*V} \1U_V^{(i)} \right) = \sum_{\*j \in \*J} \prod_{U \in \*U} P\left( \bigcap_{V \in \ch(U)} \1R_{V,U}^{(j)} \right) \label{eq:ctf_cp5}
  \end{align}
\end{restatable}

\begin{table*}[t]
    \centering
    \begin{tabular}{|c||c|c|c|c|c|c|c|c|}
    \hline
    $\*I \subseteq \{1, \dots, 4\}$ & $\emptyset$ & $\{1\}$ & $\{2\}$ & $\{3\}$ & $\{4\}$ & $\{1, 2\}$ & $\{1, 3\}$ & $\{1, 4\}$\\ \hline
    $\1A_{U_1}^{\*I}$ & $\emptyset$ & $\emptyset$ & $\emptyset$ & $\emptyset$ & $\emptyset$ & $\emptyset$ & *$[0, 1)$ & $\emptyset$ \\ \hline \hline
     $\*I \subseteq \{1, \dots, 4\}$ & $\{2, 3\}$ & $\{2, 4\}$ & $\{3, 4\}$ & $\{1, 2, 3\}$ & $\{1, 2, 4\}$ & $\{1, 3, 4\}$ & $\{2, 3, 4\}$ & $\{1, 2, 3, 4\}$\\ \hline
    $\1A_{U_1}^{\*I}$ & *$(1, 2)$ & *$(2, 3]$ & $\emptyset$ & $[1, 1]$ & $\emptyset$ & $\emptyset$ & $[2, 2]$ & $\emptyset$ \\ \hline
    \end{tabular}
    \caption{Atoms generated by subsets $\1R_{U_1}^{(i)}, i = 1, \dots, 4$ defined in \Cref{eq:projection_exp} which are contained in the domain of an exogenous variable $U_1$ drawn uniformly from an interval $[1, 3]$. Atoms with positive probability measure is marked with a asterisk ``*''.}.
    \label{tab:atom}
\end{table*}

\begin{proof}
For every $V \in \*V$, let $\left \{ \1R^{(j)}_V \mid j \in \*J_V  \right\}$ be a covering for $\1U^{(i)}_V$ defined in \Cref{lem:covering}, i.e., it satisfies \Cref{eq:lemrec1}. We first show that, for any subset $\1A \subseteq \D_{\*U}$,
  \begin{align}
      P\left ( \1U_V^{(i)} \cap \1A \right) = \sum_{j \in \*J_V } P \left(\1R^{(i)}_V \cap \1A \right). \label{eq:ctfcp1}
  \end{align}
  Let $\1A^\complement = \D \setminus \1A$. Since $\left \{ \1R_V^{(j)} \mid j \in \*J_V \right \}$ is a covering of $\1U_V^{(i)}$, we must have the following:
  \begin{align}
      &P\left ( \1U_V^{(i)} \cap \1A \right) \leq \sum_{j \in \*J_V} P \left(\1R^{(j)}_V  \cap \1A \right), \label{eq:ctfcp2} \\
      &P\left ( \1U_V^{(i)} \cap \1A^\complement \right) \leq \sum_{j \in \*J_V} P \left(\1R^{(j)}_V \cap \1A^\complement\right). \label{eq:ctfcp4}
  \end{align}
  Next, we show that the above inequality relationships are both tight. Suppose at least one of inequalities in \Cref{eq:ctfcp2,eq:ctfcp3} is strict. We must have 
  \begin{align*}
      P\left(\1U_V^{(i)} \right)  &= P\left ( \1U_V^{(i)} \cap \1A \right) + P\left ( \1U_V^{(i)} \cap \1A^\complement\right) \\
      &< \sum_{j \in \*J_V} P \left(\1R^{(j)}_V  \cap \1A \right) + \sum_{j \in \*J_V} P \left(\1R^{(j)}_V \cap \1A^\complement\right).
  \end{align*}
  The above equation implies 
  \begin{align}
      P\left(\1U_V^{(i)} \right) < \sum_{j \in \*J_V} P \left(\1R^{(j)}_V \right),
  \end{align}
  which contradicts \Cref{eq:lemrec1}. This means that the statement in \Cref{eq:ctfcp1} must hold, which implies, for any $\*i \in \*I$,
  \begin{align}
      P\left( \bigcap_{V \in \*V} \1U_V^{(i)} \right) = \sum_{\*j \in \*J} P\left( \bigcap_{V \in \*V} \1R_{V}^{(j)} \right).
  \end{align}
  Recall that each cell $\1R^{(j)}_{V}$ is a product $\vartimes_{U \in U_V} \1R^{(j)}_{V, U}$ where $\1R^{(j)}_{V, U}$ is a subset in $\D_U$. Since exogenous variables $\*U$ are mutually independent, we must have, for any $\*j \in \*J$,
  \begin{align}
      P\left( \bigcap_{V \in \*V} \1R_{V}^{(j)} \right) = \prod_{U\in \*U} P\left( \bigcap_{V \in \ch(U)} \1R_{V,U}^{(j)} \right).
  \end{align}
  This completes the proof.
\end{proof}
Consider again the SCM $M$ described in \Cref{eq:fxyz2}. Note that only function in the hypothesis class $\1D_{\emptyset} \mapsto \1D_Z$ compatible with event $Z = 1$ is $h^{(2)}_Z \equiv z \gets 1$. Similarly, event $X_{z=0} = 1, X_{z=1} = 0$ corresponds to function $h^{(3)}_X \equiv x \gets \neg z$; event $Y_{x=0} = 0, Y_{x=1} = 0$ corresponds to the function $h^{(1)}_Y(x) \equiv y \gets 0$. The decomposition of \Cref{eq:ctf_cp1} gives:
\begin{align}
  &P \left(Z = 1, X_{z=0} = 1, X_{z=1} = 0, Y_{x=0} = 0, Y_{x=1} = 0\right) \notag \\
  &= P \left (\1U^{(1)}_Z \cap \1U^{(3)}_X \cap \1U^{(1)}_Y \right) \label{eq:exp1}
\end{align}
Among above quantities, $\1U^{(1)}_Y$ is covered by cells $\left \{\1R^{(1)}_Y, \1R^{(2)}_Y \right \}$ defined in \Cref{eq:cell1}. $\1U^{(1)}_Z $ and $\1U^{(3)}_X$ are covered by cells $\1R^{(1)}_Z$ and $\1R^{(1)}_X$, respectively, given by 
\begin{align}
    &\1R^{(1)}_Z = \{u_1 \in [0, 1] \}, &&\1R^{(1)}_X = \{u_2 \in [0, 1] \}.
\end{align}
Applying the decomposition in \Cref{eq:ctf_cp5} implies
\begin{align}
    &P \left (\1U^{(1)}_Z \cap \1U^{(3)}_X \cap \1U^{(1)}_Y \right) \notag \\
    &= P\left ( \1R^{(1)}_Z \cap \1R^{(1)}_X \cap \1R^{(1)}_Y \right) +P\left ( \1R^{(1)}_Z \cap \1R^{(1)}_X \cap \1R^{(2)}_Y \right)  \notag\\
    &= P\left ( U_1 \in [0, 1] \right)P\left ( U_2 \in [0, 1] \right). \label{eq:exp2}
\end{align}
\Cref{eq:exp1,eq:exp2} together give the evaluation 
\begin{align*}
  P \left(Z = 1, X_{z=0} = 1, X_{z=1} = 0, Y_{x=0} = 0, Y_{x=1} = 0\right) = \frac{1}{9}.
\end{align*}
One could verify the above equation from the parametrization in \Cref{eq:fxyz2} using the three-step algorithm in \citep{pearl:2k} which consists of abduction, action, and prediction.

\subsection{A.4 Decomposing Covering Cells}
For an arbitrary cell $\1R_V = \times_{U \in U_V} \1R_{V, U}$, we will call every ``side'' $\1R_{V, U}$ the projection of $\1R_V$ onto domain $\D_U$, for every $U \in U_V$. Observe that for disjoint cells, their projections onto the same domain may not necessarily be disjoint. As an instance, equivalence classes $\1U^{2}_Y$ and $\1U^{4}_Y$ in \Cref{fig:app_a1a} are covered by (almost) disjoint cells $[1, 3]\times [1, 2]$ and $[0, 2] \times [2, 3]$ respectively. Their projections onto $U_1$ are intervals $[1, 3]$ and $[0, 2]$, which overlap in the sub-interval $[1, 2]$. This observation suggests that every covering cell could be further decomposed, which will be our focus in this section.

The collection of all projections of covering cells $\1R_{V,U}^{(j)}$ defined in \Cref{lem:ctf_cp4} onto an exogenous $U \in \*U$ is given by
\begin{align}
    \left \{ \1R_{V,U}^{(j)} \mid \forall V \in \ch(U), \forall \*j \in \*J \right \}. \label{eq:projection}
\end{align}
\Cref{lem:ctf_cp1,lem:ctf_cp4} shows that all counterfactual distributions in any SCM could be written as a function of probabilities over intersections of above projections, i.e.,  
\begin{align}
    \left \{ P\left( \bigcap_{V \in \ch(U)} \1R_{V,U}^{(j)} \right) \mid \forall V \in \ch(U), \forall \*j \in \*J \right \}. \label{eq:prob_proj}
\end{align}
To prove the counterfactual equivalence of canonical SCMs, it is thus sufficient to show that probabilities in \Cref{eq:prob_proj} could be generated using a discrete distribution.

For convenience, we will slightly abuse the notation and consistently represent \Cref{eq:projection} using a countable set $\left \{ \1R_{U}^{(j)} \mid j \in \mathbb{N}\right\}$. We will also utilize a special type of subsets in domain $\D_U$ generated by intersections over projections and their complements, which we call \emph{atoms}.
\begin{definition}[Atom]\label{def:atom}
  For an arbitrary $U \in \*U$, let $\left \{ \1R_{U}^{(j)} \mid j \in \mathbb{N}\right\}$ be a countable collection of subsets in $\D_U$. For any $\*I \subseteq \3N$, an \emph{atom} $\1A^{\*I}_U \subseteq \D_U$ is defined as:
\begin{align}
    \1A^{\*I}_U = \bigcap_{i \in \*I} \1R^{(i)}_U \cap \bigcap_{i \not \in \*I} \left(\D_U \setminus \1R^{(i)}_U\right). \label{eq:atom1}
\end{align}
\end{definition}
Observe that these atoms are pairwise disjoint, and that $\bigcup_{\*I \subseteq \3N} \1A^{\*I}_U = \Omega$. For instance, consider again canonical partitions described in \Cref{fig:app_a2}. Covering cells for $\1U^{(i)}_Y, \1U^{(j)}_Z$ generates a collection of projections $\left \{\1R^{(i)}_{U_1} \mid i = 1, \dots, 4 \right \}$ onto the exogenous domain of $U_1$, i.e., 
\begin{equation}
\begin{aligned}
    &\1R^{(1)}_{U_1} = [0, 1], &\1R^{(2)}_{U_1} = [1, 3],\\
    &\1R^{(3)}_{U_1} = [0, 2], &\1R^{(4)}_{U_1} = [2, 3].
\end{aligned}\label{eq:projection_exp}
\end{equation}
For an indexing set $\*I = \{1, 3\}$, atom $\1A^{\{1, 3\}}_{U_1}$ is given by
\begin{align*}
    \1A^{\{1, 3\}}_{U_1} &= \1R^{(1)}_{U_1} \cap \1R^{(3)}_{U_1} \cap \left(\D_{U_1} \setminus \1R^{(2)}_{U_1}\right)\cap \left(\D_{U_1} \setminus \1R^{(4)}_{U_1}\right)\\
    &=[0, 1]\cap [0, 2]\cap[0, 1) \cap[0, 2)\\
    &=[0, 1)
\end{align*}
\Cref{tab:atom} shows atoms computed from all indexing sets $\*I \subseteq \{1, \dots, 4\}$. We obtain a set of atoms $\left \{\1A^{(i)}_{U_1} \mid i = 1, \dots, 3 \right\}$ with positive probability measures, given by
\begin{align}
    &\1A^{(1)}_{U_1} = [0, 1), &&\1A^{(2)}_{U_1} = (1, 2), &&\1A^{(3)}_{U_1} = (2, 3].
\end{align}
Evidently, one could write probabilities over any intersection of projections in $\1R^{(i)}_{U_1}$ as a summation over some atoms $\1A^{(i)}_{U_1}$. To witness, we show in \Cref{fig:app_a3} more fine-grained partitions over exogenous domains associated with $X, Y, Z$ following the decomposition of atoms $\left \{\1A^{(i)}_{U_1} \mid i = 1, \dots, 3 \right\}$.

In general, one could represent probabilities of any event generated by a finite set of projections $\left \{\1R^{(i)}_{U_1} \mid i = 1, \dots, N \right \}$ using the decomposition of atoms. However, as the number of projections $N \to \infty$, there could exist uncountably many such atoms. Therefore, one could not immediately represent their measures as a discrete distribution. Next, we show that it suffices to consider only a countable set of atoms with positive measures. 
\begin{lemma}\label{lem:ctf_cp5}
  For an SCM $M = \tuple{\*V, \*U, \2F, P}$, for every $U \in \*U$, there exists a countable set of atoms $\left \{ \1A^{(i)}_U \mid i \in \3N \right \}$ defined in \Cref{eq:atom1} such that $\sum_{i \in \3N} P\left( \1A^{(i)}_U\right) = 1$ and for any $\*j \in \bigtimes_{V \in ch(U)} \*J_V$, 
  \begin{align*}
      &P\left( \bigcap_{V \in \ch(U)} \1R_{V,U}^{(j)} \right) = \sum_{i \in \3N} \prod_{V \in \ch(U)} I_{\1A^{(i)}_U \subseteq \1R_{V,U}^{(j)}} P\left( \1A^{(i)}_U\right).
  \end{align*}
\end{lemma}

\begin{figure}[t]
  \centering
  \null
  \begin{subfigure}{0.95\linewidth}\centering
  \resizebox{\linewidth}{!}{
  \begin{tikzpicture}
    \node[box, fill=lightblue] at (0.5, 1.5) {};
    \node[box, fill=lightblue] at (2.5, 0.5) {};

    \node[box, fill=lightgreen] at (1.5, 1.5) {};
    \node[box, fill=lightgreen] at (2.5, 1.5) {};

    \node[box, fill=lightred] at (1.5, 2.5) {};
    \node[box, fill=lightred] at (0.5, 2.5) {};

    \draw[->, >={Latex}] (-0.5,0) -- (3.5,0) node[below] {\scalebox{0.7}{$u_1$}}; 
    \draw[->, >={Latex}] (0,-0.5) -- (0,3.5) node[left] {\scalebox{0.7}{$u_2$}};

    \node (r1) at (0.5, 0.7) {\scalebox{0.7}{$\1U^{(1)}_Y$}};	
    \node (h1) at (0.5, 0.2) {\scalebox{0.7}{$y \gets 0$}};
    
    \node (r1) at (1.5, 0.7) {\scalebox{0.7}{$\1U^{(1)}_Y$}};	
    \node (h1) at (1.5, 0.2) {\scalebox{0.7}{$y \gets 0$}};


    \node (r1) at (2.5, 2.7) {\scalebox{0.7}{$\1U^{(1)}_Y$}};	
    \node (h1) at (2.5, 2.2) {\scalebox{0.7}{$y \gets 0$}};


    \node (r1) at (0.5, 1.7) {\scalebox{0.7}{$\1U^{(3)}_Y$}};	
    \node (h1) at (0.5, 1.2) {\scalebox{0.7}{$y \gets \neg x$}};


    \node (r1) at (1.5, 1.7) {\scalebox{0.7}{$\1U^{(4)}_Y$}};	
    \node (h1) at (1.5, 1.2) {\scalebox{0.7}{$y \gets 1$}};
    
    \node (r1) at (2.5, 1.7) {\scalebox{0.7}{$\1U^{(4)}_Y$}};	
    \node (h1) at (2.5, 1.2) {\scalebox{0.7}{$y \gets 1$}};

    \node (r1) at (2.5, 0.7) {\scalebox{0.7}{$\1U^{(3)}_Y$}};	
    \node (h1) at (2.5, 0.2) {\scalebox{0.7}{$y \gets \neg x$}};

    \node (r1) at (0.5, 2.7) {\scalebox{0.7}{$\1U^{(2)}_Y$}};	
    \node (h1) at (0.5, 2.2) {\scalebox{0.7}{$y \gets x$}};
    
    \node (r1) at (1.5, 2.7) {\scalebox{0.7}{$\1U^{(2)}_Y$}};	
    \node (h1) at (1.5, 2.2) {\scalebox{0.7}{$y \gets x$}};

    \draw[thick, dashed, -] (0, 1) -- (3, 1);
    \draw[thick, dashed, -] (0, 2) -- (3, 2);
    \draw[thick, dashed, -] (0, 3) -- (3, 3);

    \draw[thick, dashed, -] (1, 0) -- (1, 3);
    \draw[thick, dashed, -] (2, 0) -- (2, 3);
    \draw[thick, dashed, -] (3, 0) -- (3, 3);
    
    \draw[thin, -] (1, 0) -- (1, 0.05);

    \node at (-0.2, -0.2) {\scalebox{0.7}{0}};

    \node [below] at (1, 0) {\scalebox{0.7}{1}};	
    \node [below] at (2, 0) {\scalebox{0.7}{2}};
    \node [below] at (3, 0) {\scalebox{0.7}{3}};
    \node [left] at (0, 1) {\scalebox{0.7}{1}};
    \node [left] at (0, 2) {\scalebox{0.7}{2}};
    \node [left] at (0, 3) {\scalebox{0.7}{3}};
  \end{tikzpicture}
  }
  \caption{$y \gets f_Y(x, u_1, u_2)$}
  \label{fig:app_a3a}
  \end{subfigure}
\hfill
  \begin{subfigure}{0.95\linewidth}\centering
  \resizebox{\linewidth}{!}{
  \begin{tikzpicture}
    \draw[->, >={Latex}] (-0.5,0) -- (3.5,0) node[below] {\scalebox{0.7}{$u_2$}}; 

    \fill[fill=lightred] (2,0.01) rectangle (3,0.5);
    \fill[fill=lightblue] (0,0.01) rectangle (1,0.5);
    \fill[fill=lightgreen] (1,0.01) rectangle (2,0.5);

    \node (r1) at (2.5, 0.33) {\scalebox{0.7}{$\1U^{(2)}_X$}};	
    \node (h1) at (2.5, 0.1) {\scalebox{0.7}{$x \gets z$}};

    \node (r1) at (0.5, 0.33) {\scalebox{0.7}{$\1U^{(3)}_X$}};	
    \node (h1) at (0.5, 0.1) {\scalebox{0.7}{$x \gets \neg z$}};

    \node (r1) at (1.5, 0.33) {\scalebox{0.7}{$\1U^{(4)}_X$}};	
    \node (h1) at (1.5, 0.1) {\scalebox{0.7}{$x \gets 1$}};

    \draw[thick, dashed, -] (0, 0) -- (0, 0.5);
    \draw[thick, dashed, -] (1, 0) -- (1, 0.5);
    \draw[thick, dashed, -] (2, 0) -- (2, 0.5);
    \draw[thick, dashed, -] (3, 0) -- (3, 0.5);

    \node [below] at (0, 0) {\scalebox{0.7}{0}};	
    \node [below] at (1, 0) {\scalebox{0.7}{1}};	
    \node [below] at (2, 0) {\scalebox{0.7}{2}};
    \node [below] at (3, 0) {\scalebox{0.7}{3}};	
  \end{tikzpicture}
  }
  \caption{$x \gets f_X(z, u_2)$}
  \label{fig:app_a3b}
  \end{subfigure}\hfill
  \begin{subfigure}{0.95\linewidth}\centering
  \resizebox{\linewidth}{!}{
  \begin{tikzpicture}
    \draw[->, >={Latex}] (-0.5,0) -- (3.5,0) node[below] {\scalebox{0.7}{$u_1$}}; 

    \fill[fill=lightred] (0,0.01) rectangle (1,0.5);

    \node (r1) at (1.5, 0.33) {\scalebox{0.7}{$\1U^{(1)}_Z$}};	
    \node (h1) at (1.5, 0.1) {\scalebox{0.7}{$z \gets 0$}};
    
    \node (r1) at (2.5, 0.33) {\scalebox{0.7}{$\1U^{(1)}_Z$}};	
    \node (h1) at (2.5, 0.1) {\scalebox{0.7}{$z \gets 0$}};

    \node (r1) at (0.5, 0.33) {\scalebox{0.7}{$\1U^{(2)}_Z$}};	
    \node (h1) at (0.5, 0.1) {\scalebox{0.7}{$z \gets 1$}};

    \draw[thick, dashed, -] (0, 0) -- (0, 0.5);
    \draw[thick, dashed, -] (1, 0) -- (1, 0.5);
    \draw[thick, dashed, -] (2, 0) -- (2, 0.5);
    \draw[thick, dashed, -] (3, 0) -- (3, 0.5);
    
    \node [below] at (0, 0) {\scalebox{0.7}{0}};	
    \node [below] at (1, 0) {\scalebox{0.7}{1}};	
    \node [below] at (2, 0) {\scalebox{0.7}{2}};
    \node [below] at (3, 0) {\scalebox{0.7}{3}};

  \end{tikzpicture}
  }
  \caption{$z \gets f_Z(u_1)$}
  \label{fig:app_a3c}
  \end{subfigure}\null
  \caption{More fine-grained Partitions of exogenous domains associated with $X, Y, Z$ based on atoms. Each equivalence class (e.g., $\1U^{(i)}_Y$) is decomposed into a finite set of pairwise disjoint cells formed by atoms (e.g., $[0, 1) \times [1, 2]$).}
  \label{fig:app_a3}
\end{figure}
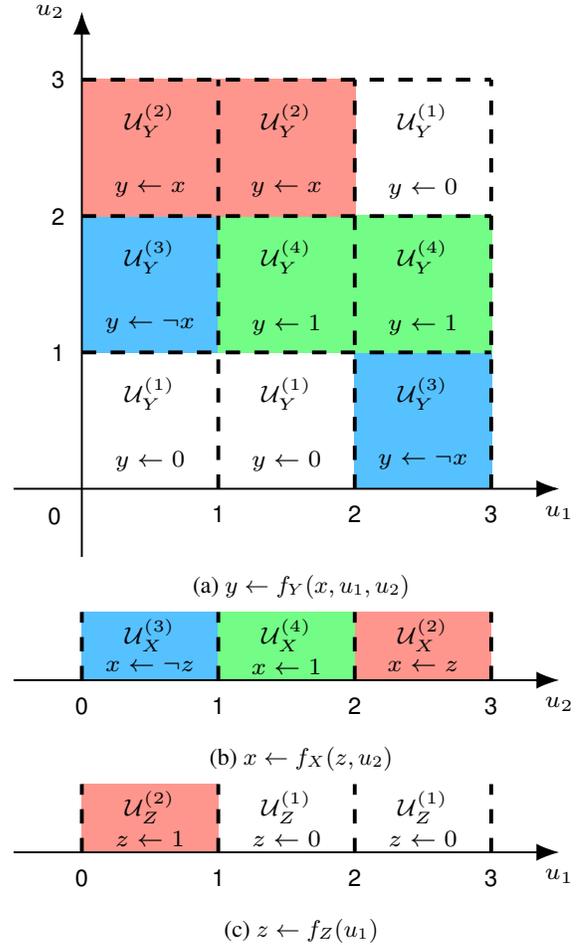

\begin{proof}
    Formally, we define
    \begin{align*}
        \1H = \left \{ \bigcup_{\*J \in \2J} \1A^{(\*J)}_U \mid \2J \subseteq 2^{\3N} \text{ countable or co-countable} \right \}. 
    \end{align*}
    It is verifiable that $\1H$ is a $\sigma$-algebra generated by projections $\left \{ \1R_{U}^{(j)} \mid j \in \mathbb{N}\right\}$. Furthermore, the intersection $\bigcap_{V \in \ch(U)} \1R_{V,U}^{(j)}$ is a measurable set in $\1H$. Therefore, it is sufficient to show that for any event $\1B_U \in \1H$,
    \begin{align}
      P\left(\1B_U \right) = \sum_{i \in \3N} \I_{\1A^{(i)}_U \subseteq \1B_U} P\left( \1A^{(i)}_U\right). \label{eq:ctf_cp6}
    \end{align}
    We first show that there exists a countable set of atoms $\left \{ \1A^{(i)}_U \mid i \in \3N \right \}$ that covers domain $\D_U$, i.e.,
    \begin{align}
        P\left (\D_U \setminus \bigcup_{i \in \3N} \1A^{(i)}_U \right) = 0. \label{eq:atom2}
    \end{align}
    Take $\1B^{(1)}_U = \D_U$ and define by induction, for all $i \in \3N$, if $P\left(\1B^{(i)}_U \right) > 0$, then let
    \begin{align}
        \1B^{(i+1)}_U = \1B^{(i)}_U \setminus \1A^{(i)}_U  
    \end{align}
    where $\1A^{(i)}_U \subseteq \1B^{(i)}_U$ is an atom with the largest positive measure among all atoms contained in $\1B^{(i)}_U$.
    
    If we repeatedly apply the above construction, one of two things may happen:
    \begin{enumerate}
        \item For some $n \in \3N$, $P\left(\1B^{(i)}_U \right) = 0$ and in this case, $\1A^{(1)}_U, \dots, \1A^{(n-1)}_U$ satisfy \Cref{eq:atom2}.
        \item For all $i \in \3N$, $P\left(\1B^{(i)}_U \right) > 0$. In this case, we have a countable set of atoms $\left \{ \1A^{(i)}_U \mid i \in \3N \right \}$ of positive measures. We now prove \Cref{eq:atom2} by contradiction. If \Cref{eq:atom2} does not hold, then there is an atom $\1A \subseteq \D_{U_V} \setminus \bigcup_{i \in \3N} \1A^{(i)}_U$ such that $P(\1A) > 0$. By our choice of $\1A^{(i)}_U$ at each step, we have that, for all $i \in \3N$, 
        \begin{align}
            P\left(\1A^{(i)}_U\right) \geq P(\1A).
        \end{align}
        Therefore, 
        \begin{align}
             P\left (\bigcup_{i \in \3N} \1A^{(i)}_U \right) = \sum_{i \in \3N}P\left(\1A^{(i)}_U\right) = \infty.
        \end{align}
        Contradiction, since the probability measure $P$ is finite.
    \end{enumerate}
    For any event $\1B_U \in \1H$, since $\1B_U \subseteq \D_U$, \Cref{eq:atom2} implies 
    \begin{align}
        P\left (\1B_U \setminus \bigcup_{i \in \3N} \1A^{(i)}_U \right) \leq P\left (\D_U \setminus \bigcup_{i \in \3N} \1A^{(i)}_U \right) = 0.
    \end{align}
    Therefore,
    \begin{align}
        P\left(\1B_U \right) &= P\left(\1B_U \cap \bigcup_{i \in \3N} \1A^{(i)}_U \right) \\
        &=\sum_{i \in \3N} P\left( \1B_U \cap \1A^{(i)}_U\right)\\
        &=\sum_{i \in \3N} \I_{\1A^{(i)}_U \subseteq \1B_U} P\left( \1A^{(i)}_U\right)
    \end{align}
    The last step holds since $\1B_U$ is a union of atoms and atoms are pairwise disjoint. This completes proof.
\end{proof}
We are now ready to prove the counterfactual equivalence for canonical SCMs with discrete exogenous domains.
\begin{restatable}{lemma}{lemdscm}\label{lem:d-scm}
  For a DAG $\G$, let $M$ be an arbitrary SCM compatible with $\G$. There exists a discrete SCM $N$ compatible with $\G$ such that $\*P^*_M = \*P^*_N$, i.e., $M$ and $N$ coincide in all counterfactual distributions.
\end{restatable}
\begin{proof}
    Let $\left \{ \1A^{(i)}_U \mid i \in \3N \right \}$ the countable set of atoms defined in \Cref{lem:ctf_cp5} for every $U \in \*U$. We construct a discrete SCM $N$ from $M$ as follows. 
    \begin{enumerate}
        \item For every $U \in \*U$, pick an arbitrary constant $u^{(i)}$ in each atom $\1A^{(i)}_U$.
        \item Define distribution $P(U)$ for every $U \in \*U$ in $N$ as:
            \begin{align}
                P_N\left(U = u^{(i)} \right) = P_M\left ( \1A^{(i)}_U \right).
            \end{align}
    \end{enumerate}
    Doing so generates a discrete SCM $N$ satisfying conditions as follows: 
    \begin{enumerate}
    \item $N$ is compatible with $\G$;
    \item $N$ and $M$ share the same set of structural functions $\2F$;
    \item $N$ and $M$ generate the same distribution over the intersections of projections of covering cells defined in \Cref{eq:prob_proj}.
    \end{enumerate}
     It follows from \Cref{lem:ctf_cp1,lem:ctf_cp4} that $M$ and $N$ must coincide in all counterfactual distributions $\*P^*$ over endogenous variables. This completes the proof.
\end{proof}
A mental image for the discretization procedure in \Cref{lem:d-scm} is described as follows. We first partition the exogenous domain $\D_{U}$ for each $U \in \*U$ into a countable collection of atoms. By doing so, we obtain a partition over the product domain $\D_{U_V} = \bigtimes_{U \in U_V}\D_U$ for every $V \in \*V$. Such a partition consists of countably many (almost) disjoint covering cells (\Cref{def:rec}) formed by products of atoms (\Cref{def:atom}). Every cell is assigned with a unique function $h_V$ in the hypothesis class $\D_{\PA_V}\mapsto \D_V$ mapping from domains of input $\PA_V$ to $V$. Given any configuration $\*U = \*u$, for every $V \in \*V$, one could find the cell containing the constant $u_V$ and generate values of $V$ following the associated function $h_V$. As an example, we show in \Cref{fig:app_a3} a graphical illustration for this discretization procedure for the SCM described in \Cref{eq:fxyz2}.

Finally, to construct a discrete SCM, it is sufficient to pick an arbitrary constant $u^{(i)}$ in each atom, assign it with the probability measure over the corresponding atom, and replace the exogenous $U$ with a variable drawn from a discrete distribution over constants $u^{(i)}$. Repeatedly applying this procedure for every exogenous $U \in \*U$ results in a canonical SCM with discrete exogenous domains. Also, one could further reduce cardinalities of exogenous domains by shrinking the support of the constructed discrete distribution. This could be done by re-weighting probabilities assigned to constant $u^{(i)}$ in each atom while maintaining probabilities over canonical partitions. Indeed, it is possible to bound the total number of atoms with positive probabilities to a finite value. The existence of such probability measures is guaranteed by the classic result of Carath\'eodory theorem \citep{caratheodory1911variabilitatsbereich}, following a similar procedure in the proof of \Cref{lem:ctf_cp4}.
\clearpage
\section{B. Markov Chain Monte Carlo for Partial Counterfactual Identification} 
In this section, we will show derivations for complete conditional distributions utilized in our proposed Gibbs samplers. We will also provide proofs for non-asymptotic bounds for empirical estimates of credible intervals used in \Cref{alg:ci}.

\subsection{B.1 Derivations of Complete Conditionals}
\paragraph{Sampling $P\left (\5u \mid \5v, \*\theta, \*\mu \right)$.}It is verifiable that variables $\*U^{(n)}, \*V^{(n)}$, $n = 1, \dots, N$, are mutually independent given parameters $\*\theta, \*\mu$. This implies 
\begin{align*}
    P\left (\5u \mid \5v, \*\theta, \*\mu \right) &= \prod_{U \in \*U}P\left (\*u^{(n)} \mid \5v, \*\theta, \*\mu \right)\\
    &=\prod_{U \in \*U}P\left (\*u^{(n)} \mid \*v^{(n)}, \*\theta, \*\mu \right)
\end{align*}
The complete conditional over $\left (\*U^{(n)} \mid \*V^{(n)}, \*\theta, \*\mu \right)$, $n = 1, \dots, N$, is given by
\begin{align*}
    P\left (\*u^{(n)} \mid \*v^{(n)}, \*\theta, \*\mu \right) &\propto P\left (\*u^{(n)} \*v^{(n)} \mid \*\theta, \*\mu \right)\\
    &\propto \prod_{V \in \*V} P\left (v^{(n)} \mid \pa^{(n)}_V, u^{(n)}_V, \*\theta, \*\mu \right)\\
    &\cdot \prod_{U \in \*U} P\left (u^{(n)}_V \mid \*\theta, \*\mu \right).
\end{align*}
Among quantities in the above equation, 
\begin{align*}
    &P\left (v^{(n)} \mid \pa^{(n)}_V, u^{(n)}_V, \*\theta, \*\mu \right) = \mu^{\left(\pa^{(n)}_V, u^{(n)} \right)}_{v^{(n)}},
\end{align*}
and 
\begin{align*}
    &P\left (u^{(n)}_V \mid \*\theta, \*\mu \right) = \theta_u \;\; \text{for} \;\; u = u^{(n)}_V.
\end{align*}

\paragraph{Sampling $P\left (\*\mu, \*\theta \mid \5v, \5u \right)$.}For every exogenous variable $U \in \*U$, we denote by $\*\theta_U$ the set of parameters $\left\{\theta_u\mid \forall u \right \}$. Similarly, for every endogenous variable $V \in \*V$, let $\*\mu_V = \left \{\mu_V^{(\pa_V, u_V)}\mid \forall \pa_V, u_V \right\}$. Obviously, parameters $\*\mu_V$ and $\*\theta_U$ are mutually independent, and they do not directly determine values of a variable (exogenous or endogenous) simultaneously. We must have
\begin{align*}
    P\left (\*\mu, \*\theta \mid \5v, \5u \right) = \prod_{V \in \*V} P\left (\mu_V \mid \5v, \5u \right)\prod_{U \in \*U}P\left (\theta_U \mid \5v, \5u \right).
\end{align*}
The above independence relationship implies that to draw samples from the posterior distribution $P\left (\*\mu, \*\theta \mid \5v, \5u \right)$, we could sample distributions over $\left (\mu_V \mid \5V, \5U \right)$ and $\left (\theta_U \mid \5V, \5U \right)$ for every $V \in \*V$ and every $U \in \*U$ separately.

Recall that for every $V \in \*V$, any $\pa_V, u_V$, $\mu_V^{(\pa_V, u_V)}=\left ( \mu_v^{(\pa_V, u_V)} \mid \forall v \in \D_V \right)$ is an indicator vector such that 
\begin{align*}
  &\mu_v^{(\pa_V, u_V)}\in \{0, 1\}, &\sum_{v \in \D_V} \mu_v^{(\pa_V, u_V)} = 1.
\end{align*}  
The complete conditional distribution over $\left (\*\mu_V \mid \5V, \5U \right)$, given by \Cref{eq:sample2}, follows from the fact that in any discrete SCM, the $n$-th observation of $V \in \*V \setminus \*Z^{(n)}$ is decided by 
\begin{align*}
    v^{(n)} \gets f_V\left (\pa^{(n)}_V , u^{(n)}_V \right) = v,
\end{align*}
where $v$ is a unique element in $\D_V$ such that $\mu_v^{\left(\pa_V, u_V \right)} = 1$.

The complete conditional distribution over $\left (\theta_U \mid \5V, \5U \right)$, given by \Cref{eq:sample3}, follows from the conjugacy of Dirichlet distributions with regard to categorical distributions (e.g., see \citep[Sec.~5.2]{ishwaran2001gibbs}). 

\paragraph{Sampling $P\left (\*u^{(n)} \mid \5v, \5u_{-n} \right)$.}At each iteration, draw $\*U^{(n)}$ from the conditional distribution given by
\begin{align*}
    P\left (\*u^{(n)} \mid \5v, \5u_{-n}\right) \notag \\
    \propto \prod_{V \in \*V \setminus \*Z^{(n)}} &P\left (v^{(n)} \mid \pa^{(n)}_{V}, u^{(n)}_{V}, \5v_{-n}, \5u_{-n} \right) \notag \\
    \prod_{U \in \*U} &P\left (u^{(n)} \mid \5v_{-n}, \5u_{-n}\right). 
\end{align*}
Among quantities in the above equation, by expanding on valus of parameters $\mu_V^{(\pa_V, u_V)}$, one could rewrite the posterior distribution $P\left (v^{(n)} \mid \pa^{(n)}_{V}, u^{(n)}_{V}, \5v_{-n}, \5u_{-n} \right)$ for every $V \in \*V \setminus \*Z^{(n)}$ as follows
\begin{align}
    &P\left (v^{(n)} \mid \pa^{(n)}_V, u^{(n)}_V, \5v_{-n}, \5u_{-n} \right) \notag \\
    &= \sum_{\pa_V, u_V}\sum_{\mu_V^{\left(\pa_V, u_V \right)}}  \mu_{v^{(n)}}^{\left(\pa_V, u_V \right)} \I_{\pa_V = \pa^{(n)}_V} \I_{u_V = u^{(n)}_V} \notag \\
    &\cdot P\left (\mu_V^{\left(\pa_V, u_V \right)}\mid \5v_{-n}, \5u_{-n} \right). \label{eq:app_f2}
\end{align}
The complete conditional over $\left (\mu_V^{\left(\pa_V, u_V \right)}\mid \5V_{-n}, \5V_{-n} \right)$, $\forall \pa_V, u_V$, follows from the definition of discrete SCMs. The $n$-th observation of $V \in \*V \setminus \*Z^{(n)}$ is decided by 
\begin{align*}
    v^{(n)} \gets f_V\left (\pa^{(n)}_V , u^{(n)}_V \right) = v,
\end{align*}
for a unique $v \in \D_V$ such that $\mu_v^{\left(\pa_V, u_V \right)} = 1$. Formally, if there exists a sample $i \neq n$ such that $V \not \in \*Z^{(i)}$ and $\pa^{(i)}_{V} = \pa_V, u^{(i)}_{V} = u_V$, the posterior over $\mu^{(\pa_V, u_V)}_V$ is given by
\begin{align*}
  P\left (\mu^{(\pa_V, u_V)}_v = 1 \mid \5v, \5u \right) = \I_{v = v^{(i)}}.
\end{align*}
Otherwise, 
\begin{align*}
  P\left (\mu^{(\pa_V, u_V)}_V \mid \5v, \5u \right) = \frac{1}{|\D_V|}.
\end{align*}
Marginalizing probabilities $P\left (\mu^{(\pa_V, u_V)}_V \mid \5v, \5u \right)$ over the domain $\D_V$ in \Cref{eq:app_f2} gives the complete conditional distribution over $\left (V^{(n)} \mid \PA^{(n)}_V, U^{(n)}_V, \5U_{-n}, \5U_{-n} \right)$.

For every $U \in \*U$, the complete conditional over $\left (U^{(n)} \mid \5V_{-n}, \5U_{-n}\right)$, given by \Cref{eq:sample3}, follows immediately from the P\'olya urn characterization of Dirichlet distributions (e.g., see \citep[Sec.~4]{ishwaran2001gibbs}).

\subsection{B.2 Monte Carlo Estimation of Credible Intervals}
Recall that for samples $\left \{\theta^{(t)} \right \}_{t = 1}^T$ drawn from $P\left (\theta_{\text{ctf}} \mid \5v \right)$, the empirical estimates for $100(1-\alpha)\%$ credible interval over $\theta_{\text{ctf}}$ are defined as:
\begin{align}
  &\hat{l}_{\alpha}(T) = \theta^{(\ceil{(\alpha/2)T})}, &&\hat{r}_{\alpha}(T) = \theta^{(\ceil{(1 - \alpha/2)T})},
\end{align}
where $\theta^{(\ceil{(\alpha/2)T})}, \theta^{(\ceil{(1 - \alpha/2)T})}$ are the $\ceil{(\alpha/2)T}$th smallest and the $\ceil{(1 - \alpha/2)T}$th smallest of $\left \{\theta^{(t)} \right \}$. One could apply standard concentration inequalities to determine a sufficient number of draws $T$ required for obtaining accurate estimates of a $100(1-\alpha)\%$ credible interval.
\lemci*
\begin{proof}
    Fix $\epsilon > 0$. If $\hat{l}_{\alpha}(T) > l_{\alpha + \epsilon}$, this means that there are at most $\ceil{(\alpha/2)T} - 1$ instances in $\left \{ \theta_{\text{ctf}}^{(t)}\right \}_{t = 1}^T$ that are smaller than or equal to $l_{\alpha + \epsilon}$. That is,
    \begin{align*}
        P \left ( \hat{l}_{\alpha}(T) > l_{\alpha + \epsilon} \right) &\leq P \left ( \sum_{t = 1}^T \I_{\theta_{\text{ctf}}^{(t)} \leq l_{\alpha + \epsilon}} \leq \ceil{(\alpha/2)T} - 1\right)\\
        &\leq P \left ( \sum_{t = 1}^T \I_{\theta_{\text{ctf}}^{(t)} \leq l_{\alpha + \epsilon}} \leq (\alpha/2)T \right)\\
        &\leq P \left ( \frac{1}{T} \sum_{t = 1}^T \I_{\theta_{\text{ctf}}^{(t)} \leq l_{\alpha + \epsilon}} \leq \frac{\alpha + \epsilon}{2} - \frac{\epsilon}{2}\right)\\
        &\leq \exp \left( - \frac{T\epsilon^2}{2} \right).
    \end{align*}
    The last step in the above equation follows from the standard Hoeffding's inequality.

    If $\hat{l}_{\alpha}(T) < l_{\alpha - \epsilon}$, this implies that there are at least $\ceil{(\alpha/2)T}$ instances in $\left \{ \theta_{\text{ctf}}^{(t)}\right \}_{t = 1}^T$ that are larger than or equal to $l_{\alpha + \epsilon}$. That is,
    \begin{align*}
        P \left ( \hat{l}_{\alpha}(T) < l_{\alpha - \epsilon} \right) &\leq P \left ( \sum_{t = 1}^T \I_{\theta_{\text{ctf}}^{(t)} \leq l_{\alpha - \epsilon}} \geq \ceil{(\alpha/2)T} \right)\\
        &\leq P \left ( \sum_{t = 1}^T \I_{\theta_{\text{ctf}}^{(t)} \leq l_{\alpha - \epsilon}} \geq (\alpha/2)T \right)\\
        &\leq P \left ( \frac{1}{T} \sum_{t = 1}^T \I_{\theta_{\text{ctf}}^{(t)} \leq l_{\alpha - \epsilon}} \geq \frac{\alpha -  \epsilon}{2} + \frac{\epsilon}{2} \right)\\
        &\leq \exp \left( - \frac{T\epsilon^2}{2} \right).
    \end{align*}
    The last step follows from the standard Hoeffding's inequality. Similarly, we could also show that 
    \begin{align*}
        &P \left ( \hat{h}_{\alpha}(T) < h_{\alpha + \epsilon} \right) \leq \exp \left( - \frac{T\epsilon^2}{2} \right), \\
        &P \left ( \hat{h}_{\alpha}(T) > h_{\alpha - \epsilon} \right) \leq \exp \left( - \frac{T\epsilon^2}{2} \right).
    \end{align*}
    Finally, bounding the error rate by $\delta / 4$ gives:
    \begin{align}
        \exp \left( - \frac{T\epsilon^2}{2} \right) = \frac{\delta}{4} \Rightarrow \epsilon = \sqrt{2T^{-1}\ln(4 / \delta)}.
    \end{align}
    Replacing the error rate $\epsilon$ with $f(T, \delta) = \sqrt{2T^{-1}\ln(4 / \delta)}$ completes the proof.
\end{proof}
As a corollary, it immediately follows from \Cref{lem:ci} that Algorithm \textsc{CredibleInterval} (\Cref{alg:ci}) is guaranteed to from a sufficient estimate of $100(1-\alpha)\%$ credible intervals within the specified margin of errors.
\corolci*
\begin{proof}
    The statement follows immediately from \Cref{lem:ci} by setting $\sqrt{2T^{-1}\ln(4 / \delta)} \leq \epsilon$. 
\end{proof}
\clearpage
\section{C. Simulation Setups and Additional Experiments} \label{appendix:c}
In this section, we will provide details on the simulation setups and preprocessing of datasets. We also conduct additional experiments on other more involved causal diagrams and using skewed hyperparameters for prior distributions. For all experiments, we will focus on Dirichlet priors in \Cref{eq:sec5.1} with hyperparameters $\alpha_U^{(u)} = \alpha_U / d_U$ for some real $\alpha_U > 0$. This is equivalent to drawing probabilities $\theta_u$ from a Dirichlet distribution defined as follows:
\begin{equation}
    \left(\theta_1, \dots, \theta_{d_U} \right) \sim \texttt{Dirichlet}\left (\frac{\alpha_U}{d_U}, \cdots, \frac{\alpha_U}{d_U}\right),
\end{equation}
All experiments were performed on a computer with 32GB memory, implemented in MATLAB. We are migrating the source code to other open-source platforms (e.g., Julia), which will be released once the code migration is done.

\paragraph{Experiment 1: Frontdoor}We study the problem of evaluating interventional probabilities $P(y_x)$ from the observational distribution $P(X, Y, W)$ in the ``Frontdoor'' diagram of \Cref{fig1c}. We collect $N = 10^4$ samples $\5v = \{x^{(n)}, y^{(n)}, w^{(n)}\}_{n = 1}^N$ from an SCM compatible with \Cref{fig1c}. Detailed parametrization of the SCM is provided in the following:
\begin{equation}
    \begin{split}
    &U_1 \sim \texttt{Unif}(0, 1), \\
    &U_2 \sim \texttt{Normal}(0, 1), \\
    &X \sim \texttt{Binomial}(1, \rho_X),\\
    &W \sim \texttt{Binomial}(1, \rho_W), \\
    &Y \sim \texttt{Binomial}(1, \rho_Y),
    \end{split} \label{eq:frontdoor}
\end{equation}
where probabilities $\rho_X, \rho_W, \rho_Y$ are given by
\begin{align*}
  &\rho_X = U_1,\\
  &\rho_W = \frac{1}{1 + \exp(-X - U_2)},\\
  &\rho_Y = \frac{1}{1 + \exp(W - U_1)}.
\end{align*}
Each observation $\left(x^{(n)}, y^{(n)}, w^{(n)}\right)$ is an independent draw from the observational distribution $P(X, Y, W)$. We set hyperparameters $\alpha_{U_1} = d_{U_1} = 8$, $\alpha_{U_1} = d_{U_2} = 4$. 

\paragraph{Experiment 2: PNS}We study the problem of evaluating the counterfactual probability $P(y_x, y'_{x'}) \equiv P(Y_x = y, Y_{x'} = y')$ for any $x \neq x', y \neq y'$ from the observational distribution $P(X, Y)$ in the ``Bow'' diagram of \Cref{fig1d}. We collect $N = 10^3$ observational samples $\5v = \{x^{(n)}, y^{(n)} \}_{n = 1}^N$ from an SCM compatible with \Cref{fig1d}. Detailed parametrization of the SCM is defined as follows:
\begin{equation}
    \begin{split}
    &U \sim \texttt{Normal}(0, 1),\\
    &X \sim \texttt{Binomial}(1, \rho_X),\\
    &E \sim \texttt{Logistic}(0, 1), \\
    &Y \gets \I_{X - U + E + 0.1 > 0},
    \end{split}\label{eq:bow}
\end{equation}
where probabilities $\rho_X$ are given by
\begin{align*}
  &\rho_X = \frac{1}{1 + \exp(U)}.
\end{align*}
Each observation $\left(x^{(n)}, y^{(n)} \right)$ is an independent draw from the observational distribution $P(X, Y)$. In this experiment, we set hyperparameters $\alpha_{U} = d_{U} = 8$. 

\paragraph{Experiment 3: IST}International Stroke Trials (IST) was a large, randomized, open trial of up to $14$ days of antithrombotic therapy after stroke onset \citep{carolei1997international}. The aim was to provide reliable evidence on the efficacy of aspirin and of heparin. The dataset is released under Open Data Commons Attribution License (ODC-By). In particular, the treatment $X$ is a pair $(i, j)$ where $i = 0$ stands for no aspirin allocation, $1$ otherwise; $j = 0$ stands for no heparin allocation, $1$ for median-dosage, and $2$ for high-dosage. The primary outcome $Y \in \{0, \dots, 3\} $ is the health of the patient $6$ months after the treatment, where $0$ stands for death, $1$ for being dependent on the family, $2$ for the partial recovery, and $3$ for the full recovery.

To emulate the presence of unobserved confounding, we filter the experimental data with selection rules $f_X^{(Z)}$, $Z \in \{0, \dots, 9\}$, following a procedure in \citep{zhang2021bounding}. More specifically, we are provided with a collection of IST samples $\{X^{(n)}, Y^{(n)}, U_{2}^{(n)} \}_{n = 1}^N$ where $U_{2}^{(n)}$ is the age of the $n$-th patient. For each data point $\left (X^{(n)}, Y^{(n)}, U_{2}^{(n)} \right)$, we introduce an instrumental variable $Z^{(n)} \in \{0, \dots, 9\}$. Values of the instrumental variable $Z^{(n)}$ for the $n$-th patient are decided by 
\begin{align}
    Z^{(n)} = \floor{10 \times U_1}, \text{ where }U^{(n)}_1 \sim \texttt{Unif}(0, 1).
\end{align}
We then check if $X^{(n)}$ satisfies the following condition 
\begin{align}
    X^{(n)} = \floor{6 \times \rho_X},
\end{align}
where parameter $\rho_X$ is given by
\begin{align*}
  &\rho_X = \frac{1}{1 + \exp\left(-U^{(n)}_2 / 100 - Z^{(n)}/10 \right)}
\end{align*}
If the above condition is satisfied, we keep the data point $\left (X^{(n)}, Y^{(n)}, Z^{(n)}, U_1^{(n)}, U_{2}^{(n)} \right)$ in the dataset; otherwise, the data point is dropped. After this data selection process is complete, we hide columns of variables $U_{1}^{(n)}, U_{2}^{(n)}$. Doing so allows us to obtain $N = 1 \times 10^3$ synthetic observational samples $\5V = \left \{X^{(n)}, Y^{(n)}, Z^{(n)} \right \}_{n = 1}^N$ that are compatible with the  ``IV''' diagram of \Cref{fig1a}. 

In this experiment, we set hyperparameters $\alpha_{U_1} = 10$ and $\alpha_{U_2} = 1$. As a baseline, we estimate the treatment effect $E[Y_{x = (1, 0)}] = 1.3418$ for only assigning aspirin $X = (1, 0)$ from randomized trial data containing $1.9285\times 10^4$ subjects. 

\begin{figure*}[t]
  \centering
  \null
  \begin{subfigure}{0.2\linewidth}\centering
    \begin{tikzpicture}
      \node[vertex] (Z) at (0.75, 0.75) {Z};
      \node[vertex] (X) at (0, 0) {X};
      \node[vertex] (W) at (1.5, 1.5) {W};
      \node[vertex] (Y) at (3, 0) {Y};
      \node[uvertex] (U1) at (0.2, 1.3) {U\textsubscript{1}};
      \node[uvertex] (U2) at (2.25, 0.75) {U\textsubscript{2}};
      \node[uvertex] (U3) at (1.5, 0.2) {U\textsubscript{3}};
      
      \draw[dir] (X) -- (Y);
      \draw[dir] (Z) -- (X);
      \draw[dir] (W) -- (Z);
      \draw[dir, dashed] (U1) -- (W);
      \draw[dir, dashed] (U1) -- (X);
      \draw[dir, dashed] (U3) -- (Z);
      \draw[dir, dashed] (U2) -- (W);
      \draw[dir, dashed] (U2) -- (Y);
  \end{tikzpicture}
  \caption{Napkin}
  \label{fig:app_c2a}
  \end{subfigure}\hfill
  \begin{subfigure}{0.2\linewidth}\centering
    \begin{tikzpicture}
      \node[vertex] (Z) at (0, 0) {Z};
      \node[vertex] (X) at (1.5, 0) {X};
      \node[vertex] (Y) at (3, 0) {Y};
      \node[uvertex, opacity = 0] (dummy) at (0.75, 1.5) {U\textsubscript{1}};
      \node[uvertex, opacity = 0] (dummy) at (2.25, 1.5) {U\textsubscript{2}};
      \node[uvertex] (U1) at (0.75, 1) {U\textsubscript{1}};
      \node[uvertex] (U2) at (2.25, 1) {U\textsubscript{2}};
      
      \draw[dir] (X) -- (Y);
      \draw[dir] (Z) -- (X);
      \draw[dir, dashed] (U1) -- (Z);
      \draw[dir, dashed] (U1) -- (X);
      \draw[dir, dashed] (U2) -- (X);
      \draw[dir, dashed] (U2) -- (Y);
  \end{tikzpicture}
  \caption{Double Bow}
    \label{fig:app_c2b}
  \end{subfigure}\hfill
  \begin{subfigure}{0.2\linewidth}\centering
    \begin{tikzpicture}
      \node[vertex] (X) at (0, 0) {X};
      \node[vertex] (Z) at (1.5, 0.75) {Z};
      \node[vertex] (Y) at (3, 0) {Y};
      \node[uvertex] (U1) at (0.5, 1.5) {U\textsubscript{1}};
      \node[uvertex] (U2) at (2.5, 1.5) {U\textsubscript{2}};
      
      \draw[dir] (X) -- (Y);
      \draw[dir] (Z) -- (X);
      \draw[dir] (Z) -- (Y);
      \draw[dir, dashed] (U1) -- (Z);
      \draw[dir, dashed] (U1) -- (X);
      \draw[dir, dashed] (U2) -- (Z);
      \draw[dir, dashed] (U2) -- (Y);
  \end{tikzpicture}
  \caption{M+BD Graph}
    \label{fig:app_c2c}
  \end{subfigure}\hfill
  \begin{subfigure}{0.3\linewidth}\centering
  \begin{tikzpicture}
    \node[vertex] (Z) at (0, 0) {Z};
    \node[vertex] (W) at (1.5, 0) {W};
    \node[vertex] (X) at (3, 0) {X};
    \node[vertex] (Y) at (4.5, 0) {Y};
    \node[uvertex, opacity = 0] (dummy) at (0.75, 1.5) {U\textsubscript{1}};
    \node[uvertex, opacity = 0] (dummy) at (2.25, 1.5) {U\textsubscript{2}};
    \node[uvertex, opacity = 0] (dummy) at (3.75, 1.5) {U\textsubscript{3}};
    \node[uvertex] (U1) at (0.75, 1) {U\textsubscript{1}};
    \node[uvertex] (U2) at (2.25, 1) {U\textsubscript{2}};
    \node[uvertex] (U3) at (3.75, 1) {U\textsubscript{3}};
    
    \draw[dir] (X) -- (Y);
    \draw[dir] (Z) -- (W);
    \draw[dir] (W) -- (X);
    \draw[dir, dashed] (U1) -- (Z);
    \draw[dir, dashed] (U1) -- (W);
    \draw[dir, dashed] (U3) -- (X);
    \draw[dir, dashed] (U3) -- (Y);
    \draw[dir, dashed] (U2) -- (W);
    \draw[dir, dashed] (U2) -- (X);
\end{tikzpicture}
\caption{Triple Bow}
  \label{fig:app_c2d}
\end{subfigure}\null
  \caption{Causal diagrams for Experiment 5 (\subref{fig:app_c2a}), Experiment 6 (\subref{fig:app_c2b}), Experiment 7 (\subref{fig:app_c2c}), and Experiment 8 (\subref{fig:app_c2d}). Each diagram contains (not exclusively) a treatment $X$, an outcome $Y$, ancestors $Z, W$, and exogenous variables $U_i$, $i = 1, 2, 3$.}
  \label{fig:app_c2}
\end{figure*}
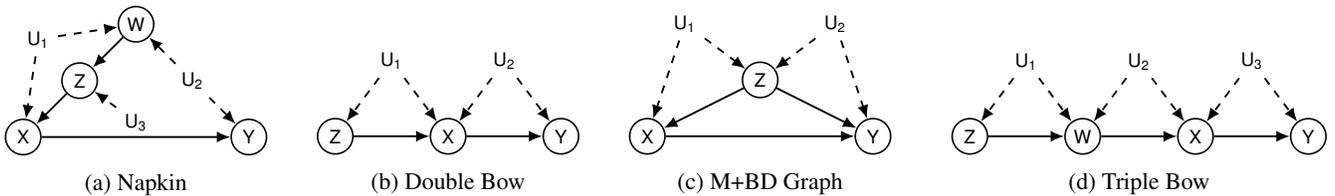

\paragraph{Experiment 4: Obs. + Exp.}We study the problem of evaluating counterfactual probabilities $P(z, x_{z'}, y_{x'})$ from the combination of the observational distribution $P(X, Y, Z)$ and interventional distributions $P(X_z, Y_z)$, $\forall z \in \D_Z$, in the causal diagram of \Cref{fig1b}. We collect $N = 10^3$ samples $\5v = \{x^{(n)}, y^{(n)}, z^{(n)} \}_{n = 1}^N$ from an SCM compatible with \Cref{fig1b}, which we define as follows:
\begin{equation}
  \begin{split}
    &U_1 \sim \texttt{Unif}(0, 1), \\
    &U_2 \sim \texttt{Unif}(0, 1), \\
    &Z \gets \min \left \{\floor{15 \cdot U_1}, 9 \right \}, \\
    &X \sim \texttt{Binomial}(9, \rho_X), \\
    &Y \sim \texttt{Binomial}(9, \rho_Y),
  \end{split} \label{eq:see_do}
\end{equation}
where for any real $\alpha \in \mathbb{R}$, the operator $\floor{\alpha}$ denotes the largest integer $n \in \mathbb{Z}$ smaller than $\alpha$, i.e., $\floor{\alpha} = \min\{n \in \mathbb{Z} \mid n \geq \alpha\}$; probabilities $\rho_X, \rho_Y$ are given by
\begin{align*}
  &\rho_X = \frac{1}{1 + \exp(-Z - U_2)},\\
  &\rho_Y = \frac{1}{1 + \exp(X / 10 - U_1 \cdot U_2)}.
\end{align*}
Each sample $\left(x^{(n)}, y^{(n)}, z^{(n)} \right)$ is an independent draw from the observational distribution $P(X, Y, Z)$ or an interventional distribution $P(X_z, Y_z)$. To obtain a sample from $P(X_z, Y_z)$, we pick a constant $z \in \D_Z$ uniformly at random, perform intervention $\doo(Z = z)$ in the SCM described in \Cref{eq:see_do} and observed subsequent outcomes. In this experiment, we set hyperparameters $\alpha_{U_1} = 10$ and $\alpha_{U_2} = 10$. 

\subsection{C.1 Additional Simulation Results}
We also evaluate our algorithms on various simulated SCM instances in other more involved causal diagrams. Overall, we found that simulation results match our findings in the main manuscript. For identifiable settings (Experiment 5), our algorithms are able to recover the actual, unknown counterfactual probabilities. For non-identifiable settings, our algorithm consistently dominates existing bounding strategies: it achieves sharp bounds if closed-formed solutions exist (Experiments 6); otherwise, it improves over state-of-art bounds (Experiment 7). Finally, for other more challenging non-identifiable settings where existing strategies do not apply (Experiments 8), our algorithm is able to achieve effective bounds over unknown counterfactual probabilities.

In all experiments, we evaluate our proposed strategy using credible intervals (\textit{ci}). In particular, we draw at least $4 \times 10^3$ samples from the posterior distribution $P\left(\theta_{\text{ctf}} \mid \5v \right)$ over the target counterfactual. This allows us to compute $100\%$ credible interval over $\theta_{\text{ctf}}$ within error $\epsilon = 0.05$, with probability at least $1 - \delta = 0.95$. As the baseline, we also include the actual counterfactual probability, labeled as $\theta^*$.

\paragraph{Experiment 5: Napkin Graph}Consider the ``Napkin'' graph in \Cref{fig:app_c2a} where $X, Y, Z, W$ are binary variables in $\{0, 1\}$; $U_1, U_2, U_3$ take values in real $\3R$. The identifiability of interventional probabilities $P(y_x)$ from the observational distribution $P(X, Y, Z, W)$ could be derived by iteratively applying inference rules of ``do-calculus'' \citep[Thm.~4.3.1]{pearl:2k}. We collect $N = 10^4$ observational samples $\5v = \{x^{(n)}, y^{(n)}, z^{(n)}, w^{(n)}\}_{n = 1}^N$ from an SCM compatible with \Cref{fig:app_c2a}, defined as follows:
\begin{equation}
    \begin{split}
    &U_i \sim \texttt{Normal}(0, 1), \;\; i = 1, 2, 3,\\
    &W \sim \texttt{Binomial}(1, \rho_W), \\
    &Z \sim \texttt{Binomial}(1, \rho_Z), \\
    &X \sim \texttt{Binomial}(1, \rho_X), \\
    &Y \sim \texttt{Binomial}(1, \rho_Y),
    \end{split}
\end{equation}
where probabilities $\rho_W, \rho_Z, \rho_X, \rho_Y$ are given by:
\begin{align*}
  &\rho_W = \frac{1}{1 + \exp(U_1 - U_2)},\\
  &\rho_Z = \frac{1}{1 + \exp(W - U_3)}, \\
  &\rho_X = \frac{1}{1 + \exp(-Z - U_1)},\\
  &\rho_Y = \frac{1}{1 + \exp(X - U_2 - 0.5)}.
\end{align*}
Each observation $\left(x^{(n)}, y^{(n)}, z^{(n)}, w^{(n)}\right)$ is an independent draw from the observational distribution $P(X, Y, Z, W)$. In this experiment, we set hyperparameters $\alpha_{U_1} = d_{U_1} = 32$, $\alpha_{U_2} = d_{U_1} = 32$, and $\alpha_{U_3} = d_{U_3} = 4$. \Cref{fig:app_c1a} shows a histogram containing samples drawn from the posterior distribution of $\left ( P(Y_{x = 0} = 1) \mid \5v \right)$. Our analysis reveals that these samples converges to the actual interventional probability $P(Y_{x = 0} = 1) = 0.6098$, which confirms the identifiability of $P(y_x)$ in the napkin graph.


\begin{figure*}[t]
\centering
\null
\begin{subfigure}{0.245\linewidth}\centering
  \includegraphics[width=\linewidth]{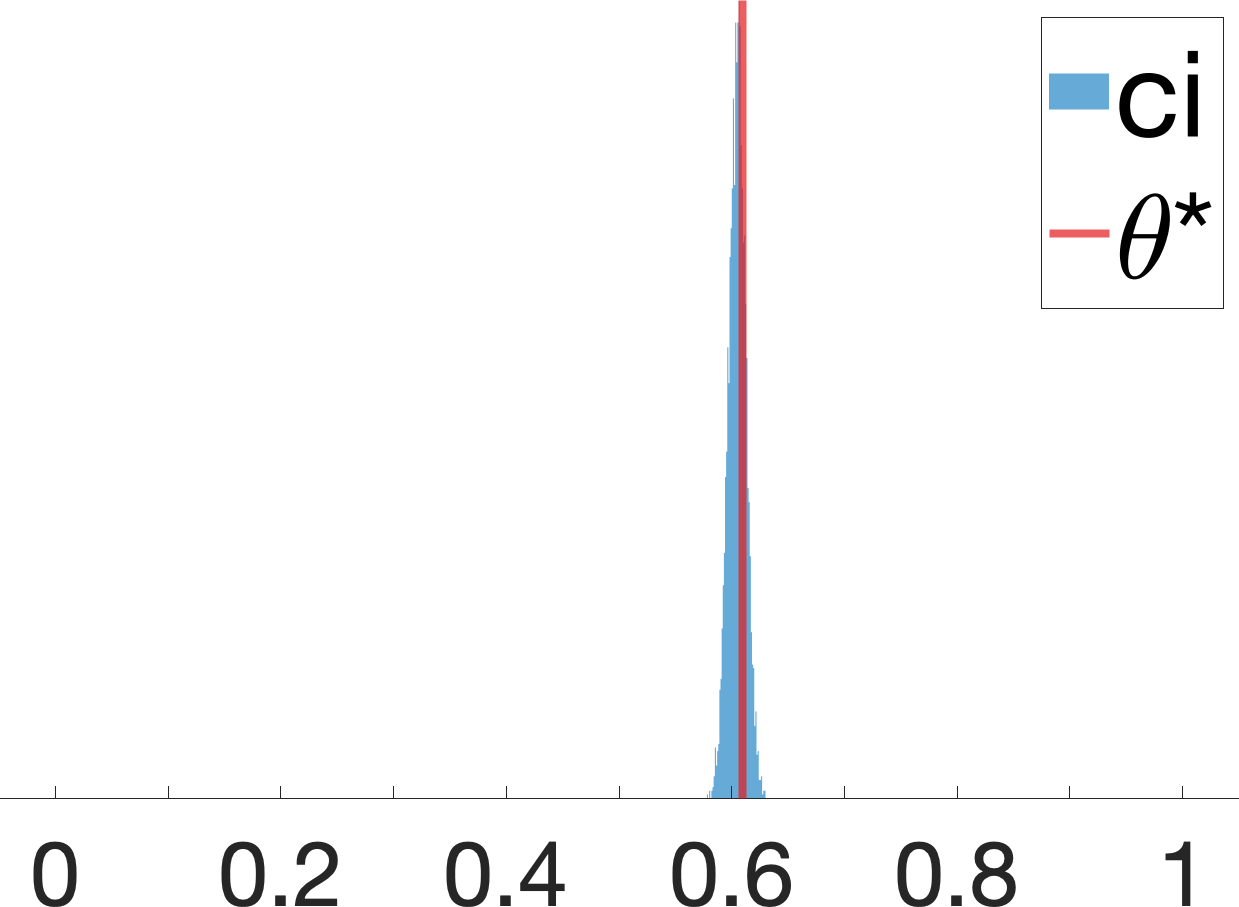}
\caption{Napkin}
\label{fig:app_c1a}
\end{subfigure}\hfill
\begin{subfigure}{0.245\linewidth}\centering
  \includegraphics[width=\linewidth]{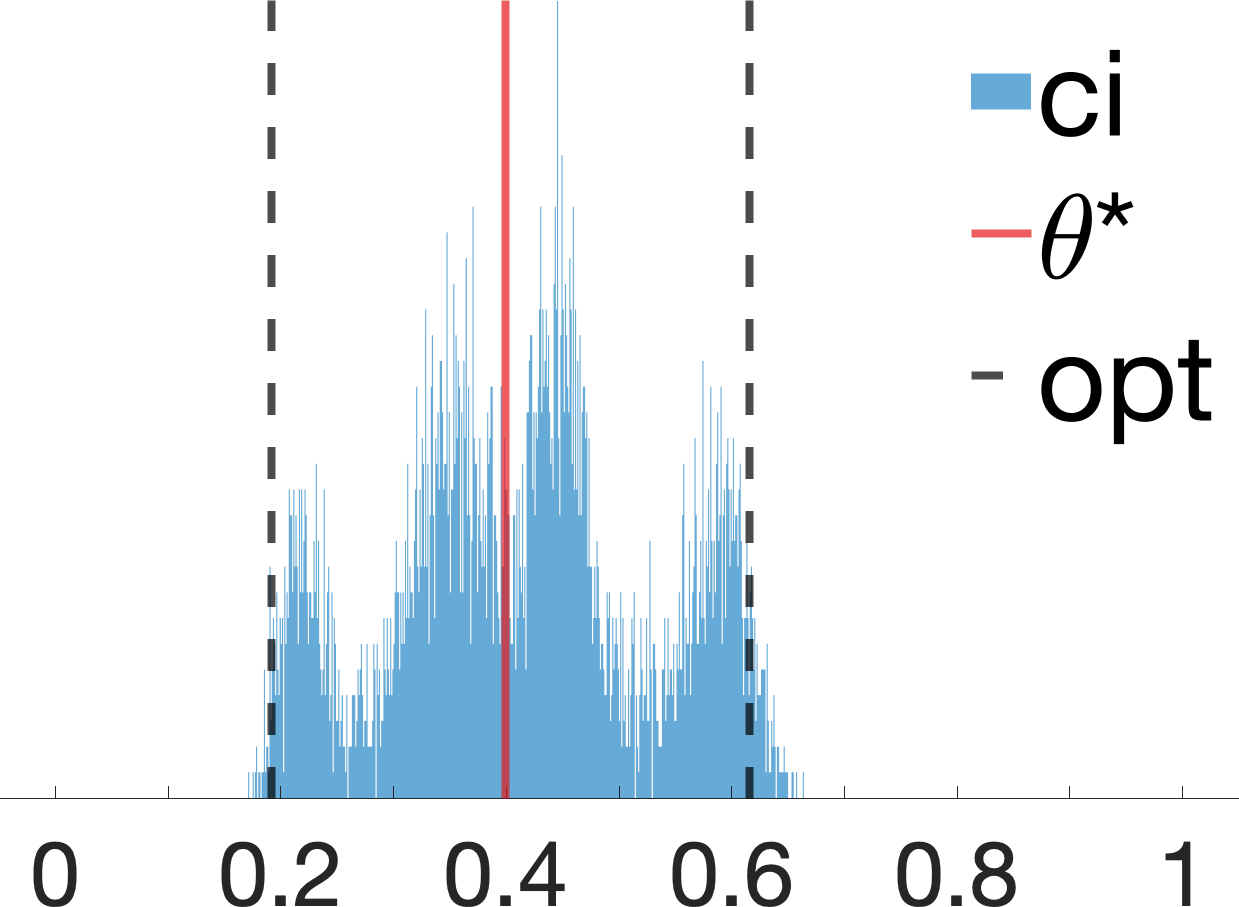}
\caption{Double bow}
\label{fig:app_c1b}
\end{subfigure}\hfill
\begin{subfigure}{0.245\linewidth}\centering
  \includegraphics[width=\linewidth]{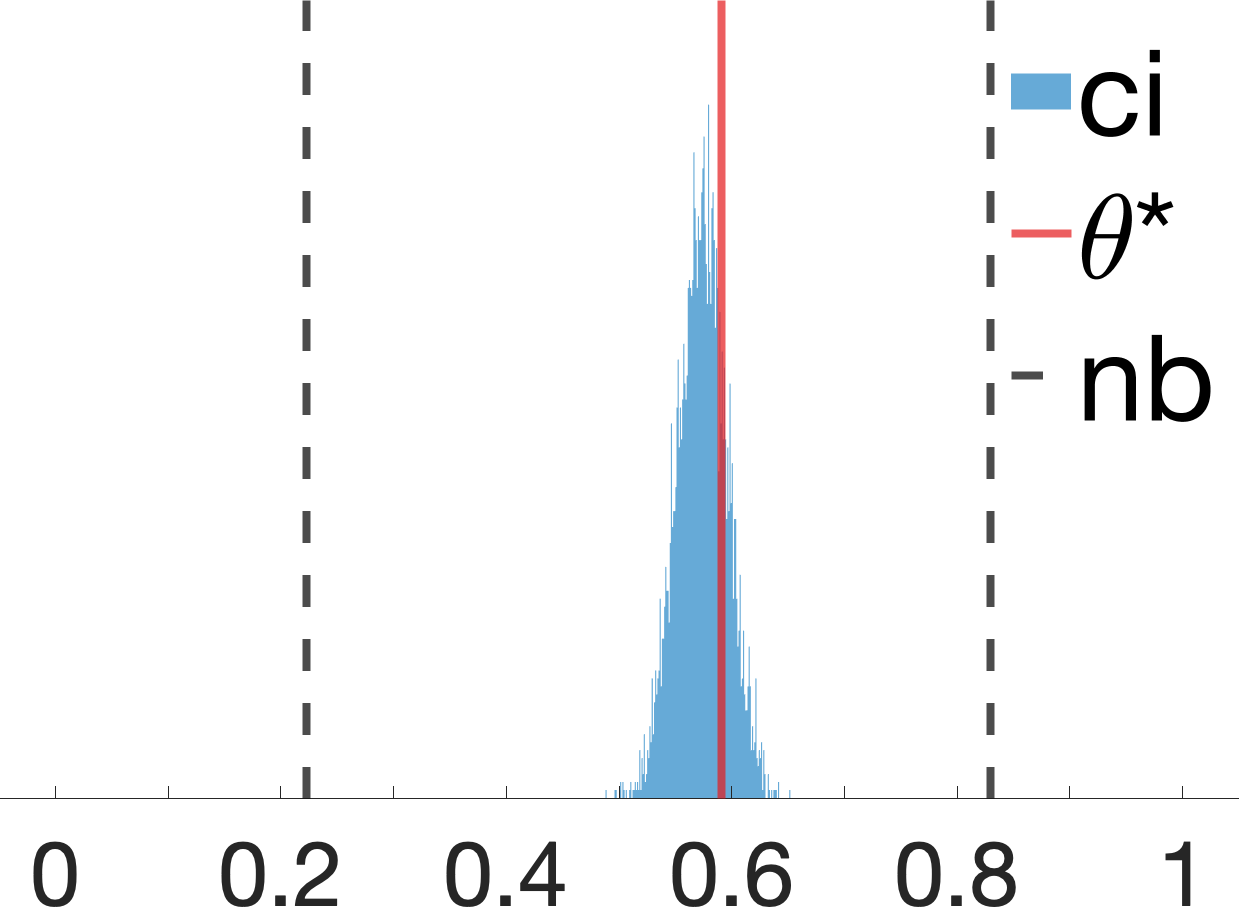}
\caption{M+BD Graph}
\label{fig:app_c1c}
\end{subfigure}\hfill
\begin{subfigure}{0.245\linewidth}\centering
  \includegraphics[width=\linewidth]{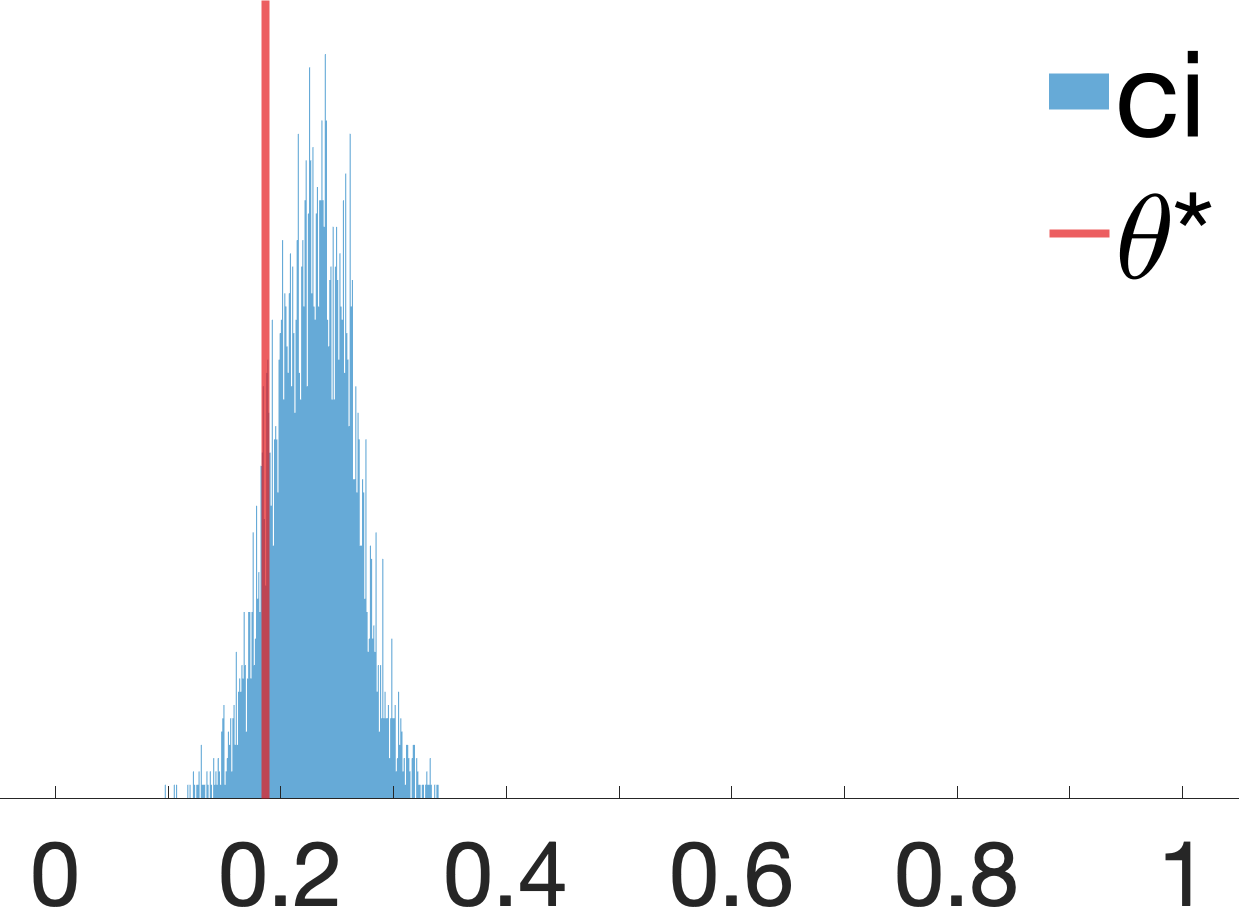}
\caption{Triple bow}
\label{fig:app_c1d}
\end{subfigure}\null
\caption{Histogram plots for samples drawn from the posterior distribution over target counterfactual probabilities. For all plots (\subref{fig:app_c1a} - \subref{fig:app_c1d}), \textit{ci} represents our proposed algorithms; \textit{bp} stands for Gibbs samplers using the representation of canonical partitions \citep{balke:pea94b}; $\theta^*$ is the actual counterfactual probability; \textit{opt} is the optimal asymptotic bounds (if exists); \textit{nb} stands for the natural bounds \citep{manski:90}.}
\label{fig:app_c1}
\end{figure*}

\paragraph{Experiment 6: Double Bow}Consider the ``Double Bow'' diagram in \Cref{fig:app_c2b} where $X, Y, Z \in \{0, 1\}$ and $U_1, U_2 \in \3R$. We study the problem of evaluating interventional probabilities $P(y_x)$ from the observational distribution $P(X, Y, Z)$. We collect $N = 10^3$ observational samples $\5v = \{x^{(n)}, y^{(n)}, z^{(n)}\}_{n = 1}^N$ from an SCM compatible with \Cref{fig1b}. The detailed parametrization of the SCM is defined as follows:
\begin{equation}
    \begin{split}
    &U_i \sim \texttt{Normal}(0, 1), \;\; i = 1, 2, \\
    &Z \sim \texttt{Binomial}(1, \rho_Z),\\
    &X \sim \texttt{Binomial}(1, \rho_X),\\
    &Y \sim \texttt{Binomial}(1, \rho_Y),
    \end{split}
\end{equation}
where probabilities $\rho_Z, \rho_X, \rho_Y$ are given by:
\begin{align*}
  &\rho_Z = \frac{1}{1 + \exp(-U_1)}, \\
  &\rho_X = \frac{1}{1 + \exp(-Z - U_1 - U_2)},\\
  &\rho_Y = \frac{1}{1 + \exp(X - U_2 + 0.5)}.
\end{align*}
Each observation $\left(x^{(n)}, y^{(n)}, z^{(n)} \right)$ is an independent draw from the observational distribution $P(X, Y, Z)$. 

\citep{balke:pea97} introduced a closed-form bound over $P(y_x)$ from the observational distribution $P(X, Y, Z)$ for the ``IV'' diagram in \Cref{fig1a} with binary $X, Y, Z \in \{0, 1\}$. It is verifiable that such a bound is also applicable in \Cref{fig:app_c1b} with binary endogenous domains, and is provably optimal (labeled as \textit{opt}). To obtain a $100\%$ credible intervals, we apply the collapsed Gibbs sampler with hyperparameters $\alpha_{U_1} = d_{U_1} = 32$ and $\alpha_{U_2} = d_{U_1} = 32$. \Cref{fig:app_c1b} shows samples drawn from the posterior distribution of $\left ( P(Y_{x = 0} = 1) \mid \5v \right)$. The analysis reveals that our algorithm derives a valid bound over the actual probability $P(Y_{x = 0} = 1) = 0.3954$; the $100\%$ credible interval converges to the optimal IV bound $l = 0.1980, r = 0.6258$.

\paragraph{Experiment 7: M+BD Graph}Consider the ``M+BD'' graph in \Cref{fig:app_c2c} where $X, Y, Z \in \{0, 1\}$ and $U_1, U_2 \in \3R$. In this case, interventional probabilities $P(y_x)$ are non-identifiable from the observational distribution $P(X, Y, Z)$ due to the presence of the collider path $X \leftarrow U_1 \rightarrow Z \leftarrow U_2 \rightarrow Y$. We collect $N = 10^3$ observational samples $\5v = \{x^{(n)}, y^{(n)}, z^{(n)}\}_{n = 1}^N$ from an SCM compatible with \Cref{fig:app_c2c}. The detailed parametrization of the SCM is provided as follows:
\begin{equation}
    \begin{split}
    &U_i \sim \texttt{Normal}(0, 1), \;\; i = 1, 2, \\
    &Z \sim \texttt{Binomial}(1, \rho_Z), \\
    &X \sim \texttt{Binomial}(1, \rho_X), \\
    &Y \sim \texttt{Binomial}(1, \rho_Y),
    \end{split}
\end{equation}
where probabilities $\rho_Z, \rho_X, \rho_Y$ are given by:
\begin{align*}
  &\rho_Z = \frac{1}{1 + \exp(-U_1)},\\
  &\rho_X = \frac{1}{1 + \exp(-Z - U_1 - U_2)}, \\
  &\rho_Y = \frac{1}{1 + \exp(X - Z - U_2)}.
\end{align*}
Each observation $\left(x^{(n)}, y^{(n)}, z^{(n)} \right)$ is an independent draw from the observational distribution $P(X, Y, Z)$. 

In this experiment, we set hyperparameters $\alpha_{U_1} = d_{U_1} = 32$ and $\alpha_{U_2} = d_{U_1} = 32$. \Cref{fig:app_c1c} shows samples drawn from the posterior distribution of $\left ( P(Y_{x = 0} = 1) \mid \5v \right)$. As a baseline, we also include the natural bounds introduced in \citep{robins:89b,manski:90} (\textit{nb}). The analysis reveals that all algorithms achieve bounds that contain the actual, target causal effect $P(Y_{x = 0} = 1) = 0.5910$. Our algorithm obtains a $100\%$ credible interval $l_{\textit{ci}}= 0.4884, r_{\textit{ci}}= 0.6519$, which improves over the existing bounding strategy ($l_{\textit{nb}}= 0.2230, r_{\textit{nb}}= 0.8296$). 

\begin{figure*}[t]
\centering
\null
\begin{subfigure}{0.245\linewidth}\centering
  \includegraphics[width=\linewidth]{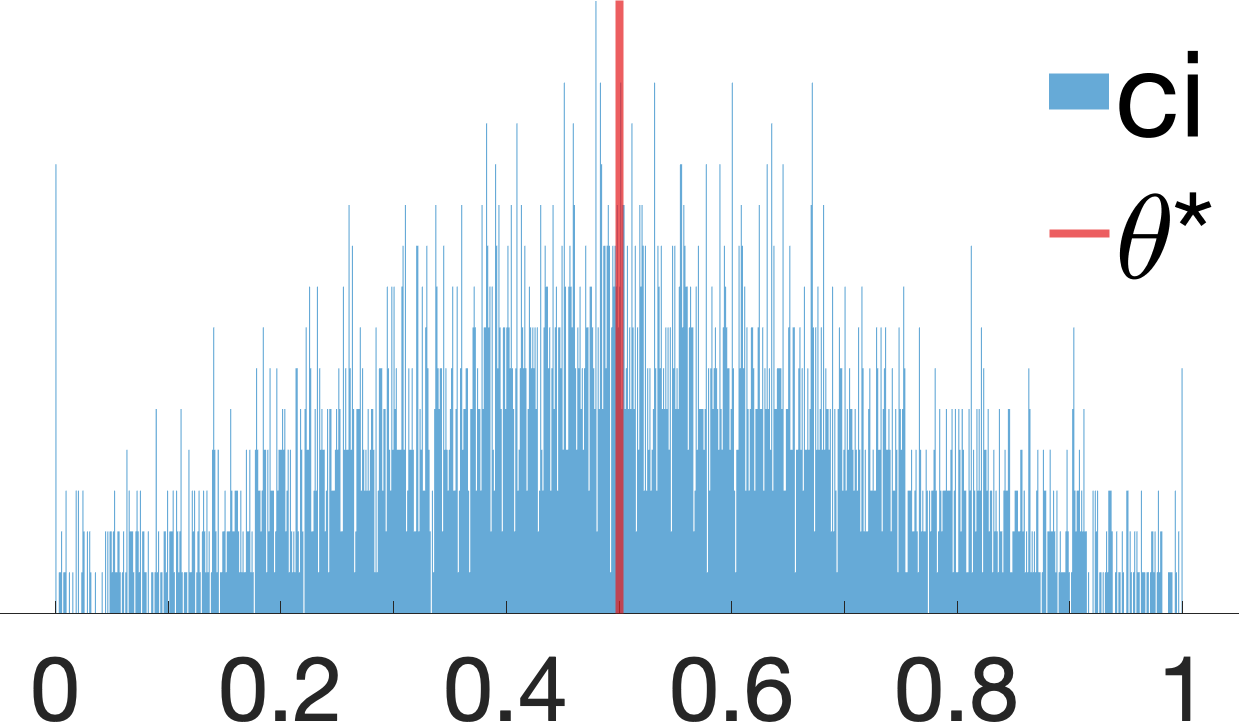}
\caption{Flat}
\label{fig:app_c3a}
\end{subfigure}\hfill
\begin{subfigure}{0.245\linewidth}\centering
  \includegraphics[width=\linewidth]{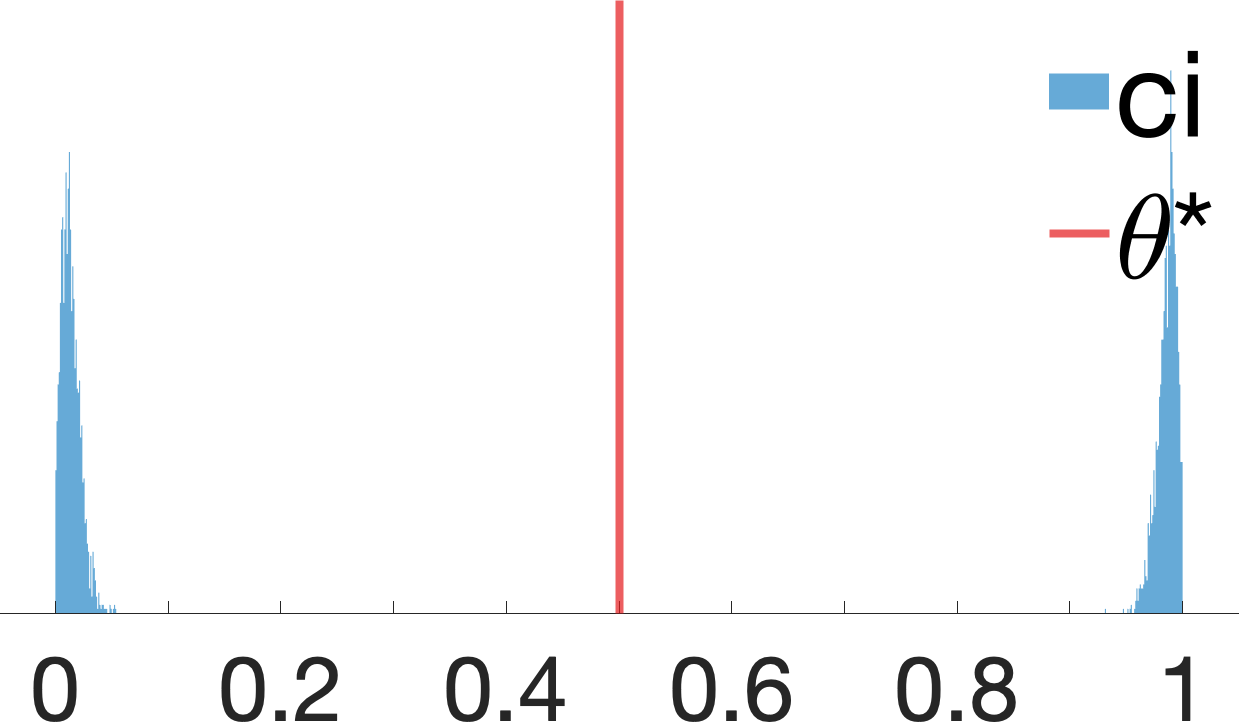}
\caption{Skewed}
\label{fig:app_c3b}
\end{subfigure}\hfill
\begin{subfigure}{0.245\linewidth}\centering
  \includegraphics[width=\linewidth]{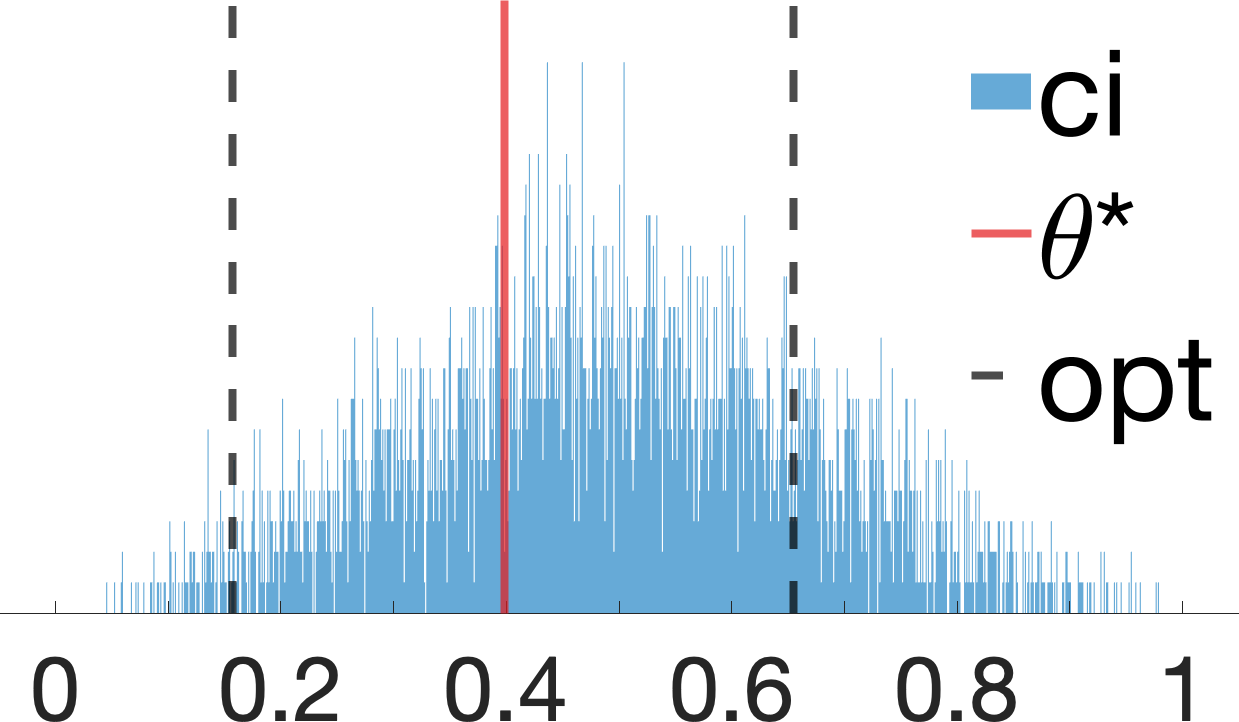}
\caption{Flat}
\label{fig:app_c3c}
\end{subfigure}\hfill
\begin{subfigure}{0.245\linewidth}\centering
  \includegraphics[width=\linewidth]{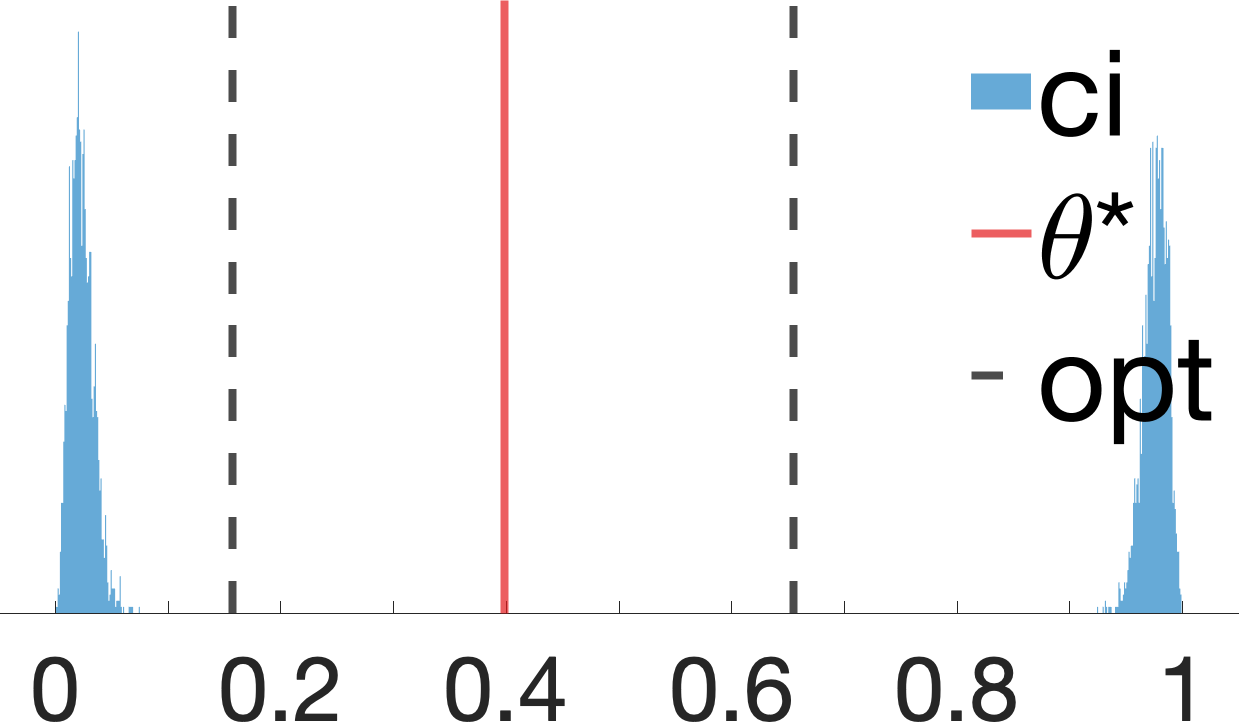}
\caption{Skewed}
\label{fig:app_c3d}
\end{subfigure}\null
\caption{Prior distributions for (\subref{fig:app_c3a}, \subref{fig:app_c3b}) Experiment 9 and (\subref{fig:app_c3c}, \subref{fig:app_c3d}) Experiment 10.}
\label{fig:app_c3}
\end{figure*}

\begin{figure*}[t]
\centering
\null
\begin{subfigure}{0.245\linewidth}\centering
  \includegraphics[width=\linewidth]{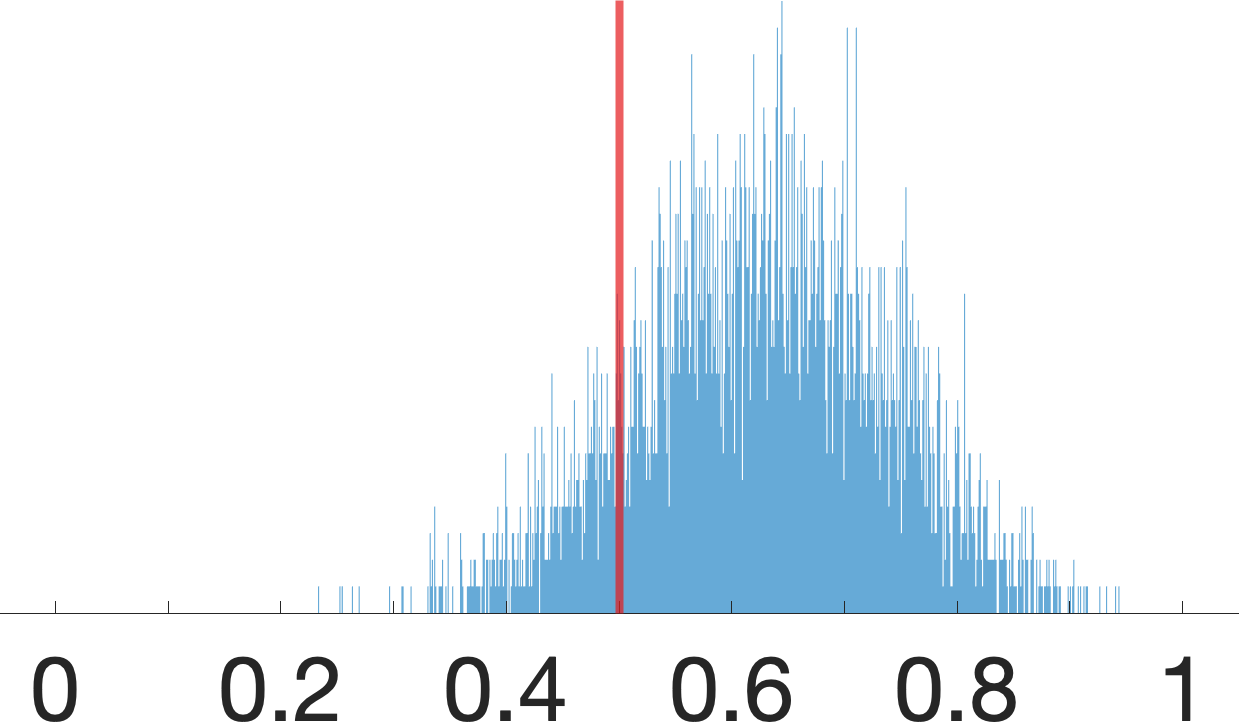}
\caption{$N = 10$}
\label{fig:app_c4a}
\end{subfigure}\hfill
\begin{subfigure}{0.245\linewidth}\centering
  \includegraphics[width=\linewidth]{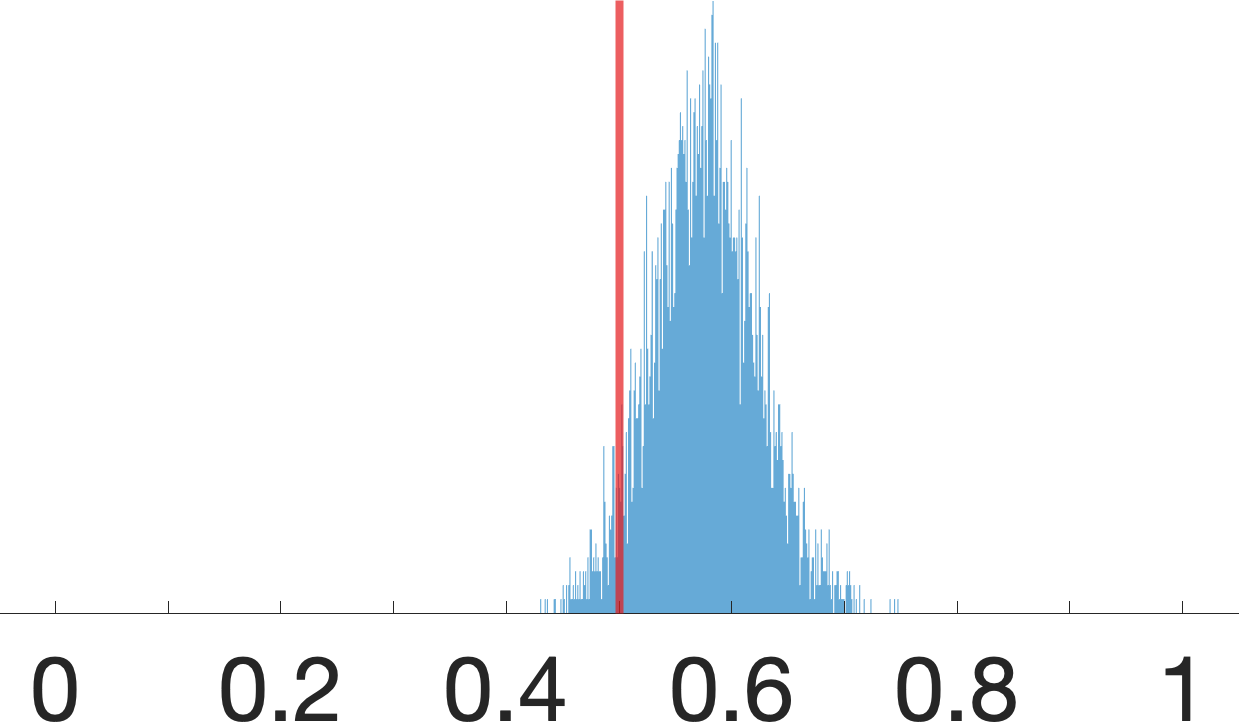}
\caption{$N = 10^2$}
\label{fig:app_c4b}
\end{subfigure}\hfill
\begin{subfigure}{0.245\linewidth}\centering
  \includegraphics[width=\linewidth]{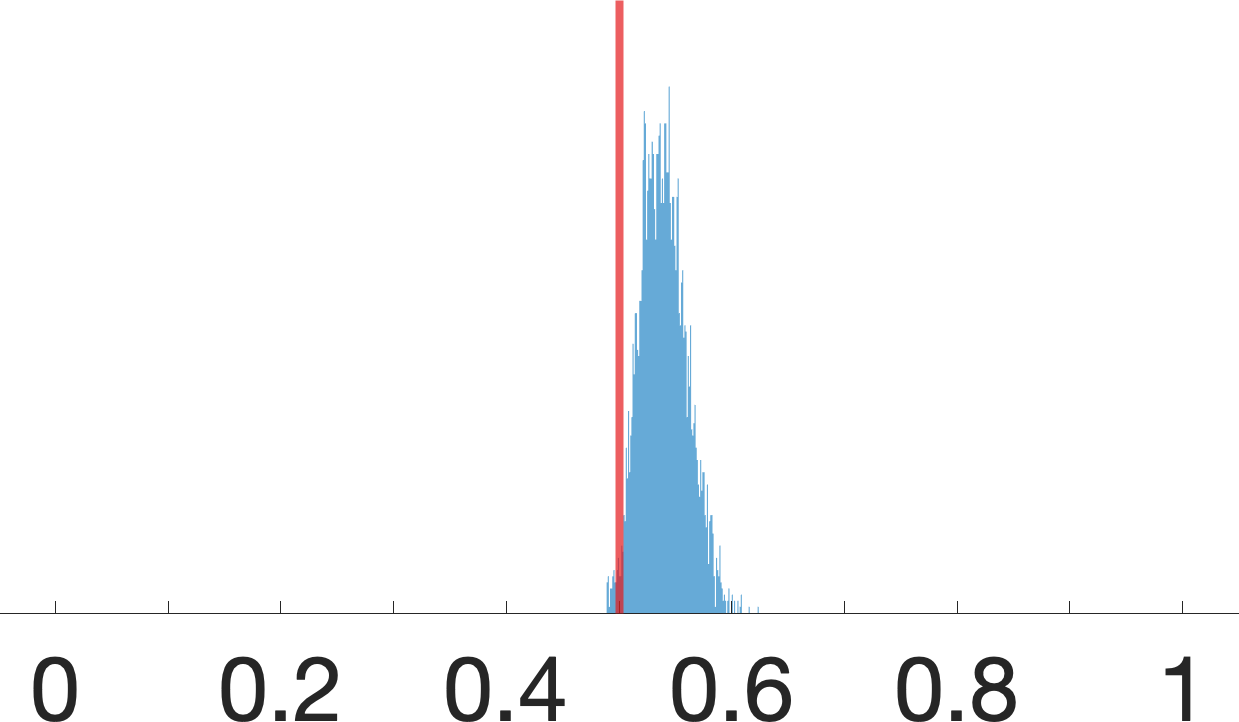}
\caption{$N = 10^3$}
\label{fig:app_c4c}
\end{subfigure}\hfill
\begin{subfigure}{0.245\linewidth}\centering
  \includegraphics[width=\linewidth]{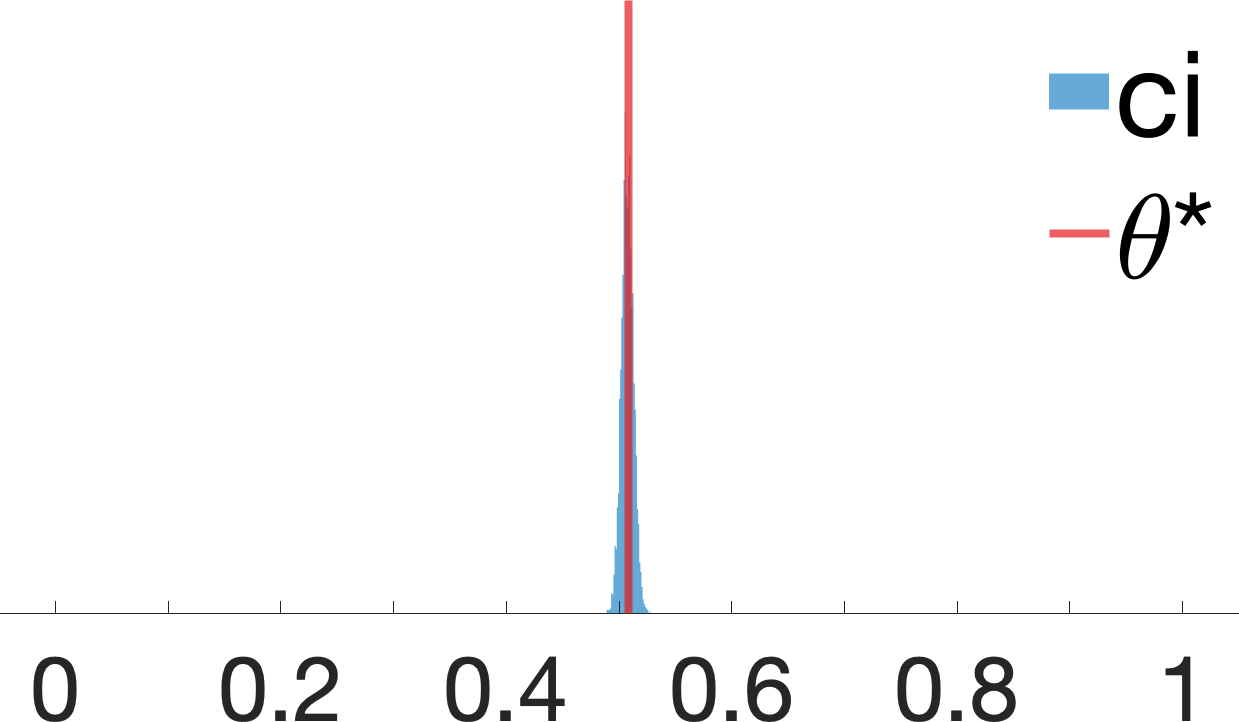}
\caption{$N = 10^4$}
\label{fig:app_c4d}
\end{subfigure}\hfill
\begin{subfigure}{0.245\linewidth}\centering
  \includegraphics[width=\linewidth]{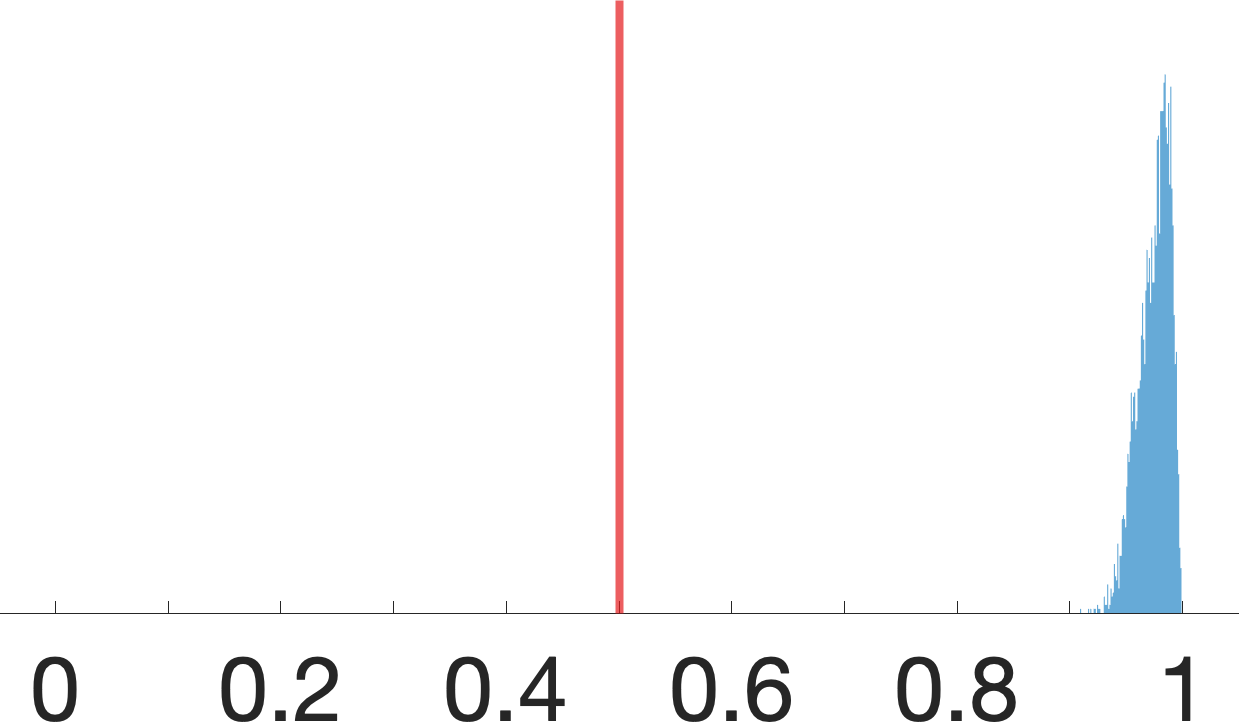}
\caption{$N = 10$}
\label{fig:app_c4e}
\end{subfigure}\hfill
\begin{subfigure}{0.245\linewidth}\centering
  \includegraphics[width=\linewidth]{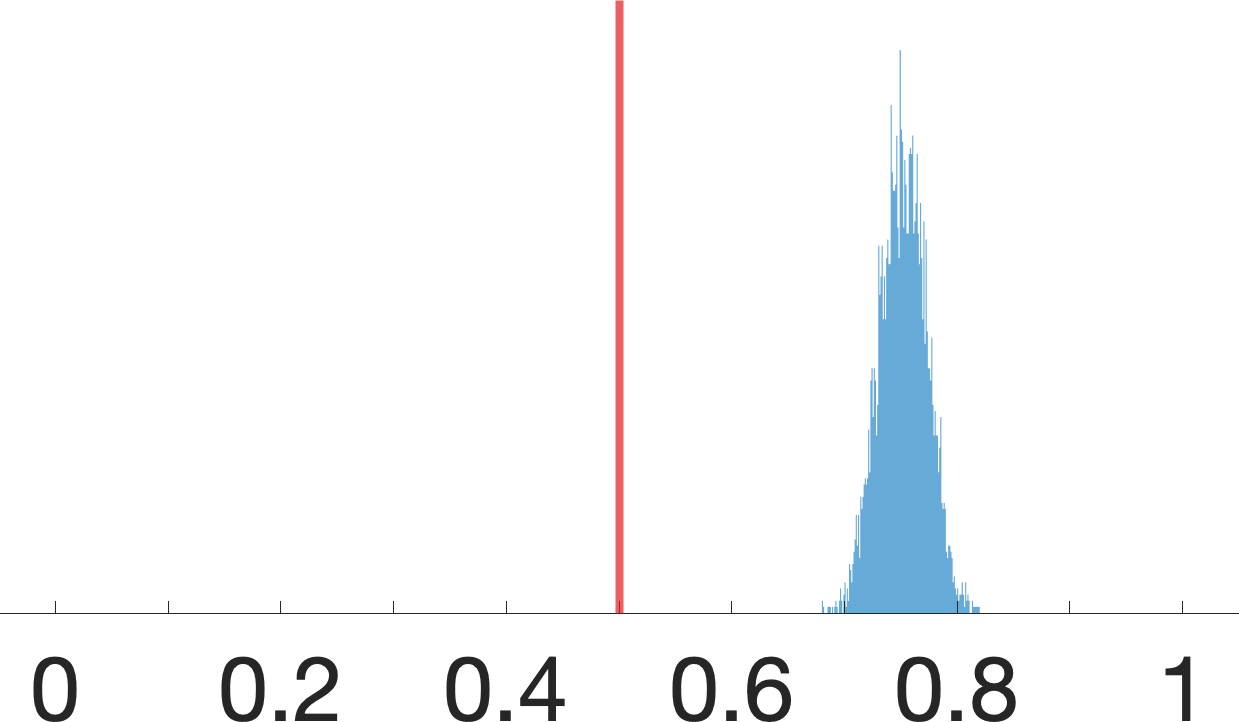}
\caption{$N = 10^2$}
\label{fig:app_c2f}
\end{subfigure}\hfill
\begin{subfigure}{0.245\linewidth}\centering
  \includegraphics[width=\linewidth]{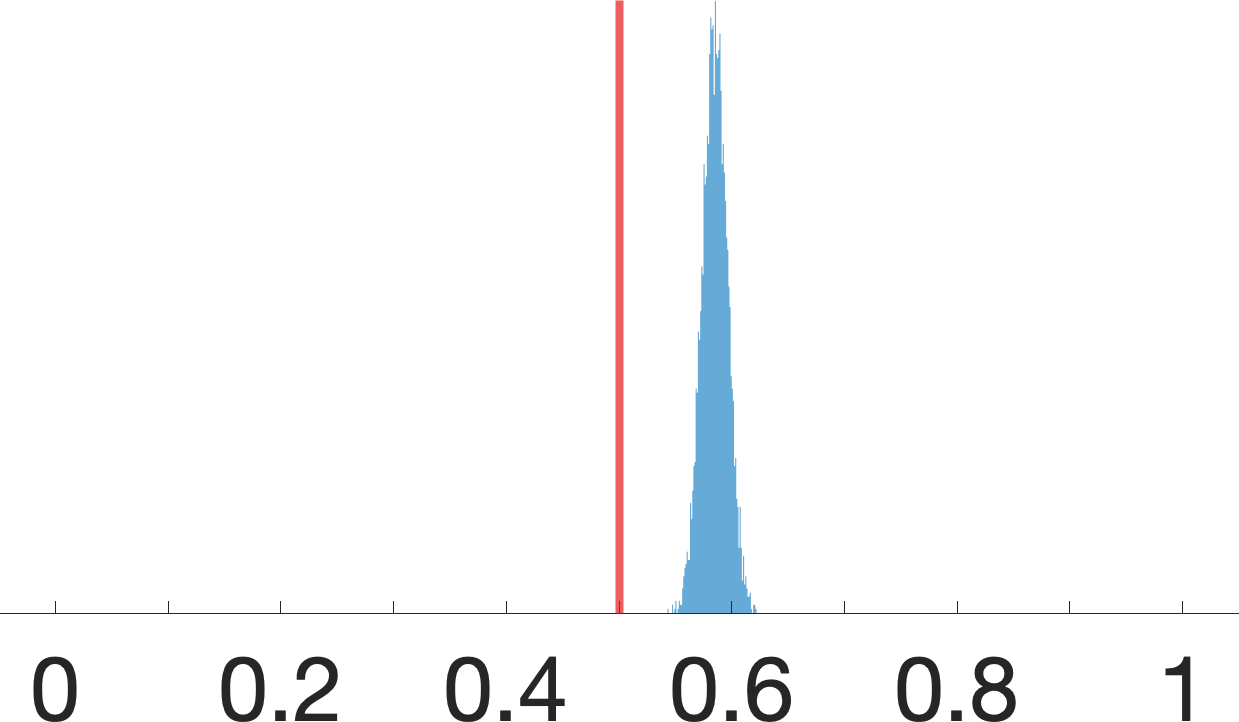}
\caption{$N = 10^3$}
\label{fig:app_c2g}
\end{subfigure}\hfill
\begin{subfigure}{0.245\linewidth}\centering
  \includegraphics[width=\linewidth]{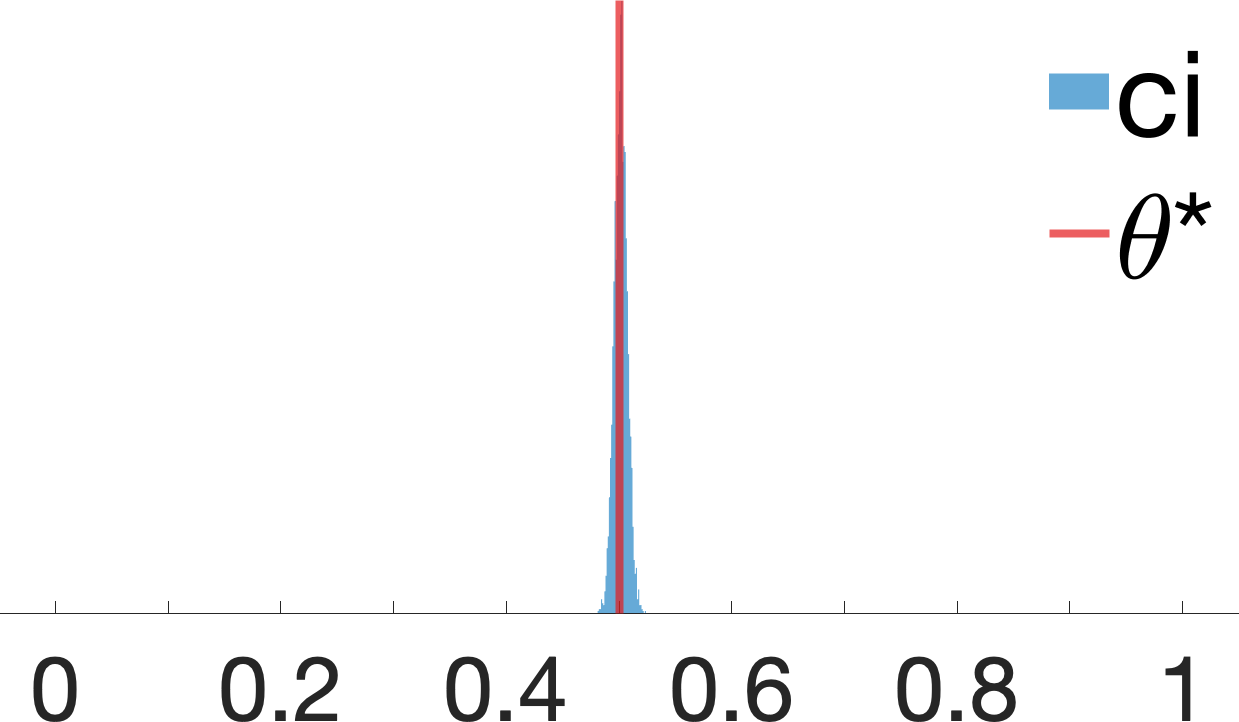}
\caption{$N = 10^4$}
\label{fig:app_c4h}
\end{subfigure}\null
\caption{Histogram plots for samples drawn from the posterior distribution over probability $P(Y_{x=0} = 0)$ in ``Frontdoor'' graph of \Cref{fig1c} using two priors. (\subref{fig:app_c4a} - \subref{fig:app_c4d}) shows the posteriors using the flat prior and observational data of size $N = 10, 10^2, 10^3$ and $10^4$ respectively; (\subref{fig:app_c4e} - \subref{fig:app_c4h}) shows the posetriors using the skewed prior and the same respective observational datasets.}
\label{fig:app_c4}
\end{figure*}

\begin{figure*}[t]
\centering
\null
\begin{subfigure}{0.245\linewidth}\centering
  \includegraphics[width=\linewidth]{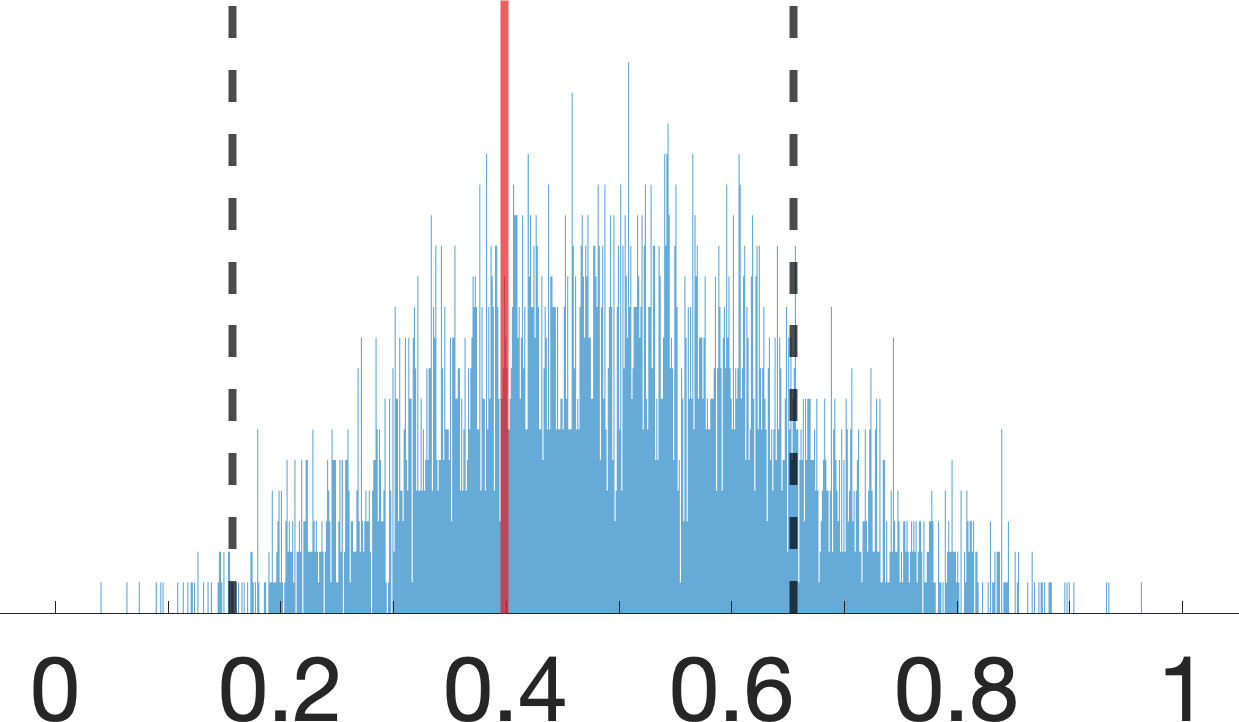}
\caption{$N = 10$}
\label{fig:app_c5a}
\end{subfigure}\hfill
\begin{subfigure}{0.245\linewidth}\centering
  \includegraphics[width=\linewidth]{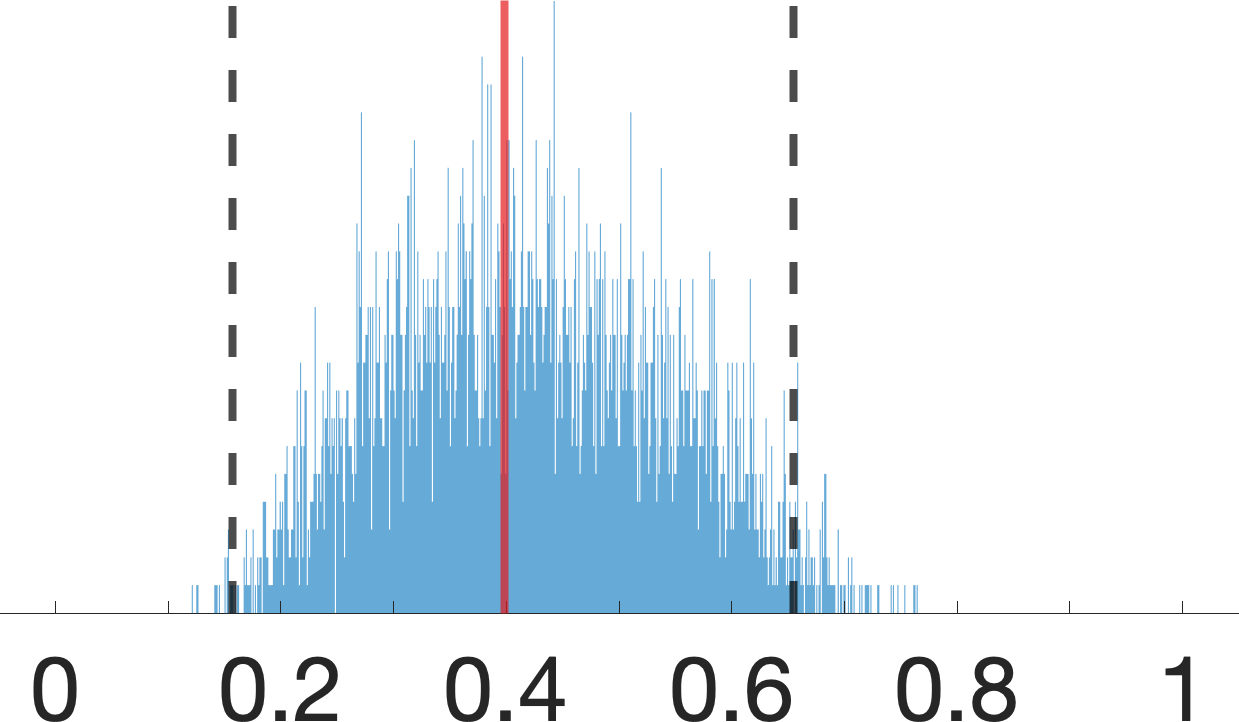}
\caption{$N = 10^2$}
\label{fig:app_c5b}
\end{subfigure}\hfill
\begin{subfigure}{0.245\linewidth}\centering
  \includegraphics[width=\linewidth]{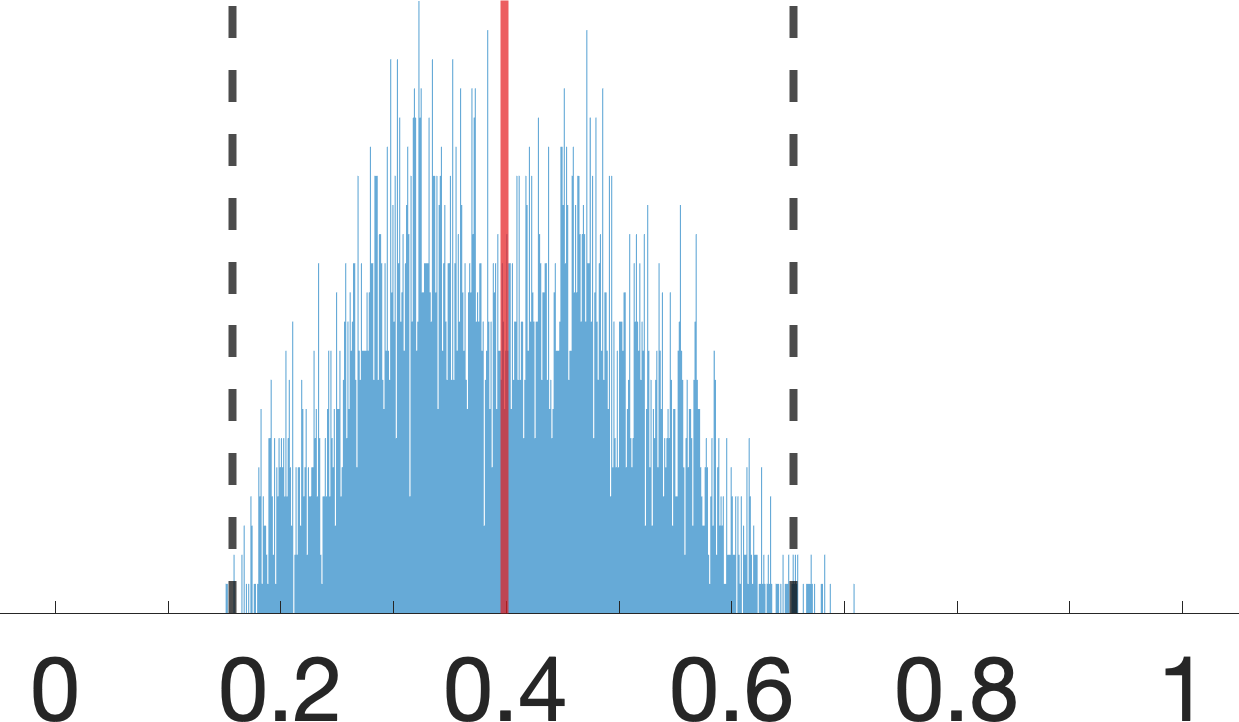}
\caption{$N = 10^3$}
\label{fig:app_c5c}
\end{subfigure}\hfill
\begin{subfigure}{0.245\linewidth}\centering
  \includegraphics[width=\linewidth]{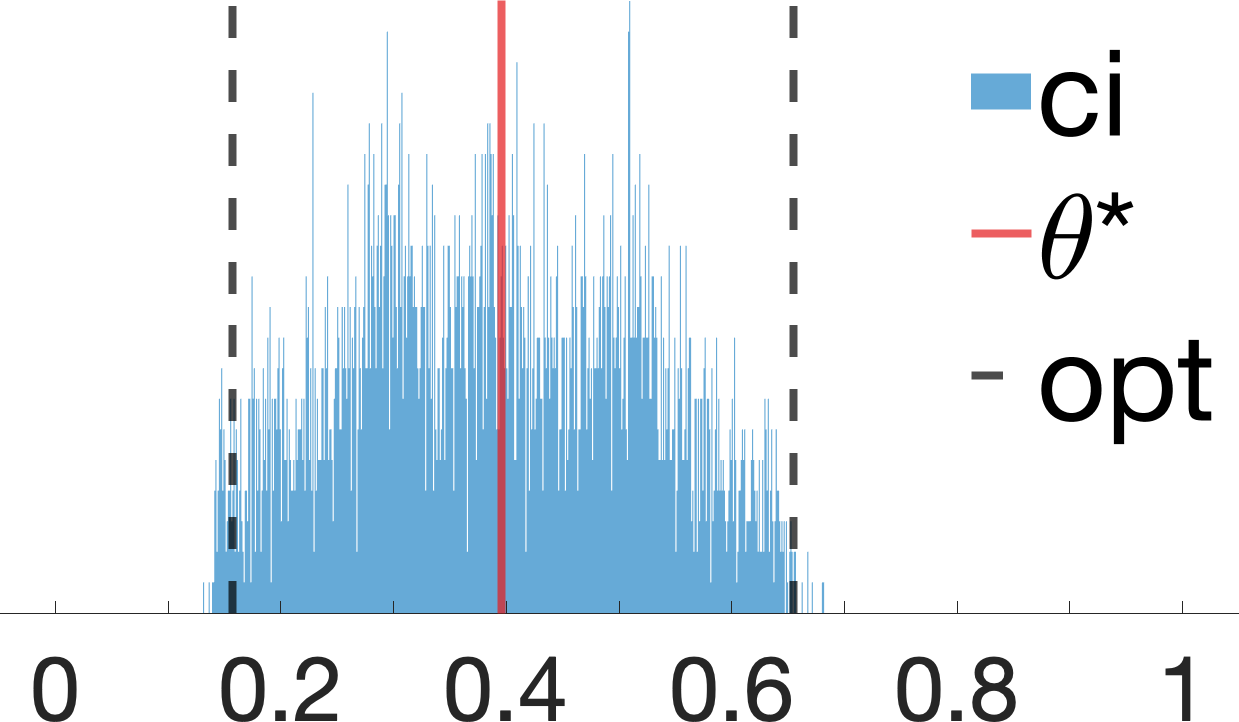}
\caption{$N = 10^4$}
\label{fig:app_c5d}
\end{subfigure}\hfill
\begin{subfigure}{0.245\linewidth}\centering
  \includegraphics[width=\linewidth]{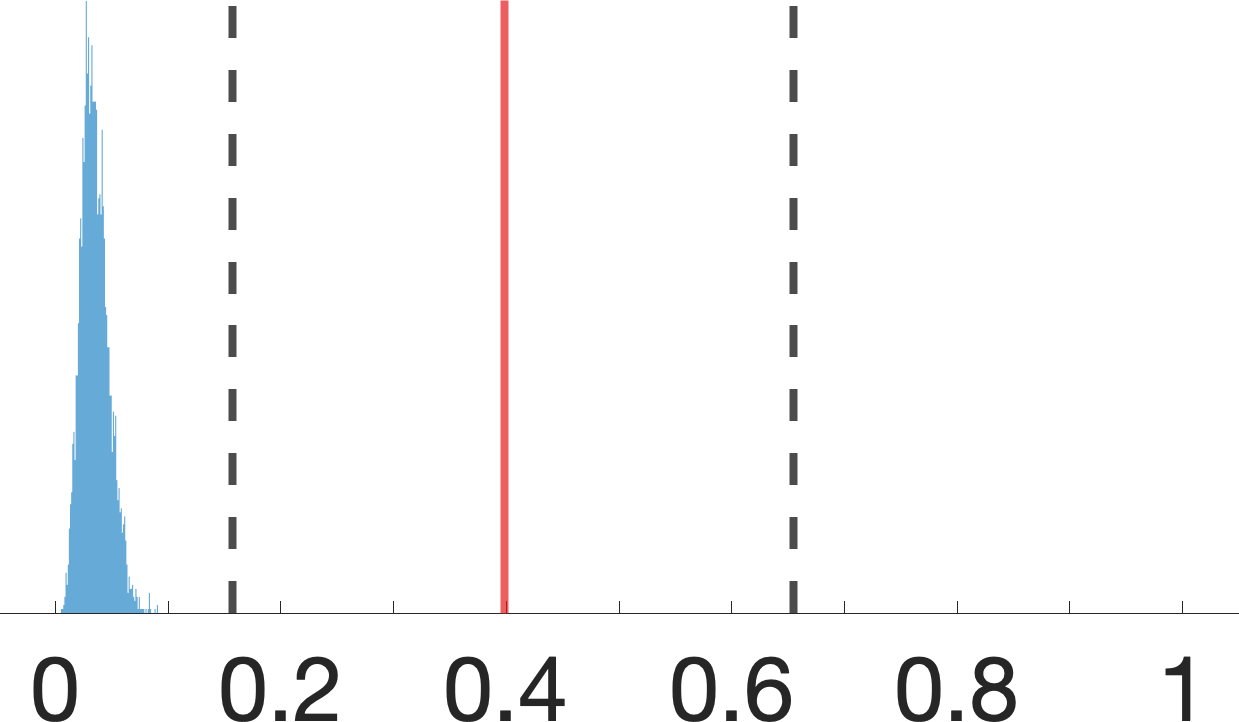}
\caption{$N = 10$}
\label{fig:app_c5e}
\end{subfigure}\hfill
\begin{subfigure}{0.245\linewidth}\centering
  \includegraphics[width=\linewidth]{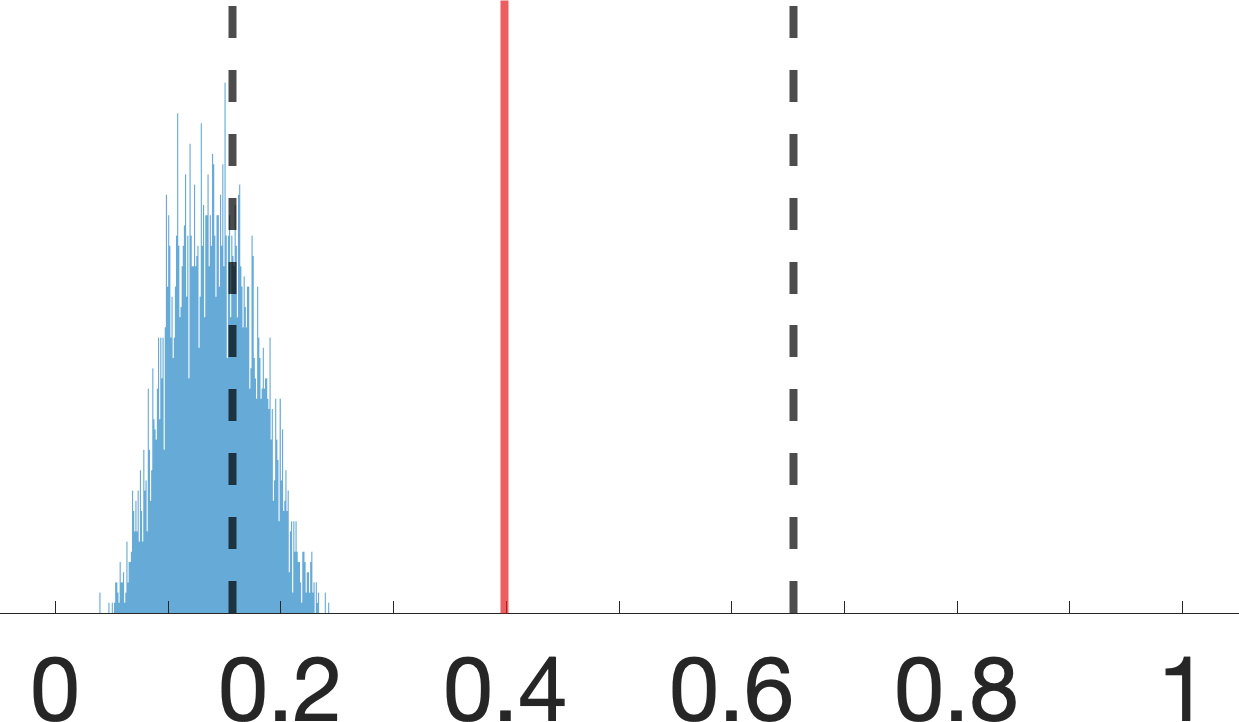}
\caption{$N = 10^2$}
\label{fig:app_c5f}
\end{subfigure}\hfill
\begin{subfigure}{0.245\linewidth}\centering
  \includegraphics[width=\linewidth]{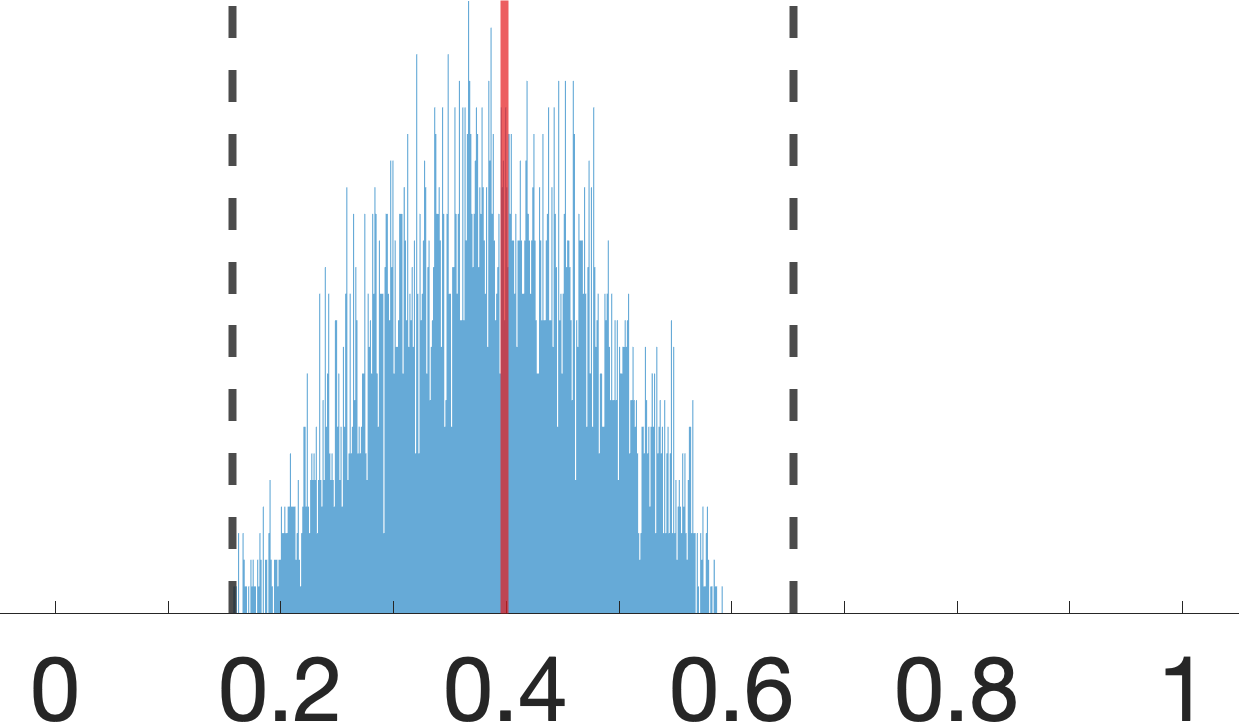}
\caption{$N = 10^3$}
\label{fig:app_c5g}
\end{subfigure}\hfill
\begin{subfigure}{0.245\linewidth}\centering
  \includegraphics[width=\linewidth]{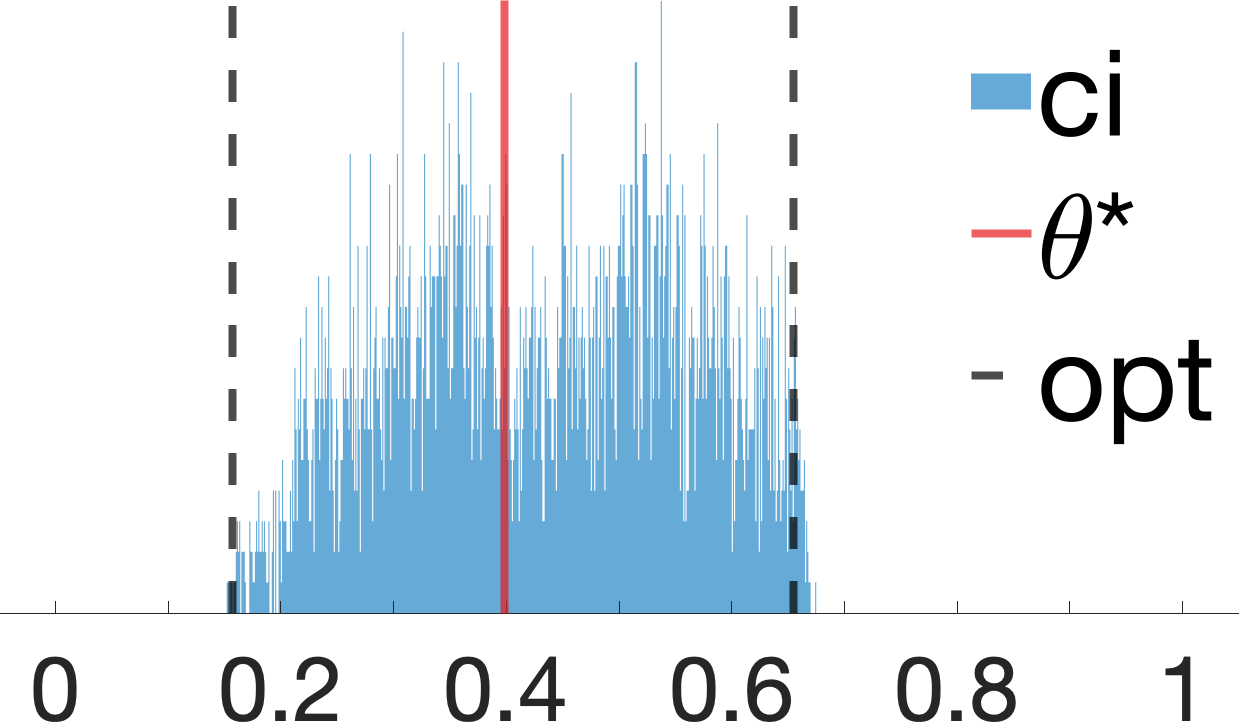}
\caption{$N = 10^4$}
\label{fig:app_c5h}
\end{subfigure}\null
\caption{Histogram plots for samples drawn from the posterior distribution over probability $P(Y_{x=0} = 0)$ in ``IV'' graph of \Cref{fig1a} using two priors. (\subref{fig:app_c5a} - \subref{fig:app_c5d}) shows the posteriors using the flat prior and observational data of size $N = 10, 10^2, 10^3$ and $10^4$ respectively; (\subref{fig:app_c5e} - \subref{fig:app_c5h}) shows the posetriors using the skewed prior and the same respective observational datasets.}
\label{fig:app_c5}
\end{figure*}

\paragraph{Experiment 8: Triple Bow}Consider the ``Triple Bow'' diagram in \Cref{fig:app_c2d} where $X, Y, Z \in \{0, 1\}$ and $U_1, U_2, U_3 \in \3R$. We are interested in evaluating the counterfactual probability $P(Y_{x=1}=1, Y_{x=0} = 0)$ from the combination of the observational distribution $P(X, Y, Z, W)$ and interventional distributions $P(X_z, Y_z, W_z)$. To our best knowledge, existing bounding strategies are not applicable to this setting. We collect $N = 10^3$ samples $\5v = \{x^{(n)}, y^{(n)}, z^{(n)}, w^{(n)}\}_{n = 1}^N$ from an SCM compatible \Cref{fig:app_c2d}. The detailed parametrization of the SCM is provided in the following:
\begin{equation}
    \begin{split}
    &U_1 \sim \texttt{Unif}(0, 1), \\
    &U_i \sim \texttt{Normal}(0, 1), \;\; i = 2, 3, \\
    &Z \sim \floor{1.5 \cdot U_1},\\
    &W \sim \texttt{Binomial}(1, \rho_W),\\
    &X \sim \texttt{Binomial}(1, \rho_X),\\
    &E \sim \texttt{Logistic}(0, 1), \\
    &Y \gets \I_{X - U_3 + E + 0.1 > 0},
    \end{split} \label{eq:triple}
\end{equation}
where probabilities $\rho_Z, \rho_W, \rho_X$ are given by:
\begin{align*}
  &\rho_Z = \frac{1}{1 + \exp(-U_1)}, \\
  &\rho_W = \frac{1}{1 + \exp(-Z - U_1 - U_2)}, \\
  &\rho_X = \frac{1}{1 + \exp(-W - U_2 - U_3)}.
\end{align*}
Each sample $\left(x^{(n)}, y^{(n)}, z^{(n)}, w^{(n)}\right)$ is an independent draw from the observational distribution $P(X, Y, Z, W)$ or an interventional distribution $P(X_z, Y_z, W_z)$. To obtain a sample from $P(x_z, y_z, w_z)$, we pick a constant $z \in \D_Z$ uniformly at random, perform intervention $\doo(Z = z)$ in the SCM described in \Cref{eq:triple} and observed subsequent outcomes. 

In this experiment, we set hyperparameters $\alpha_{U_1} = d_{U_1} = 32$ and $\alpha_{U_2} = d_{U_1} = 32$. \Cref{fig:app_c1d} shows samples drawn from the posterior distribution of $\left ( P(Y_{x = 0} = 1) \mid \5v \right)$. The analysis reveals that our proposed approach is able to achived an effective bound that contain the actual counterfactual probability $P(Y_{x=1}=1, Y_{x=0} = 0) = 0.1867$. The $100\%$ credible interval (\textit{ci}) is equal to $l= 0.1150, r= 0.3686$.

\subsection{C.2 The Effect of Sample Size and Prior Distributions}
We will evaluate our algorithms using skewed prior distributions. We found that increasing the size of observational samples was able to wash away the bias introduced by prior distributions. That is, despite the influence of prior distributions, our algorithms eventually converge to sharp bounds over unknown counterfactual probabilities as the number of observational sample grows (to infinite).

\paragraph{Experiment 9: Frontdoor}Consider first the ``Frontdoor'' graph in \Cref{fig1d} where interventional probabilities $P(y_x)$ is identifiable from the observational distribution $P(X, Y, W)$. The detailed parametrization of the underlying SCM is described in \Cref{eq:frontdoor}. We present our results using two different priors. The first is a flat (uniform) distribution over probabilities of $U_1$ and $U_2$ respectively, i.e., $\alpha_{U_1} = d_{U_1} = 8$ and $\alpha_{U_1} = d_{U_2} = 4$. The second is skewed to present a strong preference on the deterministic relationships between $X$ and $Y$; in this case, $\alpha_{1} = 300 \times d_{U_i}$, $i = 1,2$, for prior distributions associated with both $U_1$ and $U_2$. \Cref{fig:app_c3a,fig:app_c3b} shows the distribution of $P(Y_{x=0})$ induced by these two priors (in the absence of any observational data). We see that the skewed prior of \Cref{fig:app_c3b} assigns almost all weights to deterministic events $P(Y_{x=0} = 1)=1$ or $P(Y_{x=0} = 0)=1$. 

\Cref{fig:app_c2} shows posterior samples obtained by our Gibbs sampler when applied to observational data of various sizes, using both the flat prior (\Crefrange{fig:app_c4a}{fig:app_c4d}) and the skewed prior (\Crefrange{fig:app_c4e}{fig:app_c4h}). Both priors eventually collapse to the actual, unknown probability $P(Y_{x = 0} = 1) = 0.5085$. As expected, more observational data are needed for the skewed prior before the posterior distribution converges, since the skewed prior is concentrated further away from the value $0.5085$ than the uniform prior.

\paragraph{Experiment 10: IV}Consider the ``IV'' diagram in \Cref{fig1a} where $X, Y, Z$ are binary variables taking values in $\{0, 1\}$. Detailed parametrization of the SCM is provided as follows:
\begin{equation}
    \begin{split}
    &U_1 \sim \texttt{Normal}(0, 1),\\
    &U_2 \sim \texttt{Normal}(0, 1), \\
    &Z \sim \texttt{Binomial}(1, \rho_Z), \\
    &X \sim \texttt{Binomial}(1, \rho_X), \\
    &Y \sim \texttt{Binomial}(1, \rho_Y),
    \end{split}\label{eq:iv}
\end{equation}
where probabilities $\rho_Z, \rho_X, \rho_Y$ are given by
\begin{align*}
  &\rho_Z = \frac{1}{1 + \exp(-U_1)},\\
  &\rho_X = \frac{1}{1 + \exp(-Z - U_2)},\\
  &\rho_Y = \frac{1}{1 + \exp(X - U_2 + 0.5)}.
\end{align*} 
In this case, the counterfactual distribution $P(y_x)$ is not identifiable from the observational distribution $P(X, Y, Z)$ \citep{bareinboim:pea12}. Sharp bounds over $P(y_x)$ from $P(x, y, z)$ were derived in \citep{balke:pea94b} (labelled as \textit{opt}). We present our results using two different priors. The first is a flat (uniform) distribution over probabilities of $U_1$ and $U_2$ respectively, i.e., $\alpha_{U_1} = d_{U_1} = 2$ and $\alpha_{U_1} = d_{U_2} = 16$. The second is skewed to present a strong preference on the deterministic relationships between $X$ and $Y$; in this case, $\alpha_{1} = 300 \times d_{U_i}$, $i = 1,2$, for prior distributions associated with both $U_1$ and $U_2$. \Cref{fig:app_c3c,fig:app_c3d} shows distributions of $P(Y_{x=0})$ induced by these two prior distributions (in the absence of any observational data). We see that the skewed prior of \Cref{fig:app_c3d} assigns almost all weights to deterministic events $P(Y_{x=0} = 1)=1$ or $P(Y_{x=0} = 0)=1$. 

\Cref{fig:app_c5} shows posterior samples obtained by our Gibbs sampler when applied to observational data of various sizes, using both the flat prior (\Crefrange{fig:app_c5a}{fig:app_c5d}) and the skewed prior (\Crefrange{fig:app_c5e}{fig:app_c5h}). Our analysis reveals that $100\%$ credible intervals of both priors eventually converge to the sharp IV bound $l = 0.1468, r = 0.6617$ over the unknown interventional probability $P(Y_{x = 0} = 1) = 0.3954$. It is interesting to note that, in this experiment, while the choice of prior distribution does not influence the final bound, it still has an effect on the shape of posterior distributions given finite samples of the observational data. 
\clearpage
\section{D. Polynomial Optimization for Bounding Counterfactual Probabilities} \label{appendix:e}
In this section, we will demonstrate how one could solve the optimization problem in \Cref{eq:opt1} assuming access to a polynomial optimization oracle.

Recall that \Cref{thm:cfinite} implies that counterfactual distributions $P \left (\*Y_{\*x}, \dots, \*Z_{\*w} \right)$ in any causal diagram $\G$ could be generated by a discrete SCM in the canonical family $\2N(\G)$ and be written as follows:
\begin{align*}
  &P \left (\*y_{\*x}, \dots, \*z_{\*w} \right) = \sum_{\*u}\I_{\*Y_{\*x}(\*u) = \*y} \cdots \I_{\*Z_{\*w}(\*u) = \*z} \prod_{U \in \*U} \theta_u. 
\end{align*}
Among above quantities, for every exogenous $U \in \*U$, parameters $\theta_u$ are discrete probabilities $P(U = u)$ over a finite domain $\{1, \dots, d_U\}$ where the cardinality $d_U = \prod_{V \in \*C(U)} \left |\D_{\Pa_V} \mapsto \D_V \right|$. That is, parameters $\left \{ \theta_u \mid \forall u \in \D_U \right \}$ satisfy the following:
\begin{align*}
    &\theta_u \in [0, 1], &\sum_{u \in \D_U} \theta_u = 1.
\end{align*}
For every $V \in \*V$, we represent the output of function $f_V(\pa_V, u_V)$ given input $\pa_V, u_V$ using an indicator vector $\mu_V^{(\pa_V, u_V)} = \left ( \mu_v^{(\pa_V, u_V)} \mid \forall v \in \D_V \right)$ such that 
\begin{align*}
  &\mu_v^{(\pa_V, u_V)}\in \{0, 1\}, &\sum_{v \in \D_V} \mu_v^{(\pa_V, u_V)} = 1.
\end{align*} 
For subsets $\*X, \*Y \subseteq \*V$, fix constants $\*x, \*y, \*u$. The indicator function $\I_{\*Y_{\*x}(\*u) = \*y}$ could be written as a product 
\begin{align*}
    \I_{\*Y_{\*x}(\*u) = \*y} = \prod_{Y \in \*Y} \I_{Y_{\*x}(\*u) = y}.
\end{align*}
For every $Y \in \*Y$, $\I_{Y_{\*x}(\*u) = y}$ is recursively given by: 
\begin{align*}
  \I_{Y_{\*x}(\*u) = y} = 
    \begin{dcases}
      \I_{y = \*x_Y} &\mbox{if }Y \in \*X \\
      \sum_{\pa_Y} \mu_y^{\left (\pa_Y, u_Y \right)} \I_{\PA_{Y\*x}  = \pa_Y }  &\mbox{otherwise}
    \end{dcases}
\end{align*} 
The above equations allow us to write any counterfactual probability $P \left (\*y_{\*x}, \dots, \*z_{\*w} \right)$ as a polynomial function of parameters $\mu_v^{(\pa_V, u_V)}$ and $\theta_u$. This means that the polynomial optimization in \Cref{eq:opt1} could be reducible to an equivalent polynomial optimization program. For the remainder of this section, we will illustrate this reduction using various examples in different causal diagrams. 

\paragraph{Example 1: IV}Consider the ``IV'' diagram $\G$ in \Cref{fig1a}. We study the problem of bounding counterfactual probabilities $P(y'_{x'}, x, y) \equiv P(Y_{x'} = y', X = x, Y = y)$ from the observational distribution $P(X, Y, Z)$. Formally, let $\2M(\G)$ denote the set of all SCMs compatible with the diagram $\G$. One could obtain the tight bound over $P(y'_{x'}, x, y)$ from $P(X, Y, Z)$ by solving the optimization problem as follows:
\begin{equation}
    \begin{aligned}
      \underset{M \in \2M(\G)}{\min / \max} \quad &P_M \left (y'_{x'}, x, y \right)\\
      \textrm{s.t.} \quad & P_M(x, y, z) = P(x, y, z), \;\; \forall x, y, z.
    \end{aligned} \label{exp:iv}
\end{equation}
In the above optimization problem, it follows from \Cref{thm:cfinite} that the objective function could be written as
\begin{align*}
    &P_M(y'_{x'}, x, y) \\
    &= \sum_{u_1=1}^{d_1} \sum_{u_2=1}^{d_2} \mu_{y'}^{(x', u_2)} \mu_{y}^{(x, u_2)} \sum_{z} \mu_{x}^{(z, u_2)} \mu_{z}^{(u_1)}\theta_{u_1}\theta_{u_2}.
\end{align*}
Similarly, the observational constraints could be written as:
\begin{align*}
    &P_M(x, y, z) = \sum_{u_1=1}^{d_1} \sum_{u_2=1}^{d_2} \mu_z^{(u_1)} \mu_x^{(z, u_2)} \mu_y^{(x, u_2)} \theta_{u_1}\theta_{u_2}.
\end{align*}
The above equations imply that \Cref{exp:iv} could be reducible to an equivalent polynomial program as follows:
\begin{align*}
    \min / \max \;\; &\sum_{u_1=1}^{d_1} \sum_{u_2=1}^{d_2} \mu_{y'}^{(x', u_2)} \mu_{y}^{(x, u_2)} \sum_{z} \mu_{x}^{(z, u_2)} \mu_{z}^{(u_1)}\theta_{u_1}\theta_{u_2}\\
      \text{subject to} &\;\; \sum_{u_1=1}^{d_1} \sum_{u_2=1}^{d_2} \mu_z^{(u_1)} \mu_x^{(z, u_2)} \mu_y^{(x, u_2)} \theta_{u_1}\theta_{u_2} \\
      &=P(x, y, z), \;\; \forall x, y, z\\
      &\forall z, u_1, \;\;\mu_z^{(u_1)}\left(1 - \mu_z^{(u_1)} \right) = 0, \\
      &\forall z, u_1, \;\; \sum_z \mu_z^{(u_1)} = 1\\
      &\forall x, z, u_2, \;\;\mu_x^{(z, u_2)}\left(1 - \mu_x^{(z, u_2)} \right) = 0 \\
      &\forall x, z, u_2, \;\; \sum_x \mu_x^{(z, u_2)} = 1\\
      &\forall y, x, u_2, \;\; \mu_y^{(x, u_2)}\left(1 - \mu_y^{(x, u_2)} \right) = 0 \\
      &\forall y, x, u_2 \;\; \sum_y \mu_y^{(x, u_2)} = 1\\
      &\forall u_1, \;\; 0 \leq \theta_{u_1} \leq 1, \;\; \sum_{u_1}\theta_{u_1} = 1\\
      &\forall u_2, \;\; 0 \leq \theta_{u_2} \leq 1, \;\; \sum_{u_2}\theta_{u_2} = 1
\end{align*}
where cardinalities $d_{1}, d_2$ are equal to
\begin{align*}
    &d_1 = \left| \D_{Z} \right|, &&d_{2} = \left| \D_{Z} \mapsto \D_X \right|\times \left| \D_{X} \mapsto \D_Y \right|.
\end{align*}

\paragraph{Example 2}Consider the causal diagram in \Cref{fig1b}. We study the problem of bounding counterfactual probabilities $P(z, x_{z'}, y_{x'})$ from a combination of the observational distribution $P(X, Y, Z)$ and the interventional distribution $\left \{ P(X_z, Y_z) \mid \forall z \in \D_Z \right\}$. That is,
\begin{equation}
    \begin{aligned}
      \underset{M \in \2M(\G)}{\min / \max} \quad &P_M \left (z, x_{z'}, y_{x'} \right)\\
      \textrm{s.t.} \quad & P_M(x, y, z) = P(x, y, z), \;\; \forall x, y, z\\
      & P_M(x_z, y_z) = P(x_z, y_z), \;\; \forall x, y, z
    \end{aligned} \label{exp:1b}
\end{equation}
Among quantities in the above equation, it follows from \Cref{thm:cfinite} that the objective function could be written as
\begin{align*}
    &P_M(z, x_{z'}, y_{x'}) \\
    &= \sum_{u_1, u_2 = 1}^d \mu_z^{(u_1)} \mu_{x}^{(z', u_2)} \mu_y^{(x', u_1, u_2)} \theta_{u_1}\theta_{u_2}.
\end{align*}
Similarly, the observational constraints could be written as:
\begin{align*}
    &P_M(x, y, z) = \sum_{u_1, u_2 = 1}^d \mu_z^{(u_1)} \mu_x^{(z, u_2)} \mu_y^{(x, u_1, u_2)} \theta_{u_1}\theta_{u_2},
\end{align*}
and the interventional constraints imply:
\begin{align*}
    &P_M(x_z, y_z) = \sum_{u_1, u_2 = 1}^d \mu_x^{(z, u_2)} \mu_y^{(x, u_1, u_2)} \theta_{u_1}\theta_{u_2}.
\end{align*}
The above equations imply that one could obtain an optimal bound over $P(z, x_{z'}, y_{x'})$ given by \Cref{exp:1b} by solving an equivalent polynomial program as follows:
\begin{align*}    \min / \max \;\; &\sum_{u_1, u_2 = 1}^d \mu_z^{(u_1)} \mu_{x}^{(z', u_2)} \mu_y^{(x', u_1, u_2)} \theta_{u_1}\theta_{u_2}\\
    \text{s.t.} \;\; &\sum_{u_1, u_2 = 1}^d \mu_z^{(u_1)} \mu_x^{(z, u_2)} \mu_y^{(x, u_1, u_2)} \theta_{u_1}\theta_{u_2}\\
    &=P(x, y, z), \;\; \forall x, y, z\\
    &\sum_{u_1, u_2 = 1}^d \mu_x^{(z, u_2)} \mu_y^{(x, u_1, u_2)} \theta_{u_1}\theta_{u_2} \\
    &= P(x_z, y_z), \;\; \forall x, y, z\\
    &\forall z, u_1, \;\;\mu_z^{(u_1)}\left(1 - \mu_z^{(u_1)} \right) = 0\\
    &\forall z, u_1, \;\; \sum_z \mu_z^{(u_1)} = 1\\
    &\forall x, z, u_1, u_2, \;\;\mu_x^{(z, u_1, u_2)}\left(1 - \mu_x^{(z, u_1, u_2)} \right) = 0 \\
    &\forall x, z, u_1, u_2, \;\; \sum_x \mu_x^{(z, u_1, u_2)} = 1\\
    &\forall y, x, u_2, \;\; \mu_y^{(x, u_2)}\left(1 - \mu_y^{(x, u_2)} \right) = 0 \\
    &\forall y, x, u_2, \;\; \sum_y \mu_y^{(x, u_2)} = 1\\
    &\forall u_1, \;\; 0 \leq \theta_{u_1} \leq 1, \;\; \sum_{u_1}\theta_{u_1} = 1\\
    &\forall u_2, \;\; 0 \leq \theta_{u_2} \leq 1, \;\; \sum_{u_2}\theta_{u_2} = 1
\end{align*}
where the cardinality $d$ equates to 
\begin{align*}
    d = \left| \D_{Z} \right|\times \left| \D_{Z} \mapsto \D_X \right|\times \left| \D_{X} \mapsto \D_Y \right|.
\end{align*}

\paragraph{Example 3: Frontdoor}Consider the ``Frontdoor'' diagram in \Cref{fig1c}. We are interested in evaluating interventional probabilities $P(y_x)$ from the observational distribution $P(X, Y, W)$. That is,
\begin{equation}
    \begin{aligned}
      \underset{M \in \2M(\G)}{\min / \max} \quad &P_M \left (y_x\right)\\
      \textrm{s.t.} \quad & P_M(x, y, w) = P(x, y, w), \;\; \forall x, y, w
    \end{aligned} \label{exp:frontdoor}
\end{equation}
It follows from \Cref{thm:cfinite} that the objective function in the above optimization problem could be further written as
\begin{align*}
    &P_M(y_x) = \sum_{u_1 = 1}^{d_{1}}\sum_{u_1 = 1}^{d_{2}} \sum_w \mu_{y}^{(w, u_1)} \mu_{w}^{(x, u_2)} \theta_{u_1}\theta_{u_2}.
\end{align*}
Similarly, the observational constraints could be written as:
\begin{align*}
    &P_M(x, y, w) = \sum_{u_1 = 1}^{d_1}\sum_{u_1 = 1}^{d_{2}} \mu_x^{(u)} \mu_{y}^{(w, u_1)} \mu_{w}^{(x, u_2)} \theta_{u_1}\theta_{u_2}.
\end{align*}
The above equations imply that one could obtain the optimal solution in \Cref{exp:frontdoor} by solving an equivalent polynomial optimization problem defined as follows:
\begin{align*}
    \min / \max \;\; &\sum_{u_1 = 1}^{d_{1}}\sum_{u_1 = 1}^{d_{2}} \sum_w \mu_{y}^{(w, u_1)} \mu_{w}^{(x, u_2)} \theta_{u_1}\theta_{u_2}\\
    \text{subject to} \;\; &\sum_{u_1 = 1}^{d_{1}}\sum_{u_1 = 1}^{d_{2}}  \mu_x^{(u)} \mu_{y}^{(w, u_1)} \mu_{w}^{(x, u_2)} \theta_{u_1}\theta_{u_2} \\
    &=P(x, y, w), \forall x, y, w \\
    &\forall x, u_1, \;\;\mu_x^{(u)}\left(1 - \mu_x^{(u)} \right) = 0, \\
    &\forall x, u_1, \;\; \sum_x \mu_x^{(u)} = 1,\\
    &\forall y, w, u_1, \;\; \mu_y^{(w, u_1)}\left(1 - \mu_y^{(w, u_1)} \right) = 0, \\
    &\forall y, w, u_1, \;\; \sum_y \mu_y^{(w, u_1)} = 1,\\
    &\forall w, x, u_2, \;\; \mu_w^{(x, u_2)}\left(1 - \mu_w^{(x, u_w)} \right) = 0, \\
    &\forall w, x, u_2, \;\; \sum_w \mu_w^{(x, u_w)} = 1,\\
    &\forall u_1, \;\; 0 \leq \theta_{u_1} \leq 1, \;\; \sum_{u_1}\theta_{u_1} = 1,\\
    &\forall u_2, \;\; 0 \leq \theta_{u_2} \leq 1, \;\; \sum_{u_2}\theta_{u_2} = 1,
\end{align*}
where cardinalities $d_{1}, d_2$ equate to
\begin{align*}
    &d_1 = \left| \D_X \right| \times \left| \D_{W} \mapsto \D_Y \right|, &&d_{2} = \left| \D_{X} \mapsto \D_W \right|.
\end{align*}

\paragraph{Example 4: Bow}Consider the ``Bow'' diagram in \Cref{fig1d}. We study the problem of bounding counterfactual probabilities $P(y_x, y'_{x'}) \equiv P \left(Y_x = y, Y_{x = x'} = y'\right)$ from a combination of the observational distribution $P(X, Y, Z)$ and the interventional distribution $\left \{ P(Y_x) \mid \forall x \in \D_X \right\}$, i.e.,
\begin{equation}
    \begin{aligned}
      \underset{M \in \2M(\G)}{\min / \max} \quad &P_M \left (y_x, y'_{x'} \right)\\
      \textrm{s.t.} \quad & P_M(x, y) = P(x, y), \;\; \forall x, y\\
      & P_M(y_x) = P(y_x), \;\; \forall x, y
    \end{aligned} \label{exp:bow}
\end{equation}
Among quantities in the above equation, it follows from \Cref{thm:cfinite} that the objective function could be written as
\begin{align*}
    &P_M(y_x, y'_{x'}) = \sum_{u = 1}^d \mu_{y}^{(x, u)} \mu_{y'}^{(x', u)} \theta_u.
\end{align*}
Similarly, the observational constraints could be written as:
\begin{align*}
    &P_M(x, y) = \sum_{u = 1}^d \mu_x^{(u)} \mu_{y}^{(x, u)}  \theta_{u},
\end{align*}
and the interventional constraints are given by:
\begin{align*}
    &P_M(y_x) = \sum_{u = 1}^d \mu_{y}^{(x, u)}  \theta_{u}.
\end{align*}
The optimization problem defined in \Cref{exp:bow} is thus reducible to an equivalent polynomial program as follows:
\begin{align*}
    \min / \max \;\; &\sum_{u = 1}^d \mu_{y}^{(x, u)} \mu_{y'}^{(x', u)} \theta_u\\
    \text{subject to} \;\; &\sum_{u = 1}^d \mu_x^{(u)} \mu_{y}^{(x, u)}  \theta_{u} = P(x, y), \;\; \forall x, y\\
    &\sum_{u = 1}^d \mu_{y}^{(x, u)}  \theta_{u} = P(y_x), \;\; \forall x, y\\
    &\forall x, u, \;\;\mu_x^{(u)}\left(1 - \mu_x^{(u)} \right) = 0 \\
    &\forall x, u, \;\; \sum_x \mu_x^{(u)} = 1\\
    &\forall y, x, u, \;\; \mu_y^{(x, u)}\left(1 - \mu_y^{(x, u)} \right) = 0 \\
    &\forall y, x, u, \;\; \sum_y \mu_y^{(x, u)} = 1\\
    &\forall u, \;\; 0 \leq \theta_{u} \leq 1, \;\; \sum_{u}\theta_u = 1
\end{align*}
where the cardinality $d$ is equal to $\left| \D_{Z} \mapsto \D_X \right|$.
\clearpage
\section{E. A Na\"ive Generalization of \citep{balke:pea94b}}
In this section, we will describe a na\"ive generalization of the discretization procedure introduced in \citep{balke:pea94b} to the causal diagram of \Cref{fige1}. In particular, given any SCM $M$ compatible with \Cref{fige1}, we will construct a discrete SCM $N$ compatible with a different causal diagram described in \Cref{fige2} such that $M$ and $N$ coincide in all counterfactual distributions $\*P^*$. 

We first introduce some useful notations. Let $f_Z, f_X, f_Y$ denote functions associated with $Z, X, Y$ in SCM $M$. Let constants $h_Z^{(1)} = 0$ and $h_Z^{(2)} = 1$. Note that given any $U_1 = u_1$, $f_Z(u_1)$ must equate to a binary value in $\{0, 1\}$. Therefore, we could define a partition $\1U_Z^{(i)}$, $i = 1, 2$, over domains of $U_1$ such that $u_1 \in \1U_Z^{(i)}$ if and only if 
\begin{align}
    f_Z(u_1) = h_Z^{(i)}. 
\end{align}
Given any $u_2$, $f_X(\cdot, u_2)$ defines a function mapping from domains of $Z$ to $X$. Let functions in the hypothesis class $\D_Z \mapsto \D_X$ be ordered by
\begin{equation}
    \begin{aligned}
    &h_X^{(1)}(z) = 0, &&h_X^{(2)}(z) = z, \\
    &h_X^{(3)}(z) = \neg z, &&h_X^{(4)}(z) = 1. 
    \end{aligned}\label{eq:app_d1}
\end{equation}
Similarly, we could define a partition $\1U_X^{(i)}, i = 1, 2, 3, 4$ over the domain $\D_{U_2}$ such that $u_2 \in \1U_X^{(i)}$ if and only if the induced function $f_X(\cdot, u_2) = h_X^{(i)}$. Finally, let functions in $\D_{X} \mapsto \D_Y$ mapping from domains of $X$ to $Y$ be ordered by
\begin{equation}
    \begin{aligned}
    &h_Y^{(1)}(x) = 0, &&h_Y^{(2)}(x) = x, \\
    &h_Y^{(3)}(x) = \neg x, &&h_Y^{(4)}(x) = 1.
    \end{aligned}\label{eq:app_d2}
\end{equation}
For any $u_1, u_2$, the induced function $f_Y(\cdot, u_1, u_2)$ must coincide with only of the above elements in the hypothesis class $\D_X \mapsto \D_Y$. Let $\1U_Y^{(i)}, i = 1, 2, 3, 4$ be a subset of the product domain $\D_{U_1} \times \D_{U_2}$ such that $(u_1, u_2) \in \1U_Y^{(i)}$ if any only if $f_Y(\cdot, u_1, u_2) = h_Y^{(i)}$. It is verifiable that $\1U_Y^{(i)}, i = 1, 2, 3, 4$ must form a partition over $\D_{U_1} \times \D_{U_2}$.

We now construct a discrete SCM $N$ compatible with the causal diagram of \Cref{fige2}. Let the exogenous variable $U$ in $N$ be a tuple $(U_Z, U_X, U_Y)$, where $U_Z \in \{1, 2\}$, $U_X \in \{1, 2, 3, 4\}$ and $U_Y \in \{1, 2, 3, 4\}$. For any $u_Z$, values of $Z$ are decided by a function as follows:
\begin{align}
    z \gets f_Z(u_z) = h_Z^{(u_Z)},
\end{align}
where $h_Z^{(1)} = 0$ and $h_Z^{(2)} = 1$. Given any input $z, u_X$, values of $X$ are given by 
\begin{align}
    x \gets f_X(z, u_X) = h_X^{(u_X)}(z),
\end{align}
where $h_X^{(i)}(z)$, $i = 1, 2, 3, 4$, are defined in \Cref{eq:app_d1}. Similarly, given any $x, u_Y$, values of $Y$ are given by 
\begin{align}
    y \gets f_Y(x, u_Y) = h_Y^{(u_Y)}(x),
\end{align}
where $h_Y^{(i)}(x)$, $i = 1, 2, 3, 4$, are functions defined in \Cref{eq:app_d2}. Finally, we define the exogenous distribution $P(u_Z, u_X, u_Y)$ in the discrete SCM $N$ as the joint probability over partitions $\1U_Z^{(i)}, \1U_X^{(j)}, \1U_Y^{(k)}$, $i = 1, 2$, $j = 1, 2, 3, 4$, $k = 1, 2, 3, 4$. That is, 
\begin{equation}
    \begin{aligned}
    &P_N\left (U_Z = i, U_X = j, U_Y = k \right) \\
    &= P_M\left((U_1, U_2) \in \1U_Z^{(i)} \wedge \1U_X^{(j)} \wedge \1U_Y^{(k)}\right).
    \end{aligned}
\end{equation}
It follows from the decomposition in \Cref{lem:ctf_cp1} that $N$ and $M$ must coincide in all counterfactual distributions over binary $X, Y, Z$. The total cardinality of the exogenous domains in $N$ is $|\D_{U_Z}| \times |\D_{U_X}| \times |\D_{U_Y}| = 2 \times 4 \times 4 = 32$. 

However, the construction for the reverse direction does not hold true. That is, given an arbitrary discrete $N$ compatible with the causal diagram in \Cref{fige2}, one may not be able to construct an SCM $M$ compatible with the causal diagram in \Cref{fige1} such that $M$ and $N$ coincide in all counterfactual distributions. To witness, consider a discrete SCM $N$ where $P(U_Z = U_X) = 1$, i.e., variables $U_Z$ and $U_X$ are always the same, taking values in $\{1, 2\}$. Since in SCM $N$, values of $Z(u_Z)$ and $X_{z=1}(u_X)$ are given by 
\begin{align*}
&Z(u_Z) = h_Z^{(u_Z)} = 0\times \I_{u_Z = 1} + 1\times \I_{u_Z = 2}, \\
&X_{z = 1}(u_X) = h_X^{(u_X)}(1) = 0\times \I_{u_X = 1} + 1\times \I_{u_X = 2}.
\end{align*} 
This means that values of counterfactual variables $Z$ and $X_{z = 0}$ must always coincide, i.e., $P(Z = X_{x = 1}) = 1$. However, for any SCM $M$ compatible with \Cref{fige1}, counterfactual variables $Z$ and $X_z$ must be independent due to the independence restriction \citep[Ch.~7.3.2]{pearl:2k}, i.e., $Z \ci X_z$, which is a contradiction.

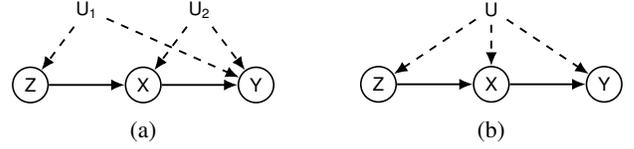
\begin{figure}[t]
  \centering
  \null
  \begin{subfigure}{0.45\linewidth}\centering
    \begin{tikzpicture}
        \node[vertex] (Z) at (0, 0) {Z};
        \node[vertex] (X) at (1.5, 0) {X};
        \node[vertex] (Y) at (3, 0) {Y};
        \node[uvertex] (U1) at (0.75, 1) {U\textsubscript{1}};
        \node[uvertex] (U2) at (2.25, 1) {U\textsubscript{2}};
        
        \draw[dir] (X) -- (Y);
        \draw[dir] (Z) -- (X);
        \draw[dir, dashed] (U1) -- (Z);
        \draw[dir, dashed] (U1) -- (Y);
        \draw[dir, dashed] (U2) -- (X);
        \draw[dir, dashed] (U2) -- (Y);
    \end{tikzpicture}
  \caption{}
  \label{fige1}
  \end{subfigure}\hfill
  \begin{subfigure}{0.45\linewidth}\centering
      \begin{tikzpicture}
          \node[vertex] (Z) at (0, 0) {Z};
          \node[vertex] (X) at (1.5, 0) {X};
          \node[vertex] (Y) at (3, 0) {Y};
          \node[uvertex] (U) at (1.5, 1) {U};
          
          \draw[dir] (X) -- (Y);
          \draw[dir] (Z) -- (X);
          \draw[dir, dashed] (U) -- (Z);
          \draw[dir, dashed] (U) -- (Y);
          \draw[dir, dashed] (U) -- (X);
      \end{tikzpicture}
    \caption{}
    \label{fige2}
  \end{subfigure}\null
  \caption{Causal diagrams (\subref{fige1}-\subref{fige2}) containing a treatment $X$, an outcome $Y$, an ancestor $Z$, and unobserved $U$s.}
  \label{fige}
\end{figure}

\end{document}